\documentclass[twoside,11pt]{article}

\usepackage[preprint]{jmlr2e}
\usepackage{amsmath}
\usepackage{tikz}
\usepackage{dsfont}
\usepackage{enumitem} 
\allowdisplaybreaks

\newcommand{\N}{\mathbb{N}}
\newcommand{\R}{\mathbb{R}}
\newcommand{\myE}[2]{\mathbb{E}_{#1}\left[ #2 \right]}
\newcommand{\myCond}[3]{\mathbb{E}_{#1}\left[\left. #2 \right\vert #3 \right]}
\newcommand*{\lmulti}{\{\mskip-5mu\{}
\newcommand*{\rmulti}{\}\mskip-5mu\}}
\newcommand*{\concat}{|\mskip-2mu|}
\newcommand{\circc}{\,\circ\,}
\newcommand{\RN}[1]{%
  \textup{\uppercase\expandafter{\romannumeral#1}}%
}

\usepackage{lastpage}
\jmlrheading{XX}{2026}{1-\pageref{LastPage}}{07/26}{X/XX}{XX-XXX}{Lukas Gonon, Thilo Meyer-Brandis, and Niklas Weber}

\ShortHeadings{Universality and Rates of GNNs}{Gonon, Meyer-Brandis, and Weber}
\firstpageno{1}

\begin{document}

\title{Universality and Approximation Rates of Graph Neural Networks with Random Features}

\author{\name Lukas Gonon \email lukas.gonon@unisg.ch \\
       \addr School of Computer Science\\
       University of St.\ Gallen\\
       St. Gallen, 9000, Switzerland
       \AND
       \name Thilo Meyer-Brandis \email meyerbra@math.lmu.de \\
       \addr Department of Mathematics\\
       LMU Munich\\
       Munich, 80333, Germany
       \AND
       \name Niklas Weber \email weber@math.lmu.de \\
       \addr Department of Mathematics\\
       LMU Munich\\
       Munich, 80333, Germany
       }

\editor{My editor}

\maketitle

\begin{abstract}
We investigate message-passing graph neural networks with random node features.
Random node features are known to enhance the expressiveness of graph neural networks (GNNs) both theoretically and empirically. Here, we establish a novel universality result focusing on permutation-equivariant neural networks (PENNs), a class of GNNs built from feedforward neural network components that subsumes many prominent GNN architectures. We show that PENNs, combined with partially random node features, can approximate arbitrarily well in probability any measurable permutation-invariant or permutation-equivariant function on directed graphs of fixed size with multidimensional node and edge features. 
For $k$-times continuously differentiable functions, $k\geq 2$, we also derive upper bounds on the approximation rates, relating the complexity of the feedforward components of a PENN in terms of layer depth and number of nonzero weights to the desired approximation accuracy.
\end{abstract}


\section{Introduction}

Graph neural networks (GNNs), and in particular message-passing GNNs, have received significant attention in recent years. 
Despite a limited theoretical understanding of the functions GNNs can approximate, they have been successfully applied across many domains where data naturally take the form of graphs, including (bio)chemistry \citep{duvenaud2015convolutional,gilmer2017neural}, social sciences \citep{wu2020graph}, and financial applications \citep{gonon2024computing,wang2021review,zhang2025forecasting,sukharev2020ews,zheng2021heterogeneous,wang2021temporal}, among others. 
For broader surveys, see, for example, \cite{zhou2020graph} or \cite{corso2024graph} and the references therein.

The origins of message-passing GNNs lie in the aggregate-and-combine (AC) framework \citep{barcelo2020logical, battaglia2018relational, gilmer2017neural}. 
AC-GNNs consist of message-passing layers that, at each step and for each node $i$, aggregate information from a node's neighborhood $\mathcal{N}(i)$ in a permutation-invariant way and then combine it with the node's current representation to produce an updated one. 
Let $N \in \mathbb{N}$ be the number of nodes, and let $\lmulti \ldots \rmulti$ denote a multiset. 
In the $L$-th message-passing round, an AC-GNN computes the new representation $h^{(L)}_i \in \mathbb{R}^{d_L}$ of node $i \in \{1,\ldots,N\}$ as
\begin{align}\label{eq:AC_GNN}
    h_i^{(L)} 
    = \operatorname{COMBINE}\!\left(
        h_i^{(L-1)},\,
        \operatorname{AGGREGATE}\!\left(\lmulti h_j^{(L-1)} \mid j \in \mathcal{N}(i) \rmulti\right)
      \right).
\end{align}
Even with the most general function choices for \textsc{AGGREGATE} and \textsc{COMBINE}, it is well known that AC-GNNs cannot distinguish certain graph structures. 
In each layer, each node's representation depends only on its neighborhood, and there exist non-isomorphic graphs for which AC-GNNs produce identical multisets of node representations. 
Such graphs are therefore indistinguishable. For the reader's convenience,  we provide illustrative examples in Appendix~\ref{appendix:graph_examples}.
This limitation is closely related to the iterative color-refinement procedure of the 1-Weisfeiler--Leman (1-WL) test \citep{weisfeiler1968reduction}, a widely used heuristic for graph isomorphism. 
As shown by \cite{jegelka2018powerful} and \cite{morris2019weisfeiler}, AC-GNNs are at most as powerful at distinguishing graphs as the 1-WL test.

To make GNNs more powerful at distinguishing graphs, two main strategies have emerged. One is to design architectures that go beyond the AC framework, enabling GNNs to capture more complex structural patterns. The other is to enrich node features with additional information so that even simple message passing can produce expressive representations.
Both approaches will be discussed in more detail later. Illustrations of simple examples where enriching the node features or altering the architecture improves GNN capabilities can be found in Appendix~\ref{appendix:graph_examples}.
 
In this work, we focus on the GNN variant of \emph{permutation-equivariant neural networks} (PENNs) introduced in \cite{herzig2018mapping} and extended in \cite{gonon2024computing}, which generalizes many popular architectures used in practice (see Remark~\ref{rem:can_approx_other_GNN}). 
We extend this framework by augmenting node features with random features and study the capability of PENNs to approximate certain functions.
Our main results are as follows.

As a first contribution, we show that  
PENNs with random features can approximate any permutation-invariant or permutation-equivariant function defined on directed graphs of fixed size with multidimensional node and edge features arbitrarily well with arbitrarily high probability.
This universal approximation result builds on classical approximation theorems for feedforward neural networks \citep{leshno1993multilayer,cybenko1989approximation,hornik1991approximation}, with the additional constraint that on graphs it is natural to restrict attention to permutation-invariant or permutation-equivariant target functions. 
Compared to existing universality results for message-passing GNNs with random features \citep{abboud2020surprising,puny2020globalattention}, our main theorem (Theorem~\ref{thm:universality_RGNN}) provides a substantial improvement. 
We specify a concrete architecture (PENNs) for which universality holds, rather than merely asserting the existence of some suitable architecture. 
We treat both permutation-invariant and permutation-equivariant functions on directed graphs, approximate measurable rather than continuous functions, impose no compactness or finiteness assumptions on node or edge features, and our framework accommodates multidimensional node and multidimensional edge features.
Furthermore, the random features we consider are very general and cover both independent and identically distributed (i.i.d.) random features used for example in \cite{sato2021random,abboud2020surprising,puny2020globalattention} and other forms of unique node identifiers \citep{murphy2019relational,dasoulas2019coloring,loukas2019graph}.  In particular, our results also hold for fixed node IDs that, for example, arise naturally when computing systemic risk measures of financial networks \citep{gonon2024computing}.
Recently, it has been noted that random features might harm generalization capabilities of GNNs most likely due to overfitting on these features and losing invariance with respect to them \citep{bechler2024utilization}. Our discussion of permutation equivariance in expectation suggests that averages of random-feature PENNs where the random features are resampled multiple times might be an approach to mitigate this issue. We prove that these models exhibit a bias towards permutation equivariance and still have the universal approximation property (Theorem~\ref{thm:universality_average_RGNN}).

As our second contribution, we provide approximation rates for PENNs with random features. For this, we restrict our attention to sufficiently regular functions and restrict node and edge features to the compact interval $[0,1]$. In this setting, we prove that PENNs with random features uniformly approximate permutation-invariant or permutation-equivariant functions arbitrarily well with arbitrarily high probability \emph{and}
we derive upper bounds for the complexity of the PENN components, which are feedforward neural networks, in terms of layer depth and number of nonzero weights.
Our approximation rates for PENNs extend existing approximation error bounds for feedforward neural networks \citep{barron2002universal,siegel2020approximation,yarotsky2017error,guhring2020error} and 
provide an important foundation for understanding the learning capabilities of these models. 
Since hyperparameters such as the number of layers must be chosen prior to training, any theoretical guidance on reasonable choices improves the performance of neural-network-based methods.

This article is structured as follows.  
The next section introduces notation and the domain of graphs.  
Section~\ref{section:gnns_and_random_features} introduces PENNs and random features. Section~\ref{section:universality} presents the universal approximation results for permutation-invariant and permutation-equivariant functions, and discusses permutation equivariance in expectation.  
Section~\ref{section:rates} contains our approximation-rate results.  
We draw a conclusion in Section~\ref{sec:conclusion} and, finally, the Appendix provides the auxiliary results regarding the proofs of our theorems, background information on feedforward neural networks, and some illustrations regarding the shortcomings of simple message passing.

\section{Notation and Domain of Graphs}
\paragraph{Notation}
Let $\N= \{1,2,\ldots\}$ denote the natural numbers, $\N_0 := \N \cup \{0\}$, and $\N_{\ge k}$ the natural numbers starting from $k \in \N$.
Let $N,N' \in \N$ and $p \geq 1$. We write $[N] := \{1,\ldots,N\}$.
For real matrices and vectors, $\|\cdot\|_{\ell^0}$ counts the number of nonzero entries.
For $a \in \R^N$, we define
$\|a\|_p := (\sum_{i=1}^N |a_i|^p )^{1/p}$ and 
$\|a\|_\infty := \max_{i \in [N]} |a_i|$.
Similarly, for $A \in \R^{N \times N'}$,
$
    \|A\|_p := (\sum_{i=1}^N\sum_{j=1}^{N'} |A_{ij}|^p )^{1/p},
$ and $
    \|A\|_\infty := \max_{i \in [N],j\in[N']} \left(|A_{i,j}|\right).
$

Let $d_1,\ldots,d_N,d_x,d_y \in \N$. For tensors $X \in \R^{d_1 \times ... \times d_N \times d_x}$ and $Y \in \R^{d_1 \times ... \times d_N \times d_y}$, we denote concatenation along the last dimension (feature dimension) by $X \concat Y \in \R^{d_1 \times ... \times d_N \times (d_x+d_y)}$.

We denote the complement of a set $\mathcal{A}$ by $\mathcal{A}^c$. A multi-set $\lmulti \ldots \rmulti$ is a set in which elements may appear multiple times.

For a probability space $(\Omega, \mathcal{F},\mathbb{P})$, let $L^0(\R)  :=L^0(\Omega, \mathcal{F},\mathbb{P}; \mathbb{R})$ denote the space of measurable functions $f:\Omega \to \R$ where functions that agree almost everywhere are identified. 
For $p > 0$, let $L^p(\R) :=L^p(\Omega, \mathcal{F},\mathbb{P}; \mathbb{R})$ denote the space of $f \in L^0(\R)$  with finite $p$-norm $\|f\|_{L^p(\R)} := (\int_\Omega |f(\omega)|^p d\mathbb{P}(\omega))^{1/p}$. Similarly, we define $L^\infty(\R)$ using the essential supremum $\|f\|_{L^\infty(\R)}:= \operatorname{ess\,sup} |f|$.
For $f = (f_1,\ldots,f_N)$, we write $f \in L^p(\R^N)$ if each $f_i \in L^p(\R)$.

For an open set $U \subseteq \R^d$ and $1 \le p \le \infty$, we write $\mathcal{L}^p(U)$ for the classical Lebesgue space of (equivalence classes of) real-valued measurable functions on $U$ with respect to Lebesgue measure, with norm $\|f\|_{\mathcal{L}^p(U)} := \left(\int_U |f(x)|^p \, dx\right)^{1/p}$ for $p<\infty$ and $\|f\|_{\mathcal{L}^\infty(U)} := \operatorname{ess\,sup}_{x \in U}|f(x)|$ for $p=\infty$. This is distinct from the probability-space notation $L^p(\R)$ above.

For $\beta \in \mathbb{N}_0^d$ we use multi-index notation $|\beta|:=\beta_1+\ldots+\beta_d$.
For $n \in \mathbb{N}_0 \cup\{\infty\}$ let $C^n(U)$ denote the space of $n$-times continuously differentiable real-valued functions on $U$.
For $f \in C^n(U)$ and $\beta \in \mathbb{N}_0^d$ with $|\beta| \leq n$, we write
$$
D^\beta f:=\frac{\partial^{|\beta|} f}{\partial x_1^{\beta_1} \partial x_2^{\beta_2} \cdots \partial x_d^{\beta_d}}.
$$
Let $n \in \mathbb{N}_0$ and $1 \leq p \leq \infty$. 
Following \cite{brezis2011functional}, we define the Sobolev space
$$
     W^{n, p}(U):=\left\{f \in \mathcal{L}^p(U) \mid D^\beta f \in \mathcal{L}^p(U) 
     \text{ for all } \beta \in \mathbb{N}_0^d \text{ with }|\beta| \leq n\right\}.
$$
For $1 \le p < \infty$, we set
$$
\|f\|_{W^{n, p}(U)}:=\left(\sum_{0 \leq|\beta| \leq n}\left\|D^\beta f\right\|_{\mathcal{L}^p(U)}^p\right)^{1 / p}
$$
and for $p=\infty$,
$$
\|f\|_{W^{n, \infty}(U)}:=\max _{0 \leq|\beta| \leq n}\left\|D^\beta f\right\|_{\mathcal{L}^{\infty}(U)}.
$$

For functions $f:\mathcal{X} \to \mathcal{Y}$ and $g:\mathcal{Y} \to \mathcal{Z}$, we write $g \circ f:\mathcal{X} \to \mathcal{Z}$ for their composition.
A function $f: U \rightarrow \mathbb{R}^m$  is Lipschitz continuous (with respect to any norms $\|\cdot\|_U$ on $U$ and $\|\cdot\|_{\R^m}$ on $\R^m$) if there is $L>0$ such that
$
\|f(x)-f(y)\|_{\R^m} \leq L\|x-y\|_U
$
for all $x, y \in U$. 

\paragraph{The Domain of Graphs}
We represent directed, weighted, and featured graphs with $N \in \mathbb{N}$ nodes, node feature dimension $d \in \mathbb{N}$, and edge feature dimension $d' \in \mathbb{N}$ as the tuple of node and edge features $(x,a) \in \mathcal{D}_{N,d,d'} := \mathbb{R}^{N \times d} \times \mathbb{R}^{N \times N \times d'}$. If each feature is restricted to a compact interval $K \subset \mathbb{R}$ (for example $K=[0,1]$), we write $\mathcal{D}_{N,d,d'}^{K} := K^{N \times d} \times K^{N \times N \times d'}$.
By including a binary flag in the node features indicating whether a node is present, this representation can also encode graphs with fewer than $N$ nodes. 
Similarly, by including a binary flag in the edge features indicating whether an edge exists (equivalently, by reserving one slice of $a$ as the adjacency matrix), we can distinguish ``no edge'' from ``zero-valued edge features.'' In terms of structure the only assumption is that there are no self-loops present, i.e., no edges from node $i$ to itself for any $i \in [N]$.

The above is not canonical, since graph nodes have no intrinsic ordering.  
Thus, multiple tensor representations may correspond to the same underlying graph.  
To formalize this, let $S_N$ be the set of all permutations of $[N]$.  
For $\sigma \in S_N$ and a graph $g = (x,a) \in \mathcal{D}_{N,d,d'}$, we define $\sigma(g) = (\sigma(x), \sigma(a))$, where $\sigma(x)_{i,k} = x_{\sigma^{-1}(i),k}$ and $\sigma(a)_{i,j,k} = a_{\sigma^{-1}(i),\,\sigma^{-1}(j),\,k}$, for $i,j \in [N], k \in [d']$.
To ensure that functions defined on graphs do not depend on the arbitrary ordering of nodes, we require permutation equivariance or permutation invariance.  
For $l \in \mathbb{N}$, a node-labeling function $f:\mathcal{D}_{N,d,d'} \to \mathbb{R}^{N \times l}$ is \emph{permutation equivariant} if $\sigma(f(g)) = f(\sigma(g))$ for all $\sigma \in S_N,\ g \in \mathcal{D}_{N,d,d'}$.
A graph-labeling function $f:\mathcal{D}_{N,d,d'} \to \mathbb{R}^{l}$ is \emph{permutation invariant} if $f(\sigma(g)) = f(g)$ for all $\sigma \in S_N,\ g \in \mathcal{D}_{N,d,d'}$.
Finally, for a graph $g = (x,a)$, we define the neighborhood of node $i \in [N]$ as $\mathcal{N}_g(i) := \{j \in [N] \mid a_{j,i} \neq 0 \}$.
When the graph is clear from context, we simply write $\mathcal{N}(i)$.

\section{Message-Passing GNNs and Random Node Features}\label{section:gnns_and_random_features}
In this section, we specify the family of GNNs referred to as permutation equivariant neural networks (PENNs), introduce random node features, and relate the setting to existing literature.

\subsection{Permutation-Equivariant Neural Networks}

We begin by introducing the PENN architecture and explain its connection to other GNNs in the literature.
Like many other GNN types, PENNs are an extension of the AC framework in~\eqref{eq:AC_GNN} with the aim to increase expressivity and predictive performance.\footnote{In this context it is noteworthy that, while a certain level of expressivity is clearly needed, a growing body of work explores whether expressivity affects generalization (the ability to predict well on unseen data); see for example \cite{maskey2026graph}.}
Other examples of architectures that go beyond the AC framework include models that incorporate higher-order neighborhoods \citep{morris2019weisfeiler} or high-dimensional tensor representations of graphs \citep{maron2018invariant,maron2019_on_universality,keriven2019universal,maron2019provably}.  
Although these approaches are theoretically powerful, they typically require prohibitive memory, and practical variants lose much of their theoretical expressiveness.
PENNs, on the other hand, are memory efficient because they only require message passing within 1-hop neighborhoods.
Originally proposed in~\cite{herzig2018mapping} and extended in~\cite{gonon2024computing}, we define PENNs as follows.

\begin{definition}\label{def:PENN} 
Let $\rho, \phi, \alpha, \psi$ be feedforward neural networks (as in Definition~\ref{def:FNN}) with compatible input and output dimensions.  
A function $f : \mathcal{D}_{N,d,d'} \to \R^{N \times l}$ with $\R^l$-valued components $(f_1,\ldots,f_N)$ is called a \emph{permutation-equivariant neural network (PENN)} if, for every input $(x,a) \in \mathcal{D}_{N,d,d'}$ and all $k \in [N]$,
\begin{align*}
    f(x,a)_k 
    := \rho\!\left(
        x_k,\ 
        \sum_{i \in \mathcal{N}(k)} \phi(x_k, a_{i,k}, x_i),\ 
        \sum_{j \in [N]} \alpha\!\left(x_j,\ \sum_{i \in \mathcal{N}(j)} \psi(x_j, a_{i,j}, x_i)\right)
    \right).
\end{align*}
\end{definition}
Definition~\ref{def:PENN} defines only single-layer PENNs. We remark that we can construct multi-layer PENNs, for example a two-layer PENN, from two networks $\operatorname{PENN}^{(1)}$ and $\operatorname{PENN}^{(2)}$ by considering the composition $\operatorname{PENN}^{(2)}(\operatorname{PENN}^{(1)}(x,a), a)$.

Next we explain how the family of PENNs can be obtained as an extension of the AC framework and we show that many commonly used GNN types are covered by the PENN architecture.

\paragraph{Relation to GNNs in the literature}
In the AC framework it is common to apply a permutation-invariant \textsc{READOUT} function as a final layer to compute a graph-level representation from all node representations \emph{after} message-passing.  
Subsequent work proposed incorporating such \textsc{READOUT} layers (or ``global attribute blocks'' in \cite{battaglia2018relational}) into \emph{every} message-passing step.  
\citet{barcelo2020logical} refer to these architectures as aggregate--combine--readout (ACR) GNNs and show that they can approximate a strictly larger class of graph functions than AC-GNNs.
An ACR-GNN updates the representation of node $i$ in the message-passing step $L \in \mathbb{N}$ via
\begin{align*}
     h_i^{(L)} &= \operatorname{COMBINE}\left(h_i^{(L-1)}, \operatorname{AGGREGATE}( \lmulti h_j^{(L-1)} \mid j \in \mathcal{N}(i) \rmulti),\right.\\
     &\left.
     \phantom{== \operatorname{COMBINE}(h_i^{(L-1)}}
     \operatorname{READOUT}(\lmulti h_j^{(L-1)} \mid j \in [N] \rmulti) \right).
\end{align*}
Originally, GNNs were designed for undirected graphs with node features or no features at all. In contrast, we study functions on directed graphs with both node and edge features. 
Therefore, instead of the multiset $\lmulti h_j^{(L-1)} \mid j \in \mathcal{N}(i) \rmulti$, which only contains representations of neighboring nodes, we consider the multiset $\lmulti (h_i^{(L-1)}, a_{j,i}, h_j^{(L-1)}) \mid j \in \mathcal{N}(i) \rmulti$ which includes the representation of the neighboring node $h_j^{(L-1)}$, the root node $h_i^{(L-1)}$, and the edge features $a_{j,i}$ of the edge from the neighbor to the root.\footnote{We work with directed graphs. Unless stated otherwise, neighborhoods refer to in-neighborhoods. This means $v$ is a neighbor of $u$ if there is an edge from $v$ to $u$. This choice is not canonical, and the theory can be reformulated for out-neighborhoods where  $v$ is a neighbor of $u$ if there is an edge from $u$ to $v$. In practice, one choice may perform better than the other depending on the learning task.} This extension is necessary if we expect the GNN to capture edge features and edge directions. Such GNNs compute new node representations as
\begin{align}
\begin{split}
\bar h_i^{(L)} &= \operatorname{AGGREGATE}( \lmulti (h_i^{(L-1)}, a_{j,i}, h_j^{(L-1)}) \mid j \in \mathcal{N}(i) \rmulti)\\
     h_i^{(L)} &= \operatorname{COMBINE}\left(h_i^{(L-1)}, \bar h_i^{(L)} , \operatorname{READOUT}(\lmulti h_j^{(L-1)} \mid j \in [N] \rmulti) \right).  \label{eq:AC-GNN_with_efeat}
\end{split}
\end{align}
Instead of applying the READOUT function directly to the current node representations, we allow an additional neighborhood aggregation step, denoted by $\operatorname{AGGREGATE}^*$. A GNN message-passing layer then computes the new node representation as
\begin{align}
\begin{split}
\bar h_i^{(L)} &= \operatorname{AGGREGATE}\!\left(\lmulti (h_i^{(L-1)}, a_{j,i}, h_j^{(L-1)}) \mid j \in \mathcal{N}(i) \rmulti \right),\\
\tilde h_i^{(L)} &= \operatorname{AGGREGATE}^*\!\left(\lmulti (h_i^{(L-1)}, a_{j,i}, h_j^{(L-1)}) \mid j \in \mathcal{N}(i) \rmulti \right),\\
h_i^{(L)} &= \operatorname{COMBINE}\!\left(h_i^{(L-1)}, \bar h_i^{(L)}, \operatorname{READOUT}\!\left(\lmulti (h_j^{(L-1)}, \tilde h_j^{(L)}) \mid j \in [N] \rmulti \right)\right).
\label{eq:ACR-GNN_with_efeat_and_agg}
\end{split}
\end{align}
This additional aggregation step is only a minor modification of the GNN architecture and, in fact, \citet{battaglia2018relational} also suggest that the READOUT function (referred to as a global attribute block) should operate on node features that have already been updated by an aggregation function. Moreover, it is easy to see that  for suitable message-passing functions, one layer of the form in~\eqref{eq:ACR-GNN_with_efeat_and_agg} can be obtained by two message-passing layers of the form in~\eqref{eq:AC-GNN_with_efeat}. Hence, the classes of representable functions are identical when multiple message-passing layers are allowed.

We can thus observe that the PENN architecture in Definition~\ref{def:PENN} corresponds to the message-passing structure in~\eqref{eq:ACR-GNN_with_efeat_and_agg} with the following choices of message-passing functions. For $i \in [N]$,  
\begin{align*}
&\operatorname{AGGREGATE}(\lmulti
(h_i,a_{j,i},h_j) \mid j \in \mathcal{N}(i)
\rmulti)=  \sum_{j \in \mathcal{N}(i)} \phi(h_i,a_{j,i},h_j),\\
&\operatorname{AGGREGATE}^*(\lmulti 
(h_i,a_{j,i},h_j) \mid j \in \mathcal{N}(i)
\rmulti)=  \sum_{j \in \mathcal{N}(i)} \psi(h_i,a_{j,i},h_j),\\
&\operatorname{READOUT}(\lmulti (h_i,\tilde h_i) \mid i \in [N]
\rmulti) =  \sum_{i=1}^N \alpha(h_i,\tilde h_i),\\
&\operatorname{COMBINE}(x,y,z) = \rho(x,y,z).
\end{align*} 
\begin{remark}\label{rem:can_approx_other_GNN}
The PENN
architecture covers many common message-passing architectures. We follow the terminology of the message-passing layers used in the PyTorch Geometric (PyG) Library \citep{fey2019fast}. Let $x_i$, $i \in [N]$, denote node features and $a_{i,j}$, $i,j \in [N]$, denote edge features.

\begin{itemize}
    \item Assume one-dimensional, non-negative edge features, let $a_{i,i}=1$ for all $i \in [N]$, and let the first node feature encode the in-degree or weighted in-degree $x_{i,1} = \sum_{j \in \mathcal{N}(i)\cup\{i\}} a_{j,i}$. Then the GCNConv architecture of~\cite{kipf2016semi}, where
    \begin{align*}
        x'_i = \Theta^T \sum_{j \in \mathcal{N}(i) \cup \{i\}}\frac{a_{j,i}}{\sqrt{x_{i,1}x_{j,1}}}x_j,
    \end{align*}
    is covered by PENNs for
    \begin{align*}
        \rho(x,y,z) &= \Theta^T y,\\
        \phi(x_i,a_{j,i},x_j) &=  \frac{a_{j,i}}{\sqrt{x_{i,1}x_{j,1}}}x_j,
    \end{align*}
    where $\Theta$ is a learnable weight matrix.

    \item The NNConv architecture of~\cite{gilmer2017neural, simonovsky2017dynamic}, where
    \begin{align*}
        x'_i = \Theta x_i + \sum_{j \in \mathcal{N}(i) } x_j h_\Theta(a_{j,i}),
    \end{align*}
    is covered if
    \begin{align*}
        \rho(x,y,z) &= \Theta x,\\
        \phi(x_i,a_{j,i},x_j) &=  x_j h_\Theta(a_{j,i}),
    \end{align*}
    where $\Theta$ is a learnable matrix and $h_\Theta$ a learnable neural network.

    \item The SAGEConv architecture of~\cite{hamilton2017inductive} with ``sum'' aggregation, where
    \begin{align*}
        x'_i = W_1 x_i + W_2 \sum_{j \in \mathcal{N}(i)} x_j,
    \end{align*}
    is covered if
    \begin{align*}
        \rho(x,y,z) &= W_1 x + W_2 y,\\
        \phi(x_i,a_{j,i},x_j) &= x_j.
    \end{align*}
    ``Mean'' aggregation is also captured if $x_{i,1}$ stores the in-degree, by $\phi(x_i,a_{j,i},x_j) = x_j / x_{i,1}$.
    
    \item The GraphConv architecture of~\cite{morris2019weisfeiler}, where
    \begin{align*}
        x'_i = W_1 x_i + W_2 \sum_{j \in \mathcal{N}(i)} a_{j,i}x_j,
    \end{align*}
    is covered if
    \begin{align*}
        \rho(x,y,z) &= W_1 x + W_2 y,\\
        \phi(x_i,a_{j,i},x_j) = a_{j,i}x_j.
    \end{align*}

    \item The GINConv architecture of~\cite{jegelka2018powerful}, where
    \begin{align*}
        x'_i = h_\Theta \left((1+ \varepsilon)x_i + \sum_{j \in \mathcal{N}(i)}x_j \right),
    \end{align*}
    is covered if
    \begin{align*}
        \rho(x,y,z) &= h_\Theta\left((1+\varepsilon) x + y\right),\\
        \phi(x_i,a_{j,i},x_j) &= x_j,
    \end{align*}
    where $h_\Theta$ is a neural network and $\varepsilon$ a scalar.
    
    \item Set $a_{i,i}=1$ for all $i \in [N]$. We artificially do this to allow self-attention and not in contrast to the standing no self-loops assumption. The GATConv architecture of~\cite{velivckovic2017graph}, is given by
    \begin{align*}
        x'_i = \sum_{j \in \mathcal{N}(i)} \alpha_{i,j} W x_j, 
    \end{align*}
    with attention coefficients
    \begin{align*}
        \alpha_{i,j} = \frac{\exp\left( \text{LeakyReLU}(b^T [Wx_i \concat Wx_j] )\right)}{\sum_{k \in \mathcal{N}(i)} \exp\left( \text{LeakyReLU}(b^T [Wx_i \concat Wx_k] )\right)}.
    \end{align*}
    It is covered by PENNs with 
    \begin{align*}
        &\rho(x,(y^A, y^B),z) = \frac{y^A}{y^B},\\
        &\phi(x_i,a_{j,i},x_j) = \phi^*(Wx_i, Wx_j, \exp\left( \text{LeakyReLU}(b^T [Wx_i \concat Wx_j] )\right)),\\
        &= \bigg(\exp\left( \text{LeakyReLU}(b^T [Wx_i \concat Wx_j] )\right)Wx_i, \exp\left( \text{LeakyReLU}(b^T [Wx_i \concat Wx_j] )\right)\bigg),
    \end{align*}
    where $W$ is a learnable weight matrix, $b$ a learnable vector, and LeakyReLU the activation from~\cite{maas2013rectifier}.
\end{itemize}
Note that since the message-passing functions $\rho,\alpha, \phi, \psi$ of PENNs are feedforward neural networks (FNNs), in some cases above it might not be possible to obtain an exact equality. However, in these cases, the universal approximation property of FNNs ensures that PENNs can at least uniformly approximate other continuous message-passing functions on appropriately chosen compact domains.
\end{remark}
In conclusion, the PENN architecture is very versatile and covers many popular existing message-passing variants.
Further, from the construction it is clear that message-passing GNNs, in particular PENNs, are permutation-equivariant functions. However, as mentioned in the introduction, ACR-GNNs are not universal approximators of permutation-equivariant functions. To obtain universality, we will further extend the architecture and introduce random node features.

\subsection{Random Node Features}\label{section:random_node_features}
In order to increase the expressivity of PENNs and to obtain universality, we augment the node features with additional information. Since our universality result is formulated in probability we consider random graphs on some probability space $(\Omega, \mathcal{F}, \mathbb{P})$.  
Let $G = (X,A)$ denote a random graph taking values in $\mathcal{D}_{N,d,d'}$ and let $R$ denote random node features taking values in $\R^{N \times d_r}$, where $d_r \in \N$ is the dimension of the random node features.
Then, instead of using $\operatorname{PENN}(X,A)$ we augment the node features by  concatenating the random features and consider
\begin{align*}
    \operatorname{PENN}(X \concat R, A),
\end{align*}
for some $\operatorname{PENN}:\mathcal{D}_{N,d+d_r,d'} \to \R^{N \times l}$ and some $l \in \N$.
The key condition we impose on the random features $R$ is that they provide \emph{finite unique node features} --- either almost surely or with high probability --- in the following sense.
\begin{definition}\label{def:rv_unique_finite_feat}
\begin{enumerate}[label=(\roman*)]
    \item\label{def:rv_unique_finite_feat:finite_unique_as}
    Random features $R \in L^0(\R^{N \times d_r})$ provide \emph{finite unique node features} if there exist $M \in \N$ and a measurable function $h : \R^{d_r} \to [M]$ such that, for any $i,j \in [N], i \neq j$,
    \begin{align*}
        h(R_i) \neq h(R_j) \quad a.s.
    \end{align*}
    \item\label{def:rv_unique_finite_feat:finite_unique_whp}
    Random features $R \in L^0(\R^{N \times d_r})$ provide finite unique node features \emph{with high probability} if, for every $p \in (0,1)$, there exist $M \in \N$ and a measurable function $h : \R^{d_r} \to [M]$ such that
    \begin{align*}
        \mathbb{P}\left( \forall i,j \in [N], i \neq j: h(R_i) \neq h(R_j)\right) \geq p.
    \end{align*}
\end{enumerate}
\end{definition}
Note that random features as in Definition~\ref{def:rv_unique_finite_feat} \ref{def:rv_unique_finite_feat:finite_unique_as} or \ref{def:rv_unique_finite_feat:finite_unique_whp} ensure that, almost surely or with probability at least $p$, respectively, the augmented random graph $(X\concat R,A)$ takes values in a set $\mathcal{D} \subseteq \mathcal{D}_{N,d+d_r,d'}$ that has finite unique node features in the following sense.
\begin{definition}\label{def:set_unique_finite_feat}
A set of graphs $\mathcal{D} \subseteq \mathcal{D}_{N,d,d'}$ has \emph{finite unique node features} if there exist $M \in \N$ and a measurable function $h : \R^{d} \to [M]$ such that, for every $(x,a) \in \mathcal{D}$ and all distinct nodes $i,j \in [N]$, the node features satisfy
\begin{align*}
        h(x_i) \neq h(x_j).
    \end{align*}
\end{definition}
Instead of working with random variables on a probability space, it may be convenient to formulate our framework in terms of probability measures on the domain of graphs. For the sake of completeness, we also present this formulation in the following.
\begin{definition}\label{def:measure_finite_unique_feat}
\begin{enumerate}[label=(\roman*)]
    \item\label{def:measure_finite_unique_feat_as}
    A probability measure $\mu : \mathcal{B}(\mathcal{D}_{N,d,d'}) \to [0,1]$ is said to \emph{induce finite unique node features} if there exist $M \in \N$ and a measurable function $h : \R^d \to [M]$ such that
    \begin{align*}
        \mu \left( \left\{(x,a) \in \mathcal{D}_{N,d,d'} \, \middle| \, \forall i \neq j \in [N]: h(x_i) \neq h(x_j) \right\}\right) = 1.
    \end{align*}
    \item\label{def:measure_finite_unique_feat_whp}
    A probability measure $\mu : \mathcal{B}(\mathcal{D}_{N,d,d'}) \to [0,1]$ is said to \emph{induce finite unique node features with high probability} if, for every $p \in (0,1)$, there exist $M \in \N$ and a measurable function $h : \R^d \to [M]$ such that
    \begin{align*}
        \mu \left( \left\{(x,a) \in \mathcal{D}_{N,d,d'} \, \middle| \, \forall i \neq j \in [N]: h(x_i) \neq h(x_j) \right\}\right) \geq p.
    \end{align*}
\end{enumerate}
\end{definition}
With such probability measures it is possible to formulate an alternative measure-theoretic version of the universal approximation result based on random node features; see Theorem~\ref{thm:universality_GNN_with_measure}.

By Definition~\ref{def:rv_unique_finite_feat}, random features that provide finite unique node features trivially also provide finite unique node features with high probability.  
A simple example is the deterministic vector $R \equiv (1,\dots, N)^T$.  
With $M = N$ and $h$ equal to the identity, this yields finite unique node features.  
Alternatively, one may take $R$ to be the matrix whose rows contain the binary or one-hot encodings of the numbers $1,\dots,N$, and let $h$ map these encodings to their corresponding integers.
We may also consider random permutations $\sigma(R)$ of any such $R$, where $\sigma$ is drawn uniformly from $S_N$.  
These still provide finite unique node features and additionally satisfy
$\sigma(R) \stackrel{d}{=} \sigma^*(\sigma(R))$
for any fixed permutation $\sigma^* \in S_N$.

An example that provides finite unique node features with high probability is the random variable  
$R \in L^0(\R^{N})$ whose components $R_i$ are independent and uniformly distributed on $[0,1]$; see Lemma~\ref{lemma:unif}.

\begin{lemma}\label{lemma:unif}
Let $R \in L^0(\R^{N})$ be such that the components $R_i$, $i = 1,\dots,N$, are independent and uniformly distributed on $[0,1]$.  
Then $R$ provides finite unique node features with high probability.
\end{lemma}
\begin{proof}
Let $p \in (0,1)$. Choose $M \in \N$ such that $1 - \frac{N^2}{M} \ge p$.
Partition the set $[0,1]$ into $M$ disjoint intervals  
$I_i = [\tfrac{i}{M}, \tfrac{i+1}{M})$ for $i = 0,\ldots,M-2$, $I_{M-1} = [\tfrac{M-1}{M},1]$. 
Define a measurable function $h : \R \to [M]$ by
\begin{align*}
        h(x) = \mathds{1}_{\left\{\cdot < 0\right\}}(x) + M \mathds{1}_{\left\{\cdot > 1 \right\}}(x) + \sum_{i=0}^{M-1} (i+1)\mathds{1}_{\{\cdot \in I_i\}}(x).
\end{align*} 
In particular, $h(x) = i+1$ whenever $x \in I_i$ for some $i = 0,\dots,M-1$.
Since the components of $R$ are i.i.d., the probability that all ``IDs'' $h(R_i)$, $i \in [N]$, are distinct satisfies
\begin{align*}
\mathbb{P}\!\left( \forall i \neq j : h(R_i) \neq h(R_j) \right)
&= 1 - \mathbb{P}\!\left( \exists\, i \neq j : h(R_i) = h(R_j) \right) \\
&\ge 1 - N^2\, \mathbb{P}\!\left( h(R_1) = h(R_2) \right) \\
&= 1 - \frac{N^2}{M}.
\end{align*}
By construction, this quantity is greater than or equal to $p$, which completes the proof.
\end{proof}

\begin{remark}
In fact, for any continuously distributed real random variable $R$, a vector of independent copies $(R^{(1)},\dots,R^{(N)}) \in L^0(\R^N)$ provides finite unique node features with high probability.
This follows from the proof of Lemma~\ref{lemma:unif} with the measurable function
$h(x) := \sum_{i=0}^{M-1}(i+1)\mathds{1}_{I_i}$, where the intervals $I_i$ form an ``equiprobable'' partition of $\R$ in the following sense.  
For $M \in \N$, choose the thresholds $q_0 < \dots < q_{M-2} \in \R$ such that
$I_0 = (-\infty,q_0]$, $I_i = (q_{i-1},q_i], i=1,...,M-2$, and $I_{M-1} = (q_{M-2}, \infty)$
satisfy $\mathbb{P}(R \in I_i) = \tfrac{1}{M}$ for all $i = 0,\dots, M-1$.
\end{remark}

\paragraph{Related literature}
A large body of work proposes data-augmentation schemes that empirically improve GNN performance, many of which rely on randomness. Examples include \cite{murphy2019relational}, \cite{dasoulas2019coloring}, \cite{loukas2019graph}, \cite{sato2021random}, \cite{abboud2020surprising}, and \cite{puny2020globalattention}, which we discuss below; \cite{eliasof2023graph} and \cite{canturk2023graph}, where positional and structural encodings are obtained from random features; and \cite{reidgeneral} and \cite{choromanski2023taming}, where random walks generate random features.
In what follows, we focus on approaches that provably increase expressiveness in the sense of universal approximation of permutation-invariant or permutation-equivariant functions.

In this context, the groundwork was laid by \citet{sato2021random}, who proposed concatenating the node features $x \in \R^{N \times d}$ of any deterministic graph input $(x,a)$ with random features $R \in L^0(\R^N)$ with i.i.d.\ components $R_i$, $i \in [N]$ (from now on simply called \emph{i.i.d.\ random features}). Instead of $\operatorname{GNN}(x,a)$, they consider random graph neural networks of the form
\begin{align*}
\operatorname{GNN}(x \concat R,a),
\end{align*}
where the random features are sampled newly each time the function is called.  
They showed that such random features improve the empirical and theoretical performance of several GNN architectures. In particular, assuming countable node features and random features with finite support, they proved that rGINs (graph isomorphism networks (GINs) from~\cite{jegelka2018powerful} augmented with random features) can approximate certain node-labeling functions with high probability, including those solving the ``minimum dominating set'' and ``maximum matching'' problems. Their empirical results on synthetic data further demonstrate that rGINs overcome three limitations of normal GINs, as they can determine the existence of a triangle, compute local clustering coefficients, and learn an algorithm for the minimum dominating set problem.
They also evaluated rGINs and rGCNs (graph convolutional networks (GCNs) from~\cite{kipf2016semi} with random node features) on three biological real-world datasets and found that models with random features slightly outperform their standard counterparts.

Building on this random-feature idea, \citet{abboud2020surprising} proved a universality result for ACR-GNNs. They showed that for every permutation-invariant function $f$ defined on the finite domain containing undirected graphs of fixed size without node or edge features, there exists an ACR-GNN which, when augmented with i.i.d.\ random features, approximates $f$ with high probability. However, they do not specify a concrete GNN architecture; instead, they construct aggregate, combine, and readout functions tailored to the target function $f$.
To support the theoretical expressiveness gained from random features, they conducted experiments on carefully designed synthetic datasets that require expressiveness beyond the 1-WL test (and thus cannot be solved by standard AC-GNNs). On these datasets, they compared memory-efficient GCNs from~\cite{kipf2016semi} with random features (referred to as GCN-RNI) against higher-order GNN variants that exceed the 1-WL expressive power but are computationally expensive. Their results show that GCN-RNI achieves near-perfect performance, competitive with higher-order GNNs, while being significantly more efficient.

Another, slightly stronger, universality result was provided by \citet{puny2020globalattention}. They showed that ACR-GNNs with random node features $R \in L^0(K^{N \times d_r})$ whose components $R_{i,k}$, $i \in [N], k \in [d_r]$, are i.i.d.\footnote{They in fact require all entries to be i.i.d. When we refer to i.i.d.\ random features, we typically assume only that the random feature vectors $R_i \in L^0(\R^{d_r})$, $i \in [N]$, are i.i.d., which also includes random node numberings.} and take values in a compact set $K \subset \R$, can approximate continuous permutation-invariant functions on compact sets of directed graphs of fixed size with one-dimensional edge features, with high probability, provided the random feature dimension $d_r$ is chosen sufficiently large.
They do not commit to a specific GNN architecture. Instead, they consider all ACR-GNNs whose message-passing mechanism can transfer structural information into the node features (this means for node features $x$, weighted adjacency matrix $a$, and random features $r$, there exist parameters such that the message-passing output is $(x, r, ar)$), and whose readout function or global attribute block is given by the DeepSets architecture of~\cite{zaheer2017deep}.
To evaluate the practical benefits of random node features, they conducted experiments on a synthetic node-classification dataset and found that random features consistently improve the performance of all tested GNN architectures (GCN, GAT, SAGEConv, GIN, GatedGCN).

All three approaches share a common limitation: adding random node features breaks the permutation invariance or equivariance of the GNN. This is expected: for some graph input $(x,a)$ and a realization $r$ of the random features $R$, the values of $\sigma\!\left(\operatorname{GNN}(x \concat r, a)\right)$ and $\operatorname{GNN}(\sigma(x) \concat r, \sigma(a))$ may differ. However, \citet{bechler2024utilization} show that GNNs must include layers that are sensitive to the random input---thereby breaking in-/equivariance---because otherwise the random features cannot increase expressiveness.
Motivated by this observation, \citet{bechler2024utilization} propose a method to regularize random-feature GNNs toward permutation in-/equivariance using a contrastive loss, and they empirically demonstrate improved performance, in particular faster convergence.

Despite this loss of \emph{pathwise} in-/equivariance, GNNs with random features still satisfy permutation in-/equivariance \emph{in distribution} under mild assumptions on the random features, and, in particular, they are permutation equivariant \emph{in expectation}. We will later show how this property can be exploited to obtain ``almost'' permutation-equivariant results; see Proposition~\ref{prop:equi_in_exp} and the subsequent discussion.

There exist additional results related to our work that rely on augmenting node features with information specifically designed to distinguish nodes, or even to make each node label unique.

\citet{murphy2019relational} propose \emph{relational pooling} (RP), which considers functions of the form
$$ 
\bar f(x,a) = \frac{1}{|S_N|}\sum_{\sigma \in S_N} f(\sigma(x,a))
$$ 
thereby generating a permutation-invariant function $\bar f$ from a potentially permutation-sensitive function $f$. They show that if $f$ is sufficiently expressive (for example a universal approximator of continuous functions) and node and edge features take values in a finite set, then RP can learn distinct representations of non-isomorphic graphs. This is essentially\footnote{This statement is subject to some technical conditions and one must apply another universal function (for example an MLP) to the graph representations.} equivalent to approximating continuous permutation-invariant functions; see \cite{chen2019equivalence, dasoulas2019coloring}.

Although this result is not directly tied to GNNs, which are known to be non-universal, \citet{murphy2019relational} propose concatenating a special case of random features, namely deterministic unique node identifiers $u \in \R^{N \times d_r}$ (for example a one-hot encoding of each node’s row in the adjacency matrix), to the node features $x$, and then computing
\begin{align*}
    \bar f(x,a) = \frac{1}{|S_N|} \sum_{\sigma \in S_N} f(x \concat \sigma(u), a).
\end{align*}
By choosing $f$ as GIN~\citep{jegelka2018powerful}, this construction yields a model that is provably more expressive than the 1-WL test. Relational pooling---evaluating all permutations of $u$ and averaging---addresses the issue that appending (random) features inside $f$ may destroy permutation invariance. The main drawback, however, is that $|S_N|$ grows too quickly for realistic graph sizes, making the full sum computationally infeasible.
To mitigate this, they propose using random permutations during training (a procedure called $\pi$-stochastic gradient descent) and averaging over a small number of permutations at inference time. This idea is closely related to permutation in-/equivariance \emph{in expectation}, discussed in Proposition~\ref{prop:equi_in_exp} and thereafter.
They demonstrate the effectiveness of relational pooling through experiments on synthetic and real-world datasets. On synthetic datasets consisting of 1-WL-indistinguishable graphs, they show that RP-GINs (GINs with relational pooling) can distinguish graphs that classical GINs provably cannot. On biochemical datasets, they show that RP improves the performance of the state-of-the-art model (in this case, the GNN architecture of \cite{duvenaud2015convolutional}).

In the same spirit as the unique node identifiers used in \cite{murphy2019relational}, \citet{dasoulas2019coloring} propose the $k$-CLIP algorithm, an extension of the ACR-GNN architecture that assigns ``colors'' to distinguish nodes with identical features, thereby increasing the expressive power of message-passing GNNs. They show that, in the limit $k \to \infty$, $k$-CLIP can approximate continuous permutation-invariant functions on compact sets of undirected graphs with node features. Even for relatively small values of $k$ (between $0$ and $16$), their method slightly outperforms benchmark approaches, including GINs, on real-world datasets (social networks and bioinformatics), and significantly outperforms them on synthetic datasets specifically designed to test a model’s ability to detect structural properties such as triangle-freeness.

Finally, without committing to a specific GNN architecture, \citet{loukas2019graph} show that AC-GNNs can approximate any graph function computable by a ``Turing machine,'' provided that node features disambiguate the nodes, the message-passing layers are sufficiently expressive, and the layer size and width (hidden-state dimension) are chosen appropriately. They also note that, with disambiguating node features, even single-layer ACR-GNNs could, in principle, compute any such Turing-complete function if the readout function is expressive enough. Although ``Turing completeness'' is a different perspective on GNN expressiveness that we do not pursue here, the result highlights that augmented node features can, at least theoretically, substantially strengthen GNNs.

\begin{remark}
We emphasize that, just like i.i.d.\ random features, deterministic unique node identifiers and random permutations thereof are simply special cases of random features. One advantage of using (random permutations of) unique node identifiers \citep{murphy2019relational, dasoulas2019coloring, loukas2019graph} compared to i.i.d.\ random features $R \in L^0(\R^{N \times d_r})$ with i.i.d.\ components $R_i$, $i \in [N]$ \citep{sato2021random, abboud2020surprising}) is that they provide unique node labels almost surely. A drawback, however, is that such identifiers cannot be assigned independently to each node, since the identifiers of all other nodes must be taken into account.
For example, $2N$ i.i.d.\ random node features can serve as valid random features for either two graphs of size $N$ or one graph of size $2N$. In contrast, two identical sets of $N$ unique node identifiers do \emph{not} provide unique identifiers for a graph of size $2N$, since each identifier appears twice. This makes i.i.d.\ random features easier to handle in practice. According to \citet{abboud2020surprising}, further advantages of i.i.d.\ random features are that they are data-agnostic (they work for any graph size and are not tailored to a specific graph), do not impose additional structure (such as an ordering of nodes), and can yield infinitely many representations for a single graph (for example when drawn from a continuous distribution), whereas unique node identifiers typically come from a finite set of size at least $N$.
Our results apply to both types of random features, since both provide finite unique node features in probability as specified in Definition~\ref{def:rv_unique_finite_feat}.
\end{remark}

\section{Universality of PENNs with Random Node Features} \label{section:universality}
With the PENN architecture and random node features in place, we are ready to state and prove the central result of this paper. We first establish the universal approximation property of PENNs with random node features for permutation-equivariant functions. Then we present an alternative measure-theoretic version of the result, an almost sure version under stronger assumptions, and explain how all results extend naturally to permutation-invariant functions. After this, we discuss permutation equivariance in expectation: since augmenting with random node features might break \emph{pathwise} permutation equivariance, one could consider averages of PENNs where the random features are resampled multiple times. We show that this approach is biased more towards permutation equivariance and still has the universal approximation property.
\subsection{Universal Approximation Result}
Let us first turn to the universality result we aim to establish. We show that PENNs from Definition~\ref{def:PENN} with random node features from Section~\ref{section:random_node_features} are universal approximators of permutation-equivariant functions on directed graphs with multidimensional node and edge features. The corresponding result for permutation-invariant functions follows directly; see Remark~\ref{rem:extends_to_perm_inv}. Compared to existing results \citep{abboud2020surprising, puny2020globalattention, murphy2019relational, dasoulas2019coloring}, this fills substantial gaps in the literature: we allow the target function to be measurable (rather than continuous), to be permutation invariant \emph{or} equivariant (rather than only one of them), and to be defined on an unbounded, uncountable domain of directed graphs with multidimensional features.

The only condition we impose on the random features is that they provide finite unique node features with high probability as in Definition~\ref{def:rv_unique_finite_feat}.
Then the following result holds.
\begin{theorem}\label{thm:universality_RGNN}
Let $N,d,d',d_r,l \in \N$.  
Let $(X,A): \Omega \to \mathcal{D}_{N,d,d'}$ be a graph-valued random variable, and let  
$R : \Omega \to \R^{N \times d_r}$ be random features that provide finite unique node features with high probability.  
Let $f\colon\mathcal{D}_{N,d,d'} \to \R^{N \times l}$ be a measurable permutation-equivariant function.

Then, for all $\varepsilon, \delta > 0$, there exists a  
$\operatorname{PENN} : \mathcal{D}_{N,d+d_r,d'} \to \R^{N \times l}$ as in Definition~\ref{def:PENN} such that
\begin{align*}
    \mathbb{P}\left(\|f(X,A) - \operatorname{PENN}(X\concat R,A)\| \leq \varepsilon\right) \geq 1-\delta.
\end{align*}
\end{theorem}
\begin{remark}
    Note that augmenting the original node features with random features is unnecessary if the original features $X$ already provide finite unique node features with high probability. In this case, from the proof of Theorem~\ref{thm:universality_RGNN} it follows that there exists a $\operatorname{PENN} : \mathcal{D}_{N,d,d'} \to \R^{N \times l}$ as in Definition~\ref{def:PENN} such that
\begin{align*}
    \mathbb{P}\left(\|f(X,A) - \operatorname{PENN}(X,A)\| \leq \varepsilon\right) \geq 1-\delta.
\end{align*}   
\end{remark}
\begin{proof}\textbf{of Theorem~\ref{thm:universality_RGNN}}
    Let $\varepsilon, \delta > 0$.
    Since $R$ provides finite unique node features with high probability, we can choose $M \in \N$ and a measurable map $h : \R^{d_r} \to [M]$ such that $\mathbb{P}(\mathfrak{A}) \ge 1 - \delta/2$, where $\mathfrak{A} \subseteq \Omega$ is the set on which $R$ provides finite unique node features: $ \mathfrak{A} = \{\omega \in \Omega \mid \forall i,j \in [N], i \neq j: h(R_i(\omega)) \neq h(R_j(\omega))\} \subseteq \Omega$.
    Thus, $\mathcal{D} := \{(X \concat R,A)(\omega) \mid \omega \in \mathfrak{A}\} \subset \mathcal{D}_{N,d+d_r,d'}$ is a set with finite unique node features as in Definition~\ref{def:set_unique_finite_feat}. 

    Define $\tilde f: \mathcal{D}_{N,d+d_r,d'} \to \R^{N \times l}$ as the natural extension of $f$, $\tilde f(x \concat r,a) := f(x,a)$, for any $(x\concat r,a) \in \mathcal{D}_{N,d+d_r,d'}$. 
    By Proposition~\ref{prop:repr_of_GPEqui}, there exist measurable functions $\rho, \alpha, \phi, \psi$ of compatible dimensions such that the function $g: \mathcal{D}_{N,d+d_r,d'} \to \R^{N \times l}$ defined for $k \in [N]$ and $(x,a) \in \mathcal{D}_{N,d+d_r,d'}$ by 
    \begin{align}\label{eq:def_g_section_1}
        g(x,a)_k := \rho\left(x_k, \sum_{i \in \mathcal{N}(k)}\phi \left(x_k, a_{i,k}, x_i\right), \sum_{j \in [N]}\alpha\left(x_j,\sum_{i \in \mathcal{N}(j)} \psi \left(x_j,a_{i,j}, x_i\right)\right)\right),
    \end{align}
    satisfies 
    $$
    \tilde f(x,a) = g(x,a) \qquad\text{for all } (x,a) \in \mathcal{D}.
    $$
    Using this representation of $\tilde f$ on $\mathcal{D}$, we complete the proof by showing that there exist feedforward neural networks $\hat\rho, \hat\alpha, \hat\phi, \hat\psi$ approximating $\rho, \alpha, \phi, \psi$ such that the resulting PENN approximates $\tilde f$ with sufficiently high probability.

    Since the pushforward measure $\mathbb{P} \circ (X \concat R, A)^{-1}$ is a probability measure on $\mathcal{D}_{N,d+d_r,d'}$, Proposition~\ref{prop:universality_of_composition} guarantees the existence of feedforward neural networks $\hat\rho, \hat\alpha, \hat\phi, \hat\psi$ with bounded activation functions (like sigmoid) such that, for the function $\operatorname{PENN}$ defined for $k \in [N]$ and $(x,a) \in \mathcal{D}_{N,d+d_r,d'}$ by
    \begin{align}\label{eq:struct_of_GNN}
        \operatorname{PENN}(x,a)_k = \hat \rho\left(x_k, \sum_{i \in \mathcal{N}(k)}\hat \phi \left(x_k, a_{i,k}, x_i\right), \sum_{j \in [N]}\hat \alpha\left(x_j,\sum_{i \in \mathcal{N}(j)}\hat \psi \left(x_j,a_{i,j}, x_i\right)\right)\right),
    \end{align}
    we have, with $\tilde X := X \concat R$,
    \begin{align*}
        \mathbb{P}(\| g(\tilde X,A) - \operatorname{PENN}(\tilde X,A) \| < \varepsilon) \geq 1-\delta/2.
    \end{align*}
    The architecture in \eqref{eq:struct_of_GNN} is a PENN as in Definition~\ref{def:PENN}.

    Using that $\tilde f(x \concat r, a) = f(x,a)$ and that $\tilde f(\tilde X, A)$ and $g(\tilde X, A)$ coincide on $\mathfrak{A}$, we obtain
    \begin{align*}
        \mathbb{P}(\|f(X,A) - \operatorname{PENN}(X \concat R, A)\| < \varepsilon ) &\geq \mathbb{P}(\{\|f(X,A) - \operatorname{PENN}(\tilde X, A)\| < \varepsilon\} \cap \mathfrak{A})\\
        &= \mathbb{P}(\{\|\tilde f(\tilde X,A) - \operatorname{PENN}(\tilde X, A)\| < \varepsilon\} \cap \mathfrak{A})\\
        &\geq \mathbb{P}(\{\|g(\tilde X,A) - \operatorname{PENN}(\tilde X, A)\| < \varepsilon\} \cap \mathfrak{A})\\ 
        &\geq  \mathbb{P}(\|g(\tilde X,A) - \operatorname{PENN}(\tilde X, A)\| < \varepsilon ) + \mathbb{P}(\mathfrak{A}) - 1\\ 
        &\geq 1- \delta/2 + 1-\delta/2 -1 = 1-\delta.
    \end{align*}
\end{proof}
Rather than augmenting node features with random features, we can also reformulate our results in terms of measures that induce finite unique node features with high probability; see Definition~\ref{def:measure_finite_unique_feat}. 
Then Theorem~\ref{thm:universality_RGNN} can be reformulated as follows.
\begin{theorem}\label{thm:universality_GNN_with_measure}
    Let $N,d,d',l \in \N$.
    Let $\mu$ be a probability measure on $\mathcal{D}_{N,d,d'}$ that induces finite unique node features with high probability as in Definition~\ref{def:measure_finite_unique_feat}.
    Let $f: \mathcal{D}_{N,d,d'} \to \R^{N \times l}$ be a measurable permutation-equivariant function.
    For every $\varepsilon >0, \delta >0$, there exists a $\operatorname{PENN}:\mathcal{D}_{N,d,d'} \to \R^{N \times l}$ as in Definition~\ref{def:PENN} such that
\begin{align*}
    \mu\left( \left\{(x,a) \in \mathcal{D}_{N,d,d'} \, \middle | \, \|f(x,a) - \operatorname{PENN}(x,a)\| \leq \varepsilon \right\}\right) \geq 1-\delta.
\end{align*}
\end{theorem}
\begin{proof}
The proof is almost identical to that of Theorem~\ref{thm:universality_RGNN}; we only sketch the main steps.
\begin{itemize}
    \item Let $\varepsilon, \delta > 0$.  
    Since $\mu$ induces finite unique node features with high probability, there exist  
    $M \in \N$ and a measurable map $h : \R^{d} \to [M]$ such that, for
    $\mathfrak{A} = \{(x,a) \in \mathcal{D}_{N,d,d'} \mid \forall i,j \in [N], i \neq j: h(x_i) \neq h(x_j)\}$,
    we have $\mu(\mathfrak{A}) \ge 1 - \delta/2$.

    \item Since $\mathfrak{A}$ is a set with finite unique node features, Proposition~\ref{prop:repr_of_GPEqui} yields measurable functions  
    $\rho, \alpha, \phi, \psi$ such that  
    $g_{\rho,\alpha,\phi,\psi} = f$ on $\mathfrak{A}$,  
    where $g_{\rho,\alpha,\phi,\psi}$ is defined by the right-hand side of~\eqref{eq:def_g_section_1}.

    \item By Proposition~\ref{prop:universality_of_composition}, we obtain feedforward neural networks  
    $\hat\rho, \hat\alpha, \hat\phi, \hat\psi$ such that the PENN  
    $g_{\hat\rho,\hat\alpha,\hat\phi,\hat\psi}$ satisfies
    $\mu\left(\left\{(x,a) \mid \| g_{\rho,\alpha, \phi,\psi}(x,a) - g_{\hat \rho,\hat \alpha, \hat \phi, \hat \psi}(x,a))\| \leq \varepsilon\right\}\right) \geq 1-\delta/2$,
    where $g_{\hat\rho,\hat\alpha,\hat\phi,\hat\psi}$ is given by the right-hand side of~\eqref{eq:struct_of_GNN}.

    \item Reasoning as in the proof of Theorem~\ref{thm:universality_RGNN}, we conclude that
    $\mu(\{ \|f(x,a) - \operatorname{PENN}(x,a)\| \leq \varepsilon\}) \geq 1-\delta$.
\end{itemize}
\end{proof}
Under additional assumptions, it is possible to strengthen Theorem~\ref{thm:universality_RGNN} so that the approximation holds pointwise (surely), rather than only in probability. For this we assume that the permutation-equivariant target function $f:\mathcal{D}_{N,d,d'} \to \R^{N \times l}$ is continuous, and that the random node features are uniformly distributed on $\{\sigma((x_1,\ldots,x_N)) \mid \sigma \in S_N\}$,
for some $N$ distinct deterministic elements $x_1,\ldots,x_N \in [0,1]^{d_r}$. 
Focusing on a compact subset of $\mathcal{D}_{N,d,d'}$ (without loss of generality we choose $\mathcal{D}^{[0,1]}_{N,d,d'}$) the following result is a pointwise version of Theorem~\ref{thm:universality_RGNN}.

\begin{corollary}\label{cor:universality_RGNN_as}
Let $N,d,d',d_r,l \in \N$.
Let $R:\Omega \to [0,1]^{N\times d_r}$ be random features that are uniformly distributed on $\{\sigma((x_1,\ldots,x_N))\mid \sigma \in S_N\}$, for some distinct $x_1,\ldots,x_N \in [0,1]^{d_r}$.
Let further $f: \mathcal{D}_{N,d,d'} \to \R^{N \times l}$ be a continuous permutation-equivariant function.

Then, for all $\varepsilon> 0$, there exists a $\operatorname{PENN}:\mathcal{D}_{N,d+d_r,d'} \to \R^{N \times l}$ as in Definition~\ref{def:PENN} such that
\begin{align*}
    \max_{(x,a) \in \mathcal{D}_{N,d,d'}^{[0,1]}}
    \|f(x,a) - \operatorname{PENN}(x \concat R, a)\|_\infty \leq \varepsilon
\end{align*}
holds surely.
\end{corollary}
\begin{proof}
The proof is similar to that of Theorem~\ref{thm:universality_RGNN} with the only differences that we replace measurable functions with continuous ones and invoke a stronger universality result (uniform approximation of continuous functions on compacta, rather than approximation of measurable functions in probability). We only sketch the main steps.
\begin{itemize}
    \item Find measurable functions $\rho, \alpha, \phi, \psi$ such that for $g_{\rho, \alpha, \phi, \psi}$ defined as the right-hand side in \eqref{eq:def_g_section_1} $g_{\rho, \alpha, \phi, \psi}(x \concat r,a)$ equals $\tilde f(x \concat r,a) := f(x,a)$, for $(x,a) \in \mathcal{D}_{N,d,d'}$ and $r \in \{R(\omega)\mid \omega \in \Omega\}$. This is possible because the random features provide unique node features surely, allowing us to apply Proposition~\ref{prop:repr_of_GPEqui} to the set $\{(x \concat r,a)\mid (x,a) \in \mathcal{D}_{N,d,d'}, r \in \operatorname{supp}(R)\}$.
    
    \item Find continuous functions $\tilde \rho, \tilde \alpha, \tilde \phi, \tilde \psi$ such that $g_{\tilde \rho, \tilde \alpha, \tilde \phi, \tilde \psi}(x \concat R,a)$ equals $\tilde f(x \concat R,a)$ for all $(x,a) \in \mathcal{D}_{N,d,d'}$. 
    We cannot replace measurable functions by continuous ones without error in general, but since $R$ has finite support, we can construct continuous approximations that agree \emph{exactly} on this support. (See, for example, the proof of Corollary~\ref{cor:a.s._rates}.)
    
    \item Approximate $\tilde \rho, \tilde \alpha, \tilde \phi, \tilde \psi$ sufficiently well on suitable compact domains by feedforward neural networks $\hat \rho, \hat \alpha, \hat \phi, \hat \psi$, such that $\operatorname{PENN}(x \concat R,a)$ approximates $\tilde f(x \concat R,a) = f(x,a)$ uniformly for all $(x,a) \in \mathcal{D}^{[0,1]}_{N,d,d'}$.
\end{itemize}
\end{proof}
Finally, the following remark shows that all results extend directly to permutation-invariant functions. 
\begin{remark}\label{rem:extends_to_perm_inv}
Given a permutation-invariant graph-labeling function
$f:\mathcal{D}_{N,d,d'} \to \R^l$ we can construct a permutation-equivariant node-labeling function 
$f^*:\mathcal{D}_{N,d,d'} \to \R^{N \times l}$
with 
\begin{align}\label{eq:extends_to_perm_inv}
    f^* = (\underbrace{f,\ldots,f}_{\text{N times}}).
\end{align}
If $f^*$ can be approximated by some $\hat{f}$, then $f$ can be recovered either by selecting any single component $f \approx \hat f_k$, for some $k \in [N]$, or by averaging all components $f \approx \frac{1}{N}\sum_{k=1}^N \hat f_k$.

\begin{corollary}\label{cor:universality_RGNN_perm_inv}
Let $N,d,d',d_r,l \in \N$.
Let $(X,A):\Omega \to \mathcal{D}_{N,d,d'}$ be a graph-valued random variable and $R:\Omega \to \R^{N \times d_r}$ random features that provide finite unique node features with high probability. Let further $f: \mathcal{D}_{N,d,d'} \to \R^{l}$ be a measurable permutation-invariant function.

Then, for all $\varepsilon, \delta > 0$, and any $1 \leq p \leq \infty$ there exists a $\operatorname{PENN}:\mathcal{D}_{N,d+d_r,d'} \to \R^{N \times l}$ as in Definition~\ref{def:PENN} such that for every $k \in [N]$
\begin{align*}
    \mathbb{P}\left(\|f(X,A) - \operatorname{PENN}(X\concat R,A)_k\|_p \leq \varepsilon\right) \geq 1-\delta.
\end{align*}
\end{corollary}
\begin{proof}
    With $f^* = (\underbrace{f,\ldots,f}_{\text{N times}})$ from~\eqref{eq:extends_to_perm_inv} the result follows from Theorem~\ref{thm:universality_RGNN} and the fact that for any $k \in [N]$,
    \begin{align*}
        \|f(X,A) - \operatorname{PENN}(X\concat R,A)_k\|_p &= \left(\sum_{j \in [l]} |f(X,A)_{j} - \operatorname{PENN}(X \concat R,A)_{k,j}|^p \right)^\frac{1}{p}\\
        &\leq  \left(\sum_{\substack{i \in [N]\\ j \in [l]}} |f(X,A)_{j} - \operatorname{PENN}(X\concat R,A)_{i,j}|^p \right)^\frac{1}{p}\\
        &= \|f^*(X,A) - \operatorname{PENN}(X\concat R,A)\|_p.
    \end{align*}
\end{proof}
\end{remark}

\subsection{Permutation Equivariance in Expectation}
Theorem~\ref{thm:universality_RGNN} guarantees the existence of a PENN that, with random features, approximates any measurable permutation-equivariant function with arbitrarily high probability.
However, this comes at the cost of losing permutation equivariance.
Although PENN itself is permutation equivariant, with $\operatorname{R-PENN}(G) = \operatorname{R-PENN}(X,A) := \operatorname{PENN}(X \concat R, A)$, it may happen for some $\sigma \in S_N$ that
\begin{align*}
    \operatorname{R-PENN}(\sigma(G)) \neq \sigma(\operatorname{R-PENN}(G)),
\end{align*}
which is due to the fact that both $G$ and $\sigma(G)$ are augmented with the \emph{same} random features $R$, rather than augmenting $\sigma(G)$ with $\sigma(R)$.
As mentioned earlier, \citet{bechler2024utilization} show that giving up permutation equivariance is necessary to increase the expressive power of GNNs. On the other hand they show that the lack of permutation equivariance might harm the performance of GNNs.  
To mitigate this problem, we propose to consider an average of PENNs where the random features are resampled multiple times. We motivate this with the following considerations. 

Assume that the random features $R$ are independent of the graph input $(X,A)$.  
In the special case where $R = \sigma(r)$ for a random permutation $\sigma \in L^0(S_N)$ and a fixed collection of unique node identifiers $r = (r_1,\ldots,r_N) \in \mathbb{R}^{N \times d_r}$, \citet{murphy2019relational} propose \emph{relational pooling} to obtain a permutation-equivariant function
\begin{align*}
    \tau(X,A) = \frac{1}{|S_N|}\sum_{\sigma \in S_N}\operatorname{PENN}(X \concat \sigma(r) , A).
\end{align*}
In contrast to R-PENN, the function $\tau$ is permutation equivariant:  
for any permutation $\bar{\sigma} \in S_N$,
\begin{align*}
    \bar \sigma(\tau(X,A)) &= \bar \sigma\left(\frac{1}{|S_N|}\sum_{\sigma \in S_N}\operatorname{PENN}(X \concat \sigma(r) , A)\right) = \frac{1}{|S_N|}\sum_{\sigma \in S_N}\bar \sigma\left(\operatorname{PENN}(X \concat \sigma(r) , A)\right) 
    \\ 
    &= \frac{1}{|S_N|}\sum_{\sigma \in S_N}\operatorname{PENN}(\bar \sigma(X) \concat \bar \sigma(\sigma(r)) , \bar \sigma(A))
     = \tau(\bar\sigma(X), \bar \sigma(A)),
\end{align*}
which holds because PENN is permutation equivariant and $\{ \sigma(r)\mid \sigma \in S_N\} = \{\bar \sigma(\sigma(r)) \mid \sigma \in S_N\}$.
Note that in this case $\{ \sigma(r)\mid \sigma \in S_N\} = \operatorname{supp}(R)$.
The idea of averaging over all possible realizations of the random features extends to general random features whose support is finite and permutation invariant: let $R$ be random features such that $|\operatorname{supp}(R)| < \infty$ and, for every $\sigma \in S_N$, $\operatorname{supp}(R) = \operatorname{supp}(\sigma(R))$.
Then,
\begin{align*}
    \tau(X,A) = \frac{1}{|\operatorname{supp}(R)|}\sum_{r \in \operatorname{supp}(R)}\operatorname{PENN}(X \concat r, A)
\end{align*}
is permutation equivariant, because for any permutation $\bar \sigma \in S_N$
\begin{align*}
\bar \sigma(\tau(X,A)) &= \bar \sigma\left(\frac{1}{|\operatorname{supp}(R)|}\sum_{r \in \operatorname{supp}(R)}\operatorname{PENN}(X \concat r , A)\right) = \frac{1}{|\operatorname{supp}(R)|}\sum_{r \in \operatorname{supp}(R)}\bar \sigma\left(\operatorname{PENN}(X \concat r , A)\right) 
    \\ 
    &= \frac{1}{|\operatorname{supp}(R)|}\sum_{r \in \operatorname{supp}(R)}\operatorname{PENN}(\bar \sigma(X) \concat \bar \sigma(r) , \bar \sigma(A))
     = \tau(\bar\sigma(X), \bar \sigma(A)).
\end{align*}
Both functions are examples of \emph{permutation equivariance in expectation}, where the expectation is only with respect to the random features, or in other words, conditional on the graph input $(X,A)$. In general, for suitable random features, taking the expectation conditional on $(X,A)$ yields a function that is permutation equivariant.
\begin{proposition}\label{prop:equi_in_exp}
    Let $(X,A) \in L^0(\mathcal{D}_{N,d,d'})$ and let $R \in L^0(\mathbb{R}^{N \times d_r})$ be random features that are independent of $(X,A)$ and distributed in a permutation-invariant way, that is, for all $\sigma \in S_N$ it holds that $R \stackrel{d}{=} \sigma(R)$.
    Assume further that $\operatorname{PENN}(X \concat R, A) \in L^1(\R^{N \times l})$ and let 
    \begin{align*}
        \tau(X,A) = \myCond{}{\operatorname{PENN}(X \concat R,A)}{(X,A)}.
    \end{align*}
    Then, for any $\sigma \in S_N$,
    \begin{align*}
        \sigma(\tau(X,A)) = \tau(\sigma(X),\sigma(A)).
    \end{align*}
\end{proposition}
\begin{proof}
    Since $\operatorname{PENN}(X\concat R,A) \in L^1(\R^{N\times l})$, the conditional expectation $\tau(X,A)$ is well-defined, and the result follows from the independence of $R$ and $(X,A)$ together with the permutation‑invariant distribution of $R$.
    \begin{align*}
        \sigma(\tau(X,A)) &= \sigma\left(\myCond{}{\operatorname{PENN}(X \concat R,A)}{(X,A)}\right) = \myE{}{\sigma(\operatorname{PENN}(x \concat R,a))}\vert_{(x,a) = (X,A)}
        \\
        &=\myE{}{\operatorname{PENN}(\sigma(x) \concat \sigma(R),\sigma(a))}\vert_{(x,a) = (X,A)} 
        = \myE{}{\operatorname{PENN}(\sigma(x) \concat R,\sigma(a))}\vert_{(x,a) = (X,A)} 
        \\
        &= \myCond{}{\operatorname{PENN}(\sigma(X) \concat R,\sigma(A))}{(X,A)} = \tau(\sigma(X),\sigma(A)).
    \end{align*}
\end{proof}
This suggests that averaging over all possible realizations of the random features would yield a permutation-equivariant function.
In practice, however, it is usually not feasible to average over all realizations of the random features.  
Even for random permutations, the sum contains $|S_N| = N!$ terms and is therefore intractable.  
More generally, there is no reason to expect that $\operatorname{supp}(R)$ is finite.  
To overcome this issue \citet{murphy2019relational} propose approximating the average by averaging over only a small number of randomly sampled permutations.  
In the general setting, this corresponds to sampling random elements from $\operatorname{supp}(R)$. 

This motivates us to propose the following.  
Let $R^{(1)},\ldots,R^{(M)}$ be i.i.d.\ copies of $R$. We consider
\begin{align*}
    \tilde \tau(X,A) = \frac{1}{M}\sum_{i=1}^M \operatorname{PENN}(X \concat R^{(i)},A).
\end{align*}
Provided $\operatorname{PENN}(X \concat R,A) \in L^1(\R^{N\times l})$, the (strong) law of large numbers gives
\begin{align*}
    \frac{1}{M}\sum_{i=1}^M \operatorname{PENN}(X \concat R^{(i)},A) \xrightarrow[M \to \infty]{} \myE{}{\operatorname{PENN}(x \concat R,a)}\vert_{(x,a) = (X,A)} =   \myCond{}{\operatorname{PENN}(X \concat R,A)}{(X,A)}
\end{align*}
which suggests that permutation equivariance can be improved by averaging over an increasing number of i.i.d.\ copies of $R$. Most importantly, the same approximation theorem established for a single PENN also holds for this averaged PENN output.
\begin{theorem}\label{thm:universality_average_RGNN}
    Let $N,d,d',d_r,l \in \N$.
    Let $(X,A):\Omega \to \mathcal{D}_{N,d,d'}$ be a graph-valued random variable.
    Let $R:\Omega \to \R^{N \times d_r}$ be random features that provide finite unique node features with high probability and are independent of $(X,A)$. 
    Let $M \in \N$ and let $R^{(1)},...,R^{(M)}$ be i.i.d.\ copies of $R$, each independent of $(X,A)$. 
    Let $f: \mathcal{D}_{N,d,d'} \to \R^{N \times l}$ be a measurable permutation-equivariant function.

    Then, for all $\varepsilon, \delta > 0$, there exists a $\operatorname{PENN}:\mathcal{D}_{N,d+d_r,d'} \to \R^{N \times l}$ as in Definition~\ref{def:PENN} such that
\begin{align*}
    \mathbb{P}\left(\|f(X,A) - \frac{1}{M}\sum_{i \in [M]}\operatorname{PENN}(X\concat R^{(i)},A)\| \leq \varepsilon\right) \geq 1-\delta.
\end{align*}
\end{theorem}
\begin{proof}
    The proof consists largely of steps already encountered, so we only sketch it.

\begin{itemize}
    \item Let $\varepsilon,\delta >0$.
        Since $R$ provides finite unique node features with high probability, we find $M^* \in \N, h:\R^{d_r} \to [M^*]$, such that, with $\mathfrak{A} := \{\omega \in \Omega : \forall i \neq j \in [N],\ h(R_i(\omega)) \neq h(R_j(\omega))\}$, we have $\mathbb{P}(\mathfrak{A}) \geq 1- \frac{\delta}{2M}$.
        For $i \in [M]$, define analogously $\mathfrak{A}_i := \{\omega \in \Omega : \forall i' \neq j' \in [N],\ h(R^{(i)}_{i'}(\omega)) \neq h(R^{(i)}_{j'}(\omega))\}$, using the same $h$. Since $R^{(i)}$ is distributed identically to $R$, $\mathbb{P}(\mathfrak{A}_i) = \mathbb{P}(\mathfrak{A}) \geq 1-\frac{\delta}{2M}$ for every $i \in [M]$.
        
    \item We know that $\mathcal{D} = \{(x \concat r,a) \in \mathcal{D}_{N,d+d_r,d'} \mid \forall i' \neq j' \in [N],\ h(r_{i'}) \neq h(r_{j'})\} \subset \mathcal{D}_{N,d+d_r,d'}$ is a set with finite unique node features. Let $\tilde f(x \concat r,a) = f(x,a)$ be the natural extension of $f$ to $\mathcal{D}_{N,d+d_r,d'}.$
        Following Proposition~\ref{prop:repr_of_GPEqui}, we find measurable functions $\rho, \alpha, \phi, \psi$ of compatible in- and output dimensions such that for the function $g_{\rho,\alpha,\phi,\psi}$ (defined as the right-hand side of \eqref{eq:def_g_section_1}) it holds $\tilde f(x\concat r,a) = g_{\rho,\alpha,\phi,\psi}(x\concat r,a)$ for $(x \concat r,a) \in \mathcal{D}$.

    \item The pushforward measure $\mathbb{P} \circ (X \concat R, A)^{-1}$ is a probability measure on $\mathcal{D}_{N,d+d_r,d'}$.  
    By Proposition~\ref{prop:universality_of_composition}, there exist feedforward neural networks $\hat{\rho}, \hat{\alpha}, \hat{\phi}, \hat{\psi}$ such that
    \begin{align*}
        \mathbb{P}\circ (X \concat R,A)^{-1}\left(\left\{(x \concat r,a) \in \mathcal{D}_{N,d+d_r,d'} \mid \|g_{\rho,\alpha,\phi,\psi}(x \concat r,a) - g_{\hat \rho, \hat \alpha, \hat \phi, \hat \psi}(x \concat r,a) \| \leq \varepsilon \right\} \right) \geq 1- \frac{\delta}{2M}.
    \end{align*}
    Since $R^{(i)}$ is independent of $(X,A)$ and distributed identically to $R$, the pushforward measure $\mathbb{P}\circ(X\concat R^{(i)},A)^{-1}$ equals $\mathbb{P}\circ(X\concat R,A)^{-1}$ for every $i \in [M]$. Hence the same $\hat\rho,\hat\alpha,\hat\phi,\hat\psi$ satisfy, for every $i \in [M]$,
    \[
        \mathbb{P}\left(\|g_{\rho,\alpha,\phi,\psi}(X\concat R^{(i)},A) - g_{\hat\rho,\hat\alpha,\hat\phi,\hat\psi}(X\concat R^{(i)},A)\| \leq \varepsilon \right) \geq 1-\frac{\delta}{2M}.
    \]
    \item Then,
        \begin{align*}
            &\mathbb{P}\left(\|f(X,A) - \frac{1}{M}\sum_{i \in [M]} \operatorname{PENN}(X \concat R^{(i)},A)\| \leq \varepsilon \right)\\
            &= \mathbb{P}\left(\|\frac{1}{M}\sum_{i \in [M]}\tilde f(X\concat R^{(i)},A) - \frac{1}{M}\sum_{i \in [M]} \operatorname{PENN}(X \concat R^{(i)},A)\| \leq \varepsilon \right)\\
            &\geq \mathbb{P}\left(\forall i \in [M]: \|\tilde f(X\concat R^{(i)},A) - \operatorname{PENN}(X \concat R^{(i)},A)\| \leq \varepsilon \right)\\
            &\geq \sum_{i \in [M]}\left[\mathbb{P}\left(\|\tilde f(X\concat R^{(i)},A) - \operatorname{PENN}(X \concat R^{(i)},A)\| \leq \varepsilon \right) \right] - (M-1)\\
            &\geq \sum_{i \in [M]}\left[\mathbb{P}\left( \left\{\|\tilde f(X\concat R^{(i)},A) - \operatorname{PENN}(X \concat R^{(i)},A)\| \leq \varepsilon\right\} \cap \mathfrak{A}_i\right) \right] - (M-1)\\
            &=  \sum_{i \in [M]}\left[\mathbb{P}\left( \left\{\|g_{\rho,\alpha,\phi,\psi}(X\concat R^{(i)},A) - g_{\hat \rho, \hat \alpha, \hat \phi, \hat \psi}(X \concat R^{(i)},A)\| \leq \varepsilon\right\} \cap \mathfrak{A}_i\right) \right] - (M-1)\\
            &\geq \sum_{i \in [M]}\left[\mathbb{P}\left( \left\{\|g_{\rho,\alpha,\phi,\psi}(X\concat R^{(i)},A) - g_{\hat \rho, \hat \alpha, \hat \phi, \hat \psi}(X \concat R^{(i)},A)\| \leq \varepsilon\right\}\right) +  \mathbb{P}\left(\mathfrak{A}_i\right) -1 \right] - (M-1)\\
            &\geq \sum_{i \in [M]}\left[1- \frac{\delta}{2M} + 1-\frac{\delta}{2M} -1 \right] - (M-1) = M - \frac{\delta}{2} + M - \frac{\delta}{2} -M -M + 1 = 1-\delta.
        \end{align*} 
\end{itemize}
\end{proof}

\section{Approximation Rates}\label{section:rates}

In this section, we establish a stronger approximation result for a more regular class of functions and under the assumptions that the random graphs have compact support.  
We focus on permutation-equivariant functions $f:\mathcal{D}_{N,d,d'} \to \R^{N \times l}$ that are $k \in \mathbb{N}_{\ge 2}$ times continuously differentiable with bounded partial derivatives
and show that any such function $f$ can be uniformly approximated on the compact subset $\mathcal{D}_{N,d,d'}^{[0,1]} = [0,1]^{N \times d} \times [0,1]^{N \times N \times d'}$ with arbitrarily high probability by random‑feature PENNs whose FNN components have suitable architectures.

Neural-network architectures, as introduced in Definition~2.3 of \cite{guhring2020error}, specify the number of layers as well as which weights are fixed to zero and which may vary. They are defined similarly to FNNs, except that their weights are restricted to $\{0,1\}$. 
\begin{definition}[Definition~2.3, \cite{guhring2020error}]\label{def:NN_architecture}
    Let $L \in \N$ and $d_0,\ldots,d_L \in \N$. A \emph{neural-network architecture} $\mathcal{A}$ is a collection of binary matrices and vectors
    $$
    \mathcal{A} = ((\bar A^1, \bar b^1),\ldots,(\bar A^L,\bar b^L)),
    $$
    where $\bar A^l \in \{0,1\}^{d_l \times d_{l-1}}$ and $\bar b^l \in \{0,1\}^{d_l}$.
    As for neural networks, we define the complexity of an architecture by
    $L(\mathcal{A}) = L$ and $M(\mathcal{A})= \sum_{l=1}^L \|\bar A^l\|_{\ell^0} + \|\bar b^l\|_{\ell^0}$.

    Let $\Phi$ be an $L$-layer neural network with affine maps $W_l(x) = A^l x + b^l$, where $A^l \in \R^{d_l \times d_{l-1}}$ and $b^l \in \R^{d_l}$. We say that $\Phi$ \emph{has} (or is a realization of) the architecture $\mathcal{A} = ((\bar A^1, \bar b^1), \ldots, (\bar A^L, \bar b^L))$ if, for all $l=1,\ldots,L$,
    \begin{align*}
        \bar A^l_{i,j} = 0 &\implies A^l_{i,j} = 0, &i \in [d_l], j \in [d_{l-1}],\\
        \bar b^l_{i} = 0 &\implies b^l_{i} = 0, &i \in [d_l].
    \end{align*}
\end{definition}
The distinction between neural networks and neural-network architectures is practically important because the architecture must be fixed before training begins. Insights about suitable architectures therefore promise to increase the learning performance. This motivates our next result: we derive upper bounds on the number of layers and non‑zero weights required by the PENN components and characterize how these quantities scale with the approximation error $\varepsilon$ and the regularity parameter~$k$.
To formulate the theorem, we introduce the class of sufficiently regular functions as follows.
For $N,d,d',l \in \N$, $k \in \N_{\geq 2}$, and $\kappa > 0$, we define
\begin{align}
\begin{split}
   \mathcal{F}_{N,d,d', l, k,\kappa} :=  &\Bigg\{ f:\mathcal{D}_{N,d,d'} \to \R^{N \times l} \text{ perm. equiv.} \,\Bigg\vert\, \forall i \in [N], j \in [l]:  f_{i,j}  \in \mathcal{C}^k(\mathcal{D}_{N,d,d'}) \text{ and }\\
  & \phantom{\{} \forall (x,a) \in \mathcal{D}_{N,d,d'}^{[0,1]}: \max_{\substack{i \in [N], j \in [l]\\ \beta \in \N^{Nd + N^2d'}_0, |\beta|\leq k}} |D^\beta f_{i,j}(x,a)| \leq \kappa  \Bigg\}.
\end{split}
\end{align}
Then, the approximation rates are given by the following theorem. Note: we present the result specifically for random features with i.i.d.\ components that are distributed uniformly on $[0,1]$. The result can also be proven for many other distributions, in particular, any continuous distribution on $[0,1]$ or, when dropping independence and ensuring that distinct nodes are assigned distinct features, any discrete distribution with at least $N$ distinct values. The latter case includes (random permutations of) deterministic unique node identifiers for which the result even holds surely, as presented in Corollary~\ref{cor:a.s._rates}.
\begin{theorem}\label{thm:rates_for_RGNN}
Let $N,d,d',l \in \N$, $k \in \N_{\geq2}$, and $\kappa >0$.
Let $R: \Omega \to [0,1]^N$ be random features with uniformly distributed i.i.d.\ components $R_1,...,R_N$. 

For every $\delta \in (0, 1)$ there exist $c_1,c_2>0$ with $c_1 = c_1(N,d,d',l,\delta, \kappa,k)$ and $c_2= c_2(N,d,d',\delta)$ and dimensions $\bar d_\psi, d_\psi, \bar d_\alpha, d_\alpha, \bar d_\phi, d_\phi, \bar d_\rho, d_\rho \in \N$ such that the following holds.

For every $\varepsilon \in (0, \frac{1}{2})$ there exist architectures $\mathcal{A}_\rho$,$\mathcal{A}_\alpha, \mathcal{A}_\phi,$ and $\mathcal{A}_\psi$ with the property that for every $f \in \mathcal{F}_{N,d,d', l, k,\kappa}$, one can choose ReLU neural networks $\hat\rho:\R^{\bar d_\rho} \to \R^{d_\rho}$, $\hat\alpha:\R^{\bar d_\alpha} \to \R^{d_\alpha}$, 
    $\hat\phi:\R^{\bar d_\phi} \to \R^{d_\phi}$,
    and $\hat\psi:\R^{\bar d_\psi} \to \R^{d_\psi}$ each of architecture $\mathcal{A}_\rho$,$\mathcal{A}_\alpha, \mathcal{A}_\phi,\mathcal{A}_\psi$, respectively, such that the $\operatorname{PENN}$ with components $\hat\rho, \hat\alpha, \hat\phi,\hat\psi$ satisfies
\begin{align}
   \mathbb{P}\left(\max_{(x,a) \in \mathcal{D}_{N,d,d'}^{[0,1]}}\|f(x,a) - \operatorname{PENN}(x \concat R,a)\|_\infty \leq \varepsilon\right) \geq 1-\delta.
\end{align}
Moreover, the following bounds of the neural network complexity hold. 
For $\mathcal{A} \in \{\mathcal{A}_\rho, \mathcal{A}_\alpha, \mathcal{A}_\phi,\mathcal{A}_\psi\}$,
    \begin{align*}
        L(\mathcal{A}) &\leq c_1 \log_2\left(\frac{1}{\varepsilon} \right);\\
        M(\mathcal{A}) &\leq c_1 \left(\frac{1}{\varepsilon}\right)^{c_2/(k-1)} \left(\log_2\left(\frac{1}{\varepsilon}\right)\right)^2.
    \end{align*}
In particular, $c_2$ can be expressed explicitly in terms of $N,d,d',\delta$; see Equation \eqref{eq:const_c_2}. 
\end{theorem}
\begin{proof}
    Let $\delta \in (0,1)$.
    Choose $M \in \N$, with  $M\geq \frac{2N^2}{\delta}$ and $\delta' \leq \frac{1-(\frac{2-\delta}{2})^{\frac{1}{N}}}{2M}$.
    Let $\bar d_\psi, d_\psi, \bar d_\alpha, d_\alpha, \bar d_\phi, d_\phi, \bar d_\rho, d_\rho$ be as in~\eqref{eq:func_dimensions}.

    By Proposition~\ref{prop:component_error_to_total_error} and Remark~\ref{rem:simplify_constants}, there exist constants $c_1,c_2>0$ with $c_1 = c_1(N,d,d',l,\delta, \kappa,k)$ and $c_2= c_2(N,d,d',\delta)$ such that the following holds.
    For every $\varepsilon \in (0,\frac{1}{2})$, there exist neural-network architectures $\mathcal{A}_\rho, \mathcal{A}_\alpha, \mathcal{A}_\phi, \mathcal{A}_\psi$, with the property that, for every $f \in \mathcal{F}_{N,d,d',l,k,\kappa}$ and the corresponding functions $\tilde\rho:\R^{\bar d_\rho} \to \R^{d_\rho}$, $\tilde\alpha:\R^{\bar d_\alpha} \to \R^{d_\alpha}$, 
    $\tilde\phi:\R^{\bar d_\phi} \to \R^{d_\phi}$,$\tilde\psi:\R^{\bar d_\psi} \to \R^{d_\psi}$ from Proposition~\ref{prop:repr_of_GPEqui_[0,1]_approx}, there exist ReLU neural networks $\hat \rho:\R^{\bar d_\rho} \to \R^{d_\rho}$, $\hat \alpha:\R^{\bar d_\alpha} \to \R^{d_\alpha}$, 
    $\hat \phi:\R^{\bar d_\phi} \to \R^{d_\phi}$, $\hat \psi:\R^{\bar d_\psi} \to \R^{d_\psi}$, each of architecture $\mathcal{A}_\rho, \mathcal{A}_\alpha, \mathcal{A}_\phi $, and $\mathcal{A}_\psi$, respectively, for which the following holds.
    \begin{itemize}
        \item[(i)]
            For $\mathcal{A} \in \{\mathcal{A}_\rho, \mathcal{A}_\alpha, \mathcal{A}_\phi,\mathcal{A}_\psi\}$,
                \begin{align*}
                    L(\mathcal{A}) &\leq c_1 \log_2\left(\frac{1}{\varepsilon} \right);\\
                    M(\mathcal{A}) &\leq c_1 \left(\frac{1}{\varepsilon}\right)^{c_2/(k-1)} \left(\log_2\left(\frac{1}{\varepsilon}\right)\right)^2.
                \end{align*}
        \item[(ii)] Let $\tilde g:\R^{N \times (d +1) }\times \R^{N\times N \times d'}\to \R^{N \times l}$
            be defined as
            \begin{align*}
            \tilde g(x,a)_k := \tilde \rho \left(x_k, \sum_{i \in \mathcal{N}(k)}\tilde\phi \left(x_k, a_{i,k}, x_i\right), \sum_{j \in [N]}\tilde\alpha \left(x_j,\sum_{i \in \mathcal{N}(j)} \tilde\psi  \left(x_j,a_{i,j}, x_i\right)\right)\right), \qquad k \in [N]
            \end{align*}
            let $\hat g:\R^{N \times (d +1) }\times \R^{N\times N \times d'}\to \R^{N \times l}$
            be defined as
            \begin{align*}
            \hat g(x,a)_k := \hat \rho \left(x_k, \sum_{i \in \mathcal{N}(k)}\hat \phi \left(x_k, a_{i,k}, x_i\right), \sum_{j \in [N]}\hat \alpha \left(x_j,\sum_{i \in \mathcal{N}(j)} \hat \psi  \left(x_j,a_{i,j}, x_i\right)\right)\right), \qquad k \in [N]
            \end{align*}
            and let $K \subset \mathcal{D}_{N,d+1,d'}^{[0,1]} \subset \mathcal{D}_{N,d+1,d'}$ be defined as
            \begin{align*}
            K = \mathfrak{A} \cap \mathfrak{B} \cap \mathcal{D}_{N,d+1,d'}^{[0,1]}
            \end{align*}
            where, for $I_i = [\frac{i}{M}, \frac{i+1}{M})$, $i=0,\ldots, M-1$, and $h:\R \to [M]$ the function from the proof of Lemma~\ref{lemma:unif},
            \begin{align*}
            \mathfrak{A} := \left\{(x,a) \in \mathcal{D}_{N,d+1,d'} \mid \forall i,j \in [N], i \neq j: h(x_{i,d+1}) \neq h(x_{j,d+1}) \right\},
            \end{align*}
            and
            \begin{align*}
            \mathfrak{B} &:=  \left\{(x,a) \in \mathcal{D}_{N,d+1,d'} \mid \forall i\in [N]: x_{i,d+1} \notin \bigcup_{j=0}^{M}\left[\frac{j}{M}-\delta', \frac{j}{M}+\delta'\right] \right\}.
            \end{align*}
            Then
            \begin{align*}
            \max_{(x,a) \in K
            }\| \tilde g(x,a) - \hat g(x,a) \|_{\infty} \leq \varepsilon.
            \end{align*}
    \end{itemize}
    Consider the $\operatorname{PENN}$ with feedforward neural network components $\hat \rho$, $\hat \alpha$, $\hat \phi$, and $\hat \psi$ such that $\hat g = \operatorname{PENN}$.
    Since $R$ takes values in $[0,1]^N$, for any $(x,a) \in \mathcal{D}_{N,d,d'}^{[0,1]}$ it holds $\mathbb{P}((x \concat R,a) \in \mathfrak{A} \cap \mathfrak{B} \cap \mathcal{D}_{N,d+1,d'}^{[0,1]}) = \mathbb{P}((x \concat R,a) \in \mathfrak{A} \cap \mathfrak{B})$.

    Furthermore, Proposition~\ref{prop:repr_of_GPEqui_[0,1]_approx} implies that for $(x \concat r,a) \in \mathfrak{A} \cap \mathfrak{B}$ we have $f(x,a) = \tilde g(x \concat r,a)$ and $\mathbb{P}((x \concat R,a) \in \mathfrak{A} \cap \mathfrak{B}) \geq 1-\delta$.

    We conclude that,
    \begin{align*}
   &\mathbb{P}\left(\max_{(x,a) \in \mathcal{D}_{N,d,d'}^{[0,1]}}\|f(x,a) - \operatorname{PENN}(x \concat R,a)\|_\infty \leq \varepsilon\right)
   \\
   \geq &\mathbb{P}\left(\left\{\max_{(x,a) \in \mathcal{D}_{N,d,d'}^{[0,1]}}\|f(x,a) - \operatorname{PENN}(x \concat R,a)\|_\infty \leq \varepsilon\right\} \cap \left\{(x \concat R,a) \in \mathfrak{A} \cap \mathfrak{B} \cap \mathcal{D}_{N,d+1,d'}^{[0,1]}\right\}\right)
   \\
   \geq &\mathbb{P}\left(\left\{\max_{(x,a) \in \mathcal{D}_{N,d,d'}^{[0,1]}}\|f(x,a) - \operatorname{PENN}(x \concat R,a)\|_\infty \leq \varepsilon\right\} \cap \left\{(x \concat R,a) \in \mathfrak{A} \cap \mathfrak{B}\right\}\right)
   \\
   = &\mathbb{P}\left(\left\{\max_{(x,a) \in \mathcal{D}_{N,d,d'}^{[0,1]}}\|\tilde g(x,a) - \hat g(x \concat R,a)\|_\infty \leq \varepsilon\right\} \cap \left\{(x \concat R,a) \in \mathfrak{A} \cap \mathfrak{B}\right\}\right)
   \\
   \geq& \mathbb{P}\left((x \concat R,a) \in \mathfrak{A} \cap \mathfrak{B}\right) \geq 1-\delta.
\end{align*}
\end{proof}
The probability $1-\delta$ in Theorem~\ref{thm:rates_for_RGNN} reflects the likelihood that the random features are distinct and do not fall into regions where the approximation is poor (denoted in the proof by $\mathfrak{A} \cap \mathfrak{B}$). An almost sure (or sure) version of the theorem can be obtained by choosing random features such that $\mathbb{P}( (x \concat R,a) \in \mathfrak{A} \cap \mathfrak{B}) =1$ (or holds surely). The following corollary provides an example of this setting.
\begin{corollary}\label{cor:a.s._rates}
    Let $R: \Omega \to [0,1]^N$ be uniformly distributed on $\{\sigma((x_1,\ldots,x_N)) \mid \sigma \in S_N \}$ for distinct values $x_1,\ldots,x_N \in [0,1]$.
    With this choice of $R$ the approximation error bound in Theorem~\ref{thm:rates_for_RGNN} holds surely.
\end{corollary}
\begin{proof}
    Without loss of generality choose $(x_1,\ldots,x_N) = (\frac{1}{2N},\ldots,\frac{2N-1}{2N})$ and
    note that the proof can be adapted to any distinct values $x_1,\ldots,x_N \in [0,1]$.
    Then choose $M =N$, $\delta'=\frac{1}{4N}$ and follow the proof of Theorem~\ref{thm:rates_for_RGNN}. 
\end{proof}

\paragraph{Discussion}
Theorem~\ref{thm:rates_for_RGNN} has several theoretical and practical implications.
\begin{itemize}
    \item As in the previous section, the result also applies to permutation-invariant functions by the argument from~\eqref{eq:extends_to_perm_inv}.
    
    \item Similar to Theorem~\ref{thm:universality_RGNN}, Theorem~\ref{thm:rates_for_RGNN} is a universal approximation result. It guarantees the existence of PENNs that approximate certain functions when they are combined with random node features.
    Thus, PENNs with random features can serve as a strong benchmark and may even outperform state-of-the-art models.
    
    \item Additionally, Theorem~\ref{thm:rates_for_RGNN} clarifies how complex a PENN (specifically, its feedforward components) must be and how this complexity scales with the error tolerance~$\varepsilon$.
    
    \item The number of layers grows only logarithmically in $1/\varepsilon$. Hence, the depth remains moderate even as the desired accuracy increases.
    
    \item As $\varepsilon$ decreases, the number of nonzero weights is dominated by the term $\varepsilon^{-c_2/(k-1)}$, which grows polynomially rather than logarithmically. Since $c_2$ depends on $N,d,d',\delta$ (see Remark~\ref{rem:simplify_constants}), this indicates a form of curse of dimensionality with respect to the domain $\mathcal{D}_{N,d,d'}$. However, because $c_2$ is independent of $k$ and the exponent is $c_2/(k-1)$, the curse becomes less severe for functions with higher regularity~$k$. This behavior is consistent with approximation rates for feedforward neural networks, as shown, for example, in \cite{yarotsky2017error} and \cite{guhring2020error}.
\end{itemize}

\section{Conclusion}\label{sec:conclusion}
In this article, we establish a universal approximation result for PENNs combined with random node features. Our result is more general than existing universality theorems and closes several gaps in the current literature. Consequently, PENNs with random features should be considered strong baselines for graph learning tasks---either as competitive alternatives to existing state-of-the-art models or as a reference point that lends additional credibility to whichever model performs best.

Moreover, we derive upper bounds on the complexity of the feedforward components of a PENN, required to approximate functions of certain regularity. These bounds mirror known rates for standard feedforward neural networks, enhance our theoretical understanding of GNNs, and offer practical guidance for selecting suitable GNN architectures.


\appendix
\section{Auxiliary Results for Theorem~\ref{thm:universality_RGNN}}

We first show that on any set $\mathcal{D} \subseteq \mathcal{D}_{N,d,d'}$ with finite unique node features (Definition~\ref{def:set_unique_finite_feat}), every measurable permutation-equivariant function on $\mathcal{D}_{N,d,d'}$ can be represented by a suitably structured composition of measurable auxiliary functions. This extends Theorem~5.7 in \cite{gonon2024computing} and Theorem~1 in \cite{herzig2018mapping}. Combined with the universality of feedforward neural networks, this representation forms the backbone of the proof of Theorem~\ref{thm:universality_RGNN}.

\begin{proposition}\label{prop:repr_of_GPEqui}
    Let $N,d,d',l \in \N$.
    Let $\mathcal{D} \subseteq \mathcal{D}_{N, d,d'}$ be a set of graphs with finite unique node features (Def.~\ref{def:set_unique_finite_feat}). 
    For every measurable permutation-equivariant function $f: \mathcal{D}_{N, d,d'} \to \R^{N \times l}$ there exist dimensions $d_1,d_2,d_3 \in \N$ and measurable functions $\phi: \R^{2d + d'} \to \R^{d_1}$, $\psi: \R^{2d + d'} \to \R^{d_2}$, $\alpha:\R^{d+d_2} \to \R^{d_3}$ and $\rho:\R^{d+d_1+d_3} \to \R^{l}$ such that for all $(x,a) \in \mathcal{D}$ and all $k\in [N]$ it holds
    \begin{align}\label{line:repr_of_PE}
        f(x,a)_k = \rho\left(x_k, \sum_{i \in \mathcal{N}(k)}\phi \left(x_k, a_{i,k}, x_i\right), \sum_{j \in [N]}\alpha\left(x_j,\sum_{i \in \mathcal{N}(j)} \psi \left(x_j,a_{i,j}, x_i\right)\right)\right).
    \end{align}
    In particular, the right-hand side of this equation is measurable and permutation equivariant.
\end{proposition}
\begin{proof}
    First, we show that any $f: \mathcal{D}_{N, d,d'} \to \R^{N \times l}$ of the form given on the right-hand side of~\eqref{line:repr_of_PE} is permutation equivariant.
    Let $\sigma \in S_N$ be any permutation of $[N]$. 
    The function $f$ is clearly measurable, and we verify that for all $k \in [N]$,
    \begin{align*}
        f\left(\sigma \left(x,a \right)\right)_k = \sigma \left(f\left(x,a\right)\right)_k.
    \end{align*}
    It holds
    \begin{align*}
        &\forall k \in [N]: f\left(\sigma \left(x,a \right)\right)_k = \sigma \left(f\left(x,a\right)\right)_k \\
        \iff 
        &\forall k \in [N]:\sigma^{-1}\left(f\left(\sigma \left(x,a \right)\right)\right)_k = f\left(x,a\right)_k \\
        \iff
        &\forall k \in [N]:f\left(\sigma \left(x,a \right)\right)_{\sigma(k)} = f\left(x,a\right)_k .
    \end{align*}
    Furthermore, if $i,j \in [N]$ are connected by an edge in $g=(x,a)$, then $\sigma(i)$ and $\sigma(j)$ are also connected in the permuted graph $\sigma(g).$ Hence, for any $k \in [N]$,
    \begin{align*}
        \{\sigma(i) \in [N] \mid i \in \mathcal{N}_g(k)\}  = \{ i \in [N] \mid i \in \mathcal{N}_{\sigma(g)}(\sigma(k))\}.
    \end{align*}
    Then, for all $k \in [N]$ we obtain
    \begin{align*}
        f\left(\sigma \left(x,a \right)\right)_{\sigma(k)} &= \rho\left(\sigma(x)_{\sigma(k)}, \sum_{i \in \mathcal{N}_{\sigma(g)}({\sigma(k)})}\phi \left(\sigma(x)_{\sigma(k)}, \sigma(a)_{i,{\sigma(k)}}, \sigma(x)_i\right),\right.\notag \\ 
        & \phantom{=rho(\sigma(x)_{\sigma(k)}, } \left. \sum_{j \in [N]}\alpha\left(\sigma(x)_j,
        \sum_{i \in \mathcal{N}_{\sigma(g)}(j)} \psi \left(\sigma(x)_j,\sigma(a)_{i,j}, \sigma(x)_i\right)\right)\right)\\
        &= \rho\left(x_k, \sum_{i \in \mathcal{N}_{g}(k)}\phi \left(\sigma(x)_{\sigma(k)}, \sigma(a)_{\sigma(i),{\sigma(k)}}, \sigma(x)_{\sigma(i)}\right), \right.\notag \\ 
        &\phantom{=\rho(x_k, }\left.\sum_{j \in [N]}\alpha\left(\sigma(x)_{\sigma(j)},\sum_{i \in \mathcal{N}_{\sigma(g)}({\sigma(j)})} \psi \left(\sigma(x)_{\sigma(j)},\sigma(a)_{i,{\sigma(j)}}, \sigma(x)_i\right)\right)\right)\\
        &= \rho\left(x_k, \sum_{i \in \mathcal{N}_{g}(k)}\phi \left(x_k, a_{i,{k}}, x_i\right), \right.\notag \\ 
        &\phantom{=\rho(x_k, }\left.\sum_{j \in [N]}\alpha\left(\sigma(x)_{\sigma(j)},\sum_{i \in \mathcal{N}_{g}({j})} \psi \left(\sigma(x)_{\sigma(j)},\sigma(a)_{\sigma(i),{\sigma(j)}}, \sigma(x)_{\sigma(i)}\right)\right)\right)\\
        &= \rho\left(x_k, \sum_{i \in \mathcal{N}_{g}(k)}\phi \left(x_k, a_{i,{k}}, x_i\right), \sum_{j \in [N]}\alpha\left(x_j,\sum_{i \in \mathcal{N}_g(j)} \psi \left(x_j,a_{i,j}, x_i\right)\right)\right)\\
        &= f(x,a)_k.
    \end{align*}
    This shows that the function $f$ defined by the right-hand side of~\eqref{line:repr_of_PE} is indeed permutation equivariant.
    
    Next, we show that any permutation-equivariant function  
    $f$ can be represented on $\mathcal{D}$ by~\eqref{line:repr_of_PE} with appropriately chosen  
    $\rho, \phi, \alpha,$ and $\psi$.  
    The idea is to construct $\alpha$, $\psi$, and $\phi$ so that the input to $\rho$ contains all information relevant to the graph $(x,a)$ and the node $k$.  
    In fact, we will see that $\alpha$ and $\psi$ alone already encode all necessary information, so $\phi$ could in principle be chosen arbitrarily (in particular $\phi \equiv 0$).  
    We keep $\phi$ for practical reasons: when constructing GNNs based on the structure of~\eqref{line:repr_of_PE}, this additional component---although theoretically redundant---often improves empirical performance.  
    For further discussion, see Section~5.4 in \cite{gonon2024computing}.

    Since $\mathcal{D}$ provides finite unique node features, there exist $M \in \N$ with $M \ge N$ and a map  $h : \R^d \to [M]$ such that for all $(x,a) \in \mathcal{D}$ and all distinct $i,j \in [N]$ we have  $h(x_i) \neq h(x_j)$.  
    Define $\psi : \R^{d + d' + d} \to \R^{M \times M \times d'}$  to return an $M \times M \times d'$ tensor by
    \begin{align*} 
        \psi(x,y,z) = \begin{cases}y & \text{at position $h(z),h(x)$}\\ 0& \text{else}.\end{cases}
    \end{align*} 
    This construction ensures that for any input triple consisting of a destination-node feature $x$, an edge feature $y$, and a source-node feature $z$, the function $\psi$ stores the edge feature $y$ at position $(h(z), h(x))$ of the output tensor.  
    For a fixed node $j \in [N]$, summing over all neighbors $i \in \mathcal{N}(j)$ yields the tensor
    \begin{align*}
         w^{\mathcal{N}(j)} := \sum_{i \in \mathcal{N}(j)} \psi(x_j, a_{i,j}, x_i)
    \end{align*}
    which contains all information about the neighborhood of node $j$.
    Next, define
    \begin{align*}
        \alpha : \R^d \times \R^{M \times M \times d'} \to 
        \bigl(\R^{M \times (d+1)} \times \R^{M \times M \times d'}\bigr), 
        \qquad 
        \alpha(x,w) = (y,w),
    \end{align*}
    where $y$ is the $M \times (d+1)$ matrix
    \begin{align*}
        y = \begin{cases}
                x \concat 1& \text{in row $h(x)$}\\0& \text{else.}
            \end{cases}
    \end{align*}
    The matrix $y$ stores the node feature $x$ in row $h(x)$, and the appended $1$-flag later allows us to identify whether the node appears in the original graph.
    Summing over all $j \in [N]$ yields
    \[
        (z,w) := \sum_{j=1}^N \alpha\!\left(x_j, w^{\mathcal{N}(j)}\right)
        = \sum_{j=1}^N \alpha\!\left(x_j, \sum_{i \in \mathcal{N}(j)} \psi(x_j, a_{i,j}, x_i)\right),
    \]
    which contains all node and edge features of $(x,a) \in \mathcal{D}$, up to permutation and with additional zero entries.
    
    Finally, we construct a suitable map $\rho:\R^d \times \R^{d_1} \times \R^{M(d+1) + M^2d'} \to \R^l$, which we write as composition $\rho = \rho^B \circ \rho^A$.
    The first component, $\rho^A$, removes all nodes—and their associated edges—that were not flagged as originating from the original graph. As a result, $\rho^A$ reduces the tensors $(z,w)$ from size $M \times (d+1)$ and $M^2 \times d'$ to $N \times (d+1)$ and $N^2 \times d'$, respectively.

    Simultaneously, $\rho^A$ also keeps track of the re-indexing.
    It returns a vector $b \in (\{0\} \cup [N])^M$, where $b_i = n \in [N]$ indicates that—after the reduction—node $i$ now occupies position $n$, while $b_i = 0$ means that node $i$ was removed.   
    Thus, we define $\rho^A:\R^d \times \R^{d_1} \times \R^{M(d+1) + M^2d'} \to \R^d \times \R^{d_1} \times \R^{N(d+1) + N^2d'}\times \R^M$ by 
    $$
    \rho^A(x, \cdot, (z,w)) := (x, \cdot, (z',w'), b),
    $$
    where, if $(\iota_1<...<\iota_N)$ are the indices of the nonzero rows of $z$,
    \begin{align*}
        z'_i &= z_{\iota_i} & i\in[N] ,
        \\
        w'_{i,j} &= w_{\iota_i,\iota_j} &i,j\in[N],
        \\
        b_i &= \sum_{k \in [N]}k\mathds{1}_{\{\cdot =\iota_k\}}(i) &i\in [M].
    \end{align*}
    For general inputs with more or fewer than $N$ nonzero rows, we allow any measurable extension of $\rho^A$.

    We observe that the reduced tensors $(z',w')$ obtained from $(z,w)$ 
    correspond to the permuted graph $(\sigma(x), \sigma(a))$, where $\sigma \in S_N$ is the permutation defined by $\sigma(i) := b_{h(x_i)}$ for $i \in [N]$.
    The function $\rho^B: \R^d \times \R^{d_1} \times \R^{N(d+1) + N^2d'}\times \R^M \to \R^l$ simply extracts the right component of $f$, namely
    $$
    \rho^B(x, \cdot, (z',w'), b) := f(z'_{[:,1:d]},w')_{b_{h(x)}},
    $$
    where $z'_{[:,1:d]}$ denotes all rows and columns $1$ through $d$ of $z'$.
    With $\rho := \rho^B \circ \rho^A$, we obtain 
    \begin{align*}
        &\phantom{=} \rho\left(x_k, \sum_{i \in \mathcal{N}(k)}\phi \left(x_k, a_{i,k}, x_i\right), \sum_{j \in [N]}\alpha\left(x_j,\sum_{i \in \mathcal{N}(j)} \psi \left(x_j,a_{i,j}, x_i\right)\right)\right) \\
    &=f(z'_{[:,1:d]},w')_{b_{h(x_k)}} = \sigma(f(x,a))_{\sigma(k)} = f(x,a)_k. 
    \end{align*}
    This completes the proof.
\end{proof}
The next proposition establishes a universality result for functions of the form
\eqref{line:repr_of_PE}.  
We show that by approximating $\rho$, $\alpha$, $\psi$, and $\phi$ with neural networks  
$\hat\rho$, $\hat\alpha$, $\hat\psi$, and $\hat\phi$, the resulting composition can be made arbitrarily close to the original function in probability.  
This follows directly from the well-known universality of feedforward neural networks.
The result is similar to Proposition~5.9 in \cite{gonon2024computing}.  
The only differences are that our setting includes the additional component $\phi$ and uses neighborhood aggregation rather than summation over all $i \neq j$.  
Although the proof closely follows the argument in \cite{gonon2024computing}, we provide a complete version here for clarity.
\begin{proposition}\label{prop:universality_of_composition}
    Let $\rho, \phi, \alpha, \psi$ be measurable functions $\phi: \R^{2d_1+d_2} \to \R^{d_3}, \psi:  \R^{2d_1+d_2} \to \R^{d_4},  \alpha: \R^{d_1+d_4}\to \R^{d_5},$ and $\rho: \R^{d_1 + d_3+d_5}\to \R^l,$ for arbitrary dimensions $d_1,...,d_5\in \N$.
    Let $g:\mathcal{D}_{N,d_1,d_2} \to \R^{N \times l}$ be defined by
    \begin{align*}
       g(x,y)_k :=\rho\left(x_k, \sum_{i \in \mathcal{N}(k)}\phi \left(x_k, y_{i,k}, x_i\right), \sum_{j \in [N]}\alpha\left(x_j,\sum_{i \in \mathcal{N}(j)} \psi \left(x_j,y_{i,j}, x_i\right)\right)\right),
    \end{align*}
    for all $k \in [N]$.
    
    Then, for any probability measure $\mu$ on $\R^{N \times d_1}\times \R^{N\times N \times d_2}$, and any $\varepsilon,\bar \varepsilon>0$, there exist feedforward neural networks $\hat \rho, \hat \phi, \hat \alpha, \hat \psi$ with Lipschitz, bounded activation function such that the function $\hat g:\R^{N \times d_1}\times \R^{N\times N \times d_2} \to \R^{N \times l}$, defined by
    \begin{align*}
        \hat g (x,y)_k := 
           \hat\rho\left(x_k,\sum_{i \in \mathcal{N}(k)} \hat\phi(x_k, y_{i,k},x_i) ,\sum_{j\in[N]}\hat\alpha(x_j,\sum_{i \in \mathcal{N}(j)} \hat\psi(x_j, y_{i,j},x_i))\right), \qquad k \in [N],
    \end{align*}
    satisfies
    \begin{align*}
        \mu \left( \{(x,y) \in \R^{N \times d_1}\times \R^{N\times N \times d_2}: \|g(x,y) - \hat g(x,y)\| \geq \bar \varepsilon  \}\right) \leq \varepsilon.
    \end{align*}
\end{proposition}
\begin{proof}
    Let $\varepsilon,\bar \varepsilon>0.$
    Choose $\hat \rho$ such that for all $k \in [N]$
    \begin{align}
        &\mu \left( \bigg\{
        \|\rho(x_k, \sum_{i \in \mathcal{N}(k)} \phi(x_k, y_{i,k},x_i), \sum_{j\in[N]} \alpha(x_j, \sum_{i\in \mathcal{N}(j)} \psi(x_j,y_{i,j},x_i))) \right. \phantom{)} \nonumber \\ 
        \label{line:rho_prob}
        &\phantom{\mu(\{
        \| } \left. - \hat\rho(x_k, \sum_{i \in \mathcal{N}(k)} \phi(x_k, y_{i,k},x_i) ,\sum_{j\in[N]} \alpha(x_j, \sum_{i \in \mathcal{N}(j)} \psi(x_j,y_{i,j},x_i)))\| \geq \frac{\bar \varepsilon}{N}\bigg\}\right)\leq \frac{\varepsilon}{4N}.
    \end{align}
    Note that $\hat \rho$ is  Lipschitz continuous with some constant $K_{\hat \rho}>0.$
    Choose $\hat \phi$ such that for all $i,k \in [N],$
    \begin{align}\label{line:phi_prob}
        \mu(\{\|\phi(x_k,y_{ik},x_i) - \hat \phi(x_k,y_{ik},x_i) \| \geq \frac{\bar \varepsilon}{N^2K_{\hat \rho}} \}) \leq \frac{\varepsilon}{4N^2},
    \end{align}
    and $\hat \alpha$ such that for all $j \in [N],$
    \begin{align}\label{line:alpha_prob}
        \mu(\{\|\alpha(x_j,\sum_{i \in \mathcal{N}(j)}\psi(x_j,y_{i,j},x_i)) - \hat \alpha(x_j,\sum_{i \in \mathcal{N}(j)}\psi(x_j,y_{i,j},x_i)) \| \geq \frac{\bar \varepsilon}{N^2K_{\hat \rho}} \}) \leq \frac{\varepsilon}{4N}.
    \end{align}
    Finally, also $\hat \alpha$ is  Lipschitz continuous with $K_{\hat \alpha}>0$ and we can choose $\hat \psi$ such that for all $i,j \in [N]$
    \begin{align}\label{line:psi_prob}
        \mu(\{\|\psi(x_j,y_{i,j},x_i) - \hat \psi(x_j,y_{i,j},x_i)\| \geq \frac{\bar \varepsilon}{N^3K_{\hat \rho}K_{\hat \alpha}} \}) \leq \frac{\varepsilon}{4N^2}.
    \end{align}
    All these choices are valid due to Proposition~\ref{thm:universality_of_measurable_thomas}.

    Next, we establish an inclusion that will help us to bound the probability.
    \begin{align*}
        &\{\|g(x,y)-\hat g(x,y)\| \geq \bar \varepsilon\}\\
        & = \left\{ \left|
        \begin{array}{l}
              \rho(x_1, \sum_{i \in \mathcal{N}(1)} \phi(x_1, y_{i,1},x_i), \sum_{j\in[N]} \alpha(x_j, \sum_{i \in \mathcal{N}(j)} \psi(x_j,y_{i,j},x_i)))\\ - \hat\rho(x_1, \sum_{i \in \mathcal{N}(1)} \hat \phi(x_1, y_{i,1},x_i), \sum_{j\in[N]} \hat\alpha(x_j, \sum_{i \in \mathcal{N}(j)} \hat\psi(x_j,y_{i,j},x_i)))\\
              \multicolumn{1}{c}{\vdots\hskip8em\relax}\\
             \rho(x_N,\sum_{i \in \mathcal{N}(N)} \phi(x_N, y_{i,N},x_i), \sum_{j\in[N]} \alpha(x_j, \sum_{i \in \mathcal{N}(j)} \psi(x_j,y_{i,j},x_i)))\\ {}- \hat\rho(x_N,\sum_{i \in \mathcal{N}(N)} \hat\phi(x_N, y_{i,N},x_i), \sum_{j\in[N]} \hat\alpha(x_j, \sum_{i \in \mathcal{N}(j)} \hat\psi(x_j,y_{i,j},x_i)))
        \end{array}  \right\| \geq \bar \varepsilon \right\}\ \\
        \subseteq &\bigcup_{k \in [N]} \bigg \{\|\rho(x_k, \sum_{i \in \mathcal{N}(k)} \phi(x_k,              y_{i,k},x_i), \sum_{j\in[N]} \alpha(x_j, \sum_{i \in \mathcal{N}(j)} \psi(x_j,y_{i,j},x_i)))             \\
            &\phantom{\bigcup_{k \in [N]} \bigg \{\|}\quad - \hat\rho(x_k, \sum_{i \in \mathcal{N}(k)} \hat \phi(x_k, y_{i,k},x_i), \sum_{j\in[N]}      
            \hat\alpha(x_j, \sum_{i \in \mathcal{N}(j)} \hat\psi(x_j,y_{i,j},x_i)))\| \geq \frac{\bar \varepsilon}{N}\bigg \} \\
        \subseteq &\bigcup_{k \in [N]}\left(
            \bigg \{\|\rho(x_k, \sum_{i \in \mathcal{N}(k)} \phi(x_k, y_{i,k},x_i), \sum_{j\in[N]} \alpha(x_j, \sum_{i \in \mathcal{N}(j)} \psi(x_j,y_{i,j},x_i))) \right. \\
            &\phantom{\bigcup_{k \in [N]}(
            \bigg \{\|} \quad - \hat\rho(x_k, \sum_{i \in \mathcal{N}(k)} \phi(x_k, y_{i,k},x_i), \sum_{j\in[N]} \alpha(x_j, \sum_{i \in \mathcal{N}(j)} \psi(x_j,y_{i,j},x_i)))\| \geq \frac{\bar \varepsilon}{N} \bigg \} \\
            &\qquad \cup \bigg \{\|\hat\rho(x_k, \sum_{i \in \mathcal{N}(k)} \phi(x_k, y_{i,k},x_i), \sum_{j\in[N]} \alpha(x_j, \sum_{i \in \mathcal{N}(j)} \psi(x_j,y_{i,j},x_i)))\\
            &\phantom{\qquad \cup \bigg \{\|} - \hat\rho(x_k, \sum_{i \in \mathcal{N}(k)} \hat \phi(x_k, y_{i,k},x_i), \sum_{j\in[N]} \alpha(x_j, \sum_{i \in \mathcal{N}(j)} \psi(x_j,y_{i,j},x_i)))\| \geq \frac{\bar \varepsilon}{N}\bigg \}\\
            &\qquad\cup \bigg \{\|\hat\rho(x_k, \sum_{i \in \mathcal{N}(k)} \hat \phi(x_k, y_{i,k},x_i), \sum_{j\in[N]} \alpha(x_j, \sum_{i \in \mathcal{N}(j)} \psi(x_j,y_{i,j},x_i)))\\
            &\phantom{\qquad\cup \bigg \{\|}- \hat\rho(x_k, \sum_{i \in \mathcal{N}(k)} \hat \phi(x_k, y_{i,k},x_i), \sum_{j\in[N]} \hat\alpha(x_j, \sum_{i \in \mathcal{N}(j)} \psi(x_j,y_{i,j},x_i)))\| \geq \frac{\bar \varepsilon}{N}\bigg \}\\
            & \qquad \cup \bigg \{\|\hat\rho(x_k, \sum_{i \in \mathcal{N}(k)} \hat \phi(x_k, y_{i,k},x_i), \sum_{j\in[N]} \hat\alpha(x_j, \sum_{i \in \mathcal{N}(j)} \psi(x_j,y_{i,j},x_i)))\\
            & \left.\phantom{\qquad \cup \bigg \{\|} - \hat\rho(x_k, \sum_{i \in \mathcal{N}(k)} \hat \phi(x_k, y_{i,k},x_i), \sum_{j\in[N]} \hat\alpha(x_j, \sum_{i \in \mathcal{N}(j)} \hat\psi(x_j,y_{i,j},x_i)))\| \geq \frac{\bar \varepsilon}{N}\bigg \}\right).
    \end{align*}
    The probability of the first set in this expression is bounded due to \eqref{line:rho_prob}. For the second set, we observe that for every $k \in [N]$,
    \begin{align*}
        &\bigg \{\|\hat\rho(x_k, \sum_{i \in \mathcal{N}(k)} \phi(x_k, y_{i,k},x_i),\sum_{j\in[N]} \alpha(x_j, \sum_{i \in \mathcal{N}(j)} \psi(x_j,y_{i,j},x_i))) -  \\
        &\phantom{\bigg \{\|} - \hat\rho(x_k, \sum_{i \in \mathcal{N}(k)} \hat\phi(x_k, y_{i,k},x_i),\sum_{j\in[N]} \alpha(x_j, \sum_{i \in \mathcal{N}(j)} \psi(x_j,y_{i,j},x_i)))\| \geq \frac{\bar \varepsilon}{N}\bigg \} \\
        &\subseteq
            \bigg \{\|\sum_{i\in\mathcal{N}(k)} \phi(x_k,y_{i,k},x_i) -  \hat\phi(x_k,y_{i,k},x_i)\| \geq \frac{\bar \varepsilon}{NK_{\hat \rho}}\bigg \}\\
        &\subseteq \bigcup_{i \in [N]}\bigg \{\|\phi(x_k,y_{i,k},x_i) -  \hat\phi(x_k,y_{i,k},x_i)\| \geq \frac{\bar \varepsilon}{N^2K_{\hat \rho}}\bigg \}.
    \end{align*}
    A bound of its probability follows from~\eqref{line:phi_prob}.   
    Similarly, for every $k \in [N]$, we can bound the probability of the third set using~\eqref{line:alpha_prob}, since      
    \begin{align*}
        &\bigg\{\|\hat\rho(x_k, \sum_{i \in \mathcal{N}(k)} \hat\phi(x_k, y_{i,k},x_i),\sum_{j\in[N]} \alpha(x_j, \sum_{i \in \mathcal{N}(j)} \psi(x_j,y_{i,j},x_i))) - \\
        &\phantom{\bigg\{\|} - \hat\rho(x_k, \sum_{i \in \mathcal{N}(k)} \hat\phi(x_k, y_{i,k},x_i),\sum_{j\in[N]} \hat\alpha(x_j, \sum_{i \in \mathcal{N}(j)} \psi(x_j,y_{i,j},x_i)))\| \geq \frac{\bar \varepsilon}{N}\bigg\}\\
        \subseteq 
            &\bigg\{\|\sum_{j\in[N]} \alpha(x_j,\sum_{i \in \mathcal{N}(j)} \psi(x_j,y_{i,j},x_i)) -  \sum_{j\in[N]}\hat\alpha(x_j,\sum_{i \in \mathcal{N}(j)} \psi(x_j,y_{i,j},x_i))\| \geq \frac{\bar \varepsilon}{NK_{\hat \rho}}\bigg\}\\
        \subseteq &\bigcup_{j \in [N]}\bigg\{\|\alpha(x_j,\sum_{i \in \mathcal{N}(j)} \psi(x_j,y_{i,j},x_i)) -  \hat\alpha(x_j,\sum_{i \in \mathcal{N}(j)} \psi(x_j,y_{i,j},x_i))\| \geq \frac{\bar \varepsilon}{N^2K_{\hat \rho}}\bigg\}.
    \end{align*}
    Note that this last expression no longer depends on $k$, so it does not need to appear in the union over all $k\in [N]$.
    Similarly, for the final set we obtain
    \begin{align*}
            &\bigg \{\|\hat\rho(x_k, \sum_{i \in \mathcal{N}(k)} \hat\phi(x_k, y_{i,k},x_i),\sum_{j\in[N]} \hat\alpha(x_j, \sum_{i \in \mathcal{N}(j)} \psi(x_j,y_{i,j},x_i))) - \\
            &\phantom{\bigg \{\|} - \hat\rho(x_k, \sum_{i \in \mathcal{N}(k)} \hat\phi(x_k, y_{i,k},x_i),\sum_{j\in[N]} \hat\alpha(x_j, \sum_{i \in \mathcal{N}(j)} \hat\psi(x_j,y_{i,j},x_i)))\| \geq \frac{\bar \varepsilon}{N}\bigg \}\\
            &\subseteq\bigg \{\|\sum_{j\in[N]} \hat\alpha(x_j,\sum_{i \in \mathcal{N}(j)} 
            \psi(x_j,y_{i,j},x_i)) -  \sum_{j\in[N]}\hat\alpha(x_j,\sum_{i \in \mathcal{N}(j)} \hat\psi(x_j,y_{i,j},x_i))\| \geq \frac{\bar \varepsilon}{NK_{\hat \rho}}\bigg \}\\
            &\subseteq\bigcup_{j \in [N]}\bigg \{\|\hat\alpha(x_j,\sum_{i \in \mathcal{N}(j)} \psi(x_j,y_{i,j},x_i)) -  \hat\alpha(x_j,\sum_{i \in \mathcal{N}(j)} \hat\psi(x_j,y_{i,j},x_i))\| \geq \frac{\bar \varepsilon}{N^2K_{\hat \rho}}\bigg \} \\
            &\subseteq \bigcup_{j \in [N]}\bigg \{\|\sum_{i \in \mathcal{N}(j)} \psi(x_j,y_{i,j},x_i)) -  \sum_{i \in \mathcal{N}(j)} \hat\psi(x_j,y_{i,j},x_i))\| \geq \frac{\bar \varepsilon}{N^2K_{\hat \rho}K_{\hat \alpha}}\bigg \} \\
            &\subseteq \bigcup_{i,j \in [N]}\bigg \{\|\psi(x_j,y_{i,j},x_i)) -\hat\psi(x_j,y_{i,j},x_i))\| \geq \frac{\bar \varepsilon}{N^3K_{\hat \rho}K_{\hat \alpha}}\bigg \} ,
    \end{align*}
    and can bound the probability with~\eqref{line:psi_prob}.
    From this we can conclude that 
    \begin{align*}
        &\mu(\{\|g(x,y)-\hat g(x,y)\| \geq \bar \varepsilon\})\\
        &\leq \phantom{++}\sum_{k \in [N]}\mu\bigg(\bigg \{\|\rho(x_k, \sum_{i \in \mathcal{N}(k)} \phi(x_k, y_{i,k},x_i), \sum_{j\in[N]} \alpha(x_j, \sum_{i \in \mathcal{N}(j)} \psi(x_j,y_{i,j},x_i))) \\
            &\phantom{\sum_{k \in [N]}\mu\bigg(\bigg \{\|} - \hat\rho(x_k, \sum_{i \in \mathcal{N}(k)} \phi(x_k, y_{i,k},x_i), \sum_{j\in[N]} \alpha(x_j, \sum_{i \in \mathcal{N}(j)} \psi(x_j,y_{i,j},x_i)))\| \geq \frac{\bar \varepsilon}{N} \bigg \} \bigg) \\
            &+\sum_{k,i \in [N]}\mu\bigg(\bigg \{\|\phi(x_k,y_{i,k},x_i) -  \hat\phi(x_k,y_{i,k},x_i)\| \geq \frac{\bar \varepsilon}{N^2K_{\hat \rho}}\bigg \} \bigg) \\
            &+\sum_{j \in [N]}\mu\bigg(\bigg\{\|\alpha(x_j,\sum_{i \in \mathcal{N}(j)} \psi(x_j,y_{i,j},x_i)) -  \hat\alpha(x_j,\sum_{i \in \mathcal{N}(j)} \psi(x_j,y_{i,j},x_i))\| \geq \frac{\bar \varepsilon}{N^2K_{\hat \rho}}\bigg\} \bigg) \\
            &+\sum_{i,j \in [N]}\mu\bigg(\bigg\{\|\psi(x_j,y_{i,j},x_i)) -\hat \psi(x_j,y_{i,j},x_i))\| \geq \frac{\bar \varepsilon}{N^3K_{\hat \rho}K_{\hat \alpha}}\bigg\} \bigg)\\
        &\leq N\frac{\varepsilon}{4N} + N^2\frac{\varepsilon}{4N^2} + N\frac{\varepsilon}{4N} + N^2 \frac{\varepsilon}{4N^2}= \varepsilon.
    \end{align*}
\end{proof}

\section{Feedforward Neural Networks}
Feedforward neural networks (FNNs) are the building blocks of the GNNs considered in this work (PENNs). This section recalls the definition of feedforward neural networks and collects auxiliary results used in the proofs.
\begin{definition}\label{def:FNN}
    Let $L, d_0,\ldots,d_L \in \mathbb{N}$ and let $\varrho:\R \to \R$ be an activation function. A function $\Phi:\R^{d_0} \to \R^{d_L}$ is called a \emph{feedforward neural network} with $L$ layers and activation function $\varrho$ if
    \begin{align*}
        \Phi := W_L \circc F_{L-1} \circc \ldots \circc F_1, 
    \end{align*}
    where the functions $W_l:\R^{d_{l-1}} \to \R^{d_l}, l=1,\ldots,L$ are affine linear,
    \begin{align*}
        W_l(x) := A^lx + b^l,
    \end{align*}
    with $A^l \in \R^{d_{l}\times d_{l-1}}$ and $b^l \in \R^{d_l}$, and where 
    $F_l = \varrho \circc W_l,l=1,\ldots,L-1$. The activation function $\varrho$ is applied componentwise, that is, for $y \in \R^m, m \in \N,$ $\varrho(y)=(\varrho\left(y^1\right), \ldots, \varrho\left(y^m\right))^T \in \R^m$.
    We measure the complexity of a feedforward neural network by its number of layers $L(\Phi) := L$, 
    and by the number of nonzero weights $M(\Phi) := \sum_{l=1}^L \|A_l\|_{\ell^0} + \|b_l \|_{\ell^0}$.
\end{definition} 
One popular activation function is defined as follows.
\begin{definition}
    The function
    $$
    \varrho: \R \to \R, \quad x \mapsto \frac{1}{1+ e^{-x}},
    $$
    is called sigmoid activation function.
\end{definition}
The sigmoid activation function is an example of a bounded activation function. 
Although bounded activations can contribute to the so‑called vanishing‑gradient problem, they are widely used in theoretical analyses and in practice. The following result, for example, is formulated for bounded activation functions.
\begin{proposition}[Theorem B.1, \cite{biagini2023neural}] \label{thm:universality_of_measurable_thomas}
    Assume the activation function $\varrho$ is bounded and non-constant. 
    Let $f:\R^d \to \R^m$ be a measurable function and let $\mathbb{P}$ be a probability measure on $\R^d.$ Then, for any $\varepsilon, \bar \varepsilon >0$, there exists a neural network $\Phi:\R^d \to \R^m$ such that
    \begin{align*}
        \mathbb{P} ( \{x \in \R^d: \|f(x) - \Phi(x)\| \geq \bar \varepsilon\}) \leq \varepsilon.
    \end{align*}
\end{proposition}
Another widely used activation function is the rectified linear unit (ReLU).
\begin{definition}
    The function
    $$
    \varrho: \R \to \R, \quad x \mapsto \max (0,x),
    $$
    is called \emph{ReLU (Rectified Linear Unit)} activation function.
\end{definition}
Unlike the sigmoid function, ReLU is unbounded. Its piecewise linear structure and piecewise constant derivative make it computationally efficient in practice. Moreover, for ReLU neural networks (feedforward networks with ReLU activation), we can perform the following basic operation.
\begin{proposition}[Proposition 2.3, \citet{opschoor2020deep}]\label{prop:parallelize_relu_L}
    Let $d_1,d_2,d_3,L \in \N$. 
    Consider two $L$-layer ReLU neural networks $\Phi_1:\R^{d_1}\to \R^{d_2}$ and $\Phi_2:\R^{d_1}\to\R^{d_3}$. 
    Then there exists a ReLU neural network $\mathbf{P}(\Phi_1,\Phi_2):\R^{d_1}\to \R^{d_2+d_3}$ such that, for all $x \in \R^{d_1}$,
    \begin{align*}
        \mathbf{P}(\Phi_1,\Phi_2)(x) = \begin{pmatrix}
            \Phi_1(x)\\\Phi_2(x)
        \end{pmatrix}.
    \end{align*}
    We call $\mathbf{P}(\Phi_1,\Phi_2)$ the $L$-layer parallelization of $\Phi_1$ and $\Phi_2$. 
    In particular, $L(\mathbf{P}(\Phi_1,\Phi_2)) = L$ and $M(\mathbf{P}(\Phi_1,\Phi_2)) = M(\Phi_1) + M(\Phi_2)$.
\end{proposition}
\begin{remark}\label{rem:parallelize_multi_relu_L}
    By induction, one can parallelize any finite number of ReLU neural networks with identical layer structure. In this case,
    $L(\mathbf{P}(\Phi_1,\ldots,\Phi_n)) = L$ and $M(\mathbf{P}(\Phi_1,\ldots,\Phi_n)) = \sum_i M(\Phi_i)$.
\end{remark}

\begin{remark}\label{rem:parallelization_of_nn_architectures}
    Since neural-network architectures (see Definition~\ref{def:NN_architecture}) are simply neural networks whose weights are restricted to $\{0,1\}$, their parallelization can be defined in direct analogy to Proposition~\ref{prop:parallelize_relu_L}.
\end{remark}
The following result shows that, for certain classes of sufficiently regular functions, one can construct a single neural-network architecture such that every function in the class can be well approximated by a realization of this architecture. To state the result, we introduce the following function class. For $n \in \N_{\ge 2}$, $d \in \N$, $1 \le p \le \infty$, and $B > 0$, define
$$
\widetilde{\mathcal{F}}_{n, d, p, B}:=\left\{f \in W^{n, p}\left((0,1)^d\right):\|f\|_{W^{n, p}\left((0,1)^d\right)} \leq B\right\}.
$$
Then the following Proposition holds.
\begin{proposition}[Corollary 4.2, \citet{guhring2020error}]\label{prop:rates_for_NN}
    Let $d \in \mathbb{N}$, $n \in \mathbb{N}_{\ge 2}$, $1 \le p \le \infty$, $B>0$, and $0 \le s \le 1$. Then there exists a constant $c = c(d,n,p,B,s) > 0$ with the following properties. For any $\varepsilon \in (0,1/2)$, there exists a neural-network architecture
    $\mathcal{A}_{\varepsilon}=\mathcal{A}_{\varepsilon}(d, n, p, B, s, \varepsilon)$ with $d$-dimensional input and one-dimensional output such that, for every $f \in \widetilde{\mathcal{F}}_{n, d, p, B}$, there exists a ReLU neural network $\Phi_{\varepsilon}^f$ that realizes architecture $\mathcal{A}_{\varepsilon}$ and satisfies
    $$
        \left\|\Phi_{\varepsilon}^f-f\right\|_{W^{s, p}\left((0,1)^d\right)} \leq \varepsilon.
    $$
    Moreover:
    \begin{itemize}
        \item[(i)] $L\left(\mathcal{A}_{\varepsilon}\right) \leq c \cdot \log _2\left(\varepsilon^{-n /(n-s)}\right)$;
        \item[(ii)] $M\left(\mathcal{A}_{\varepsilon}\right) \leq c \cdot \varepsilon^{-d /(n-s)} \cdot \left(\log _2\left(\varepsilon^{-n /(n-s)}\right)\right)^2$.
    \end{itemize}
\end{proposition}

\begin{remark}
    Here, we focus on the integer case $s = 1$ (with $p = \infty$), approximating functions and their classical derivatives. The case $0 < s < 1$, corresponds to fractional Sobolev spaces (Sobolev--Slobodeckij spaces).
\end{remark}
The ReLU neural networks $\Phi_\varepsilon^f$ of architecture $\mathcal {A}_\varepsilon$ from Proposition~\ref{prop:rates_for_NN} are designed to approximate functions $f\colon\R^d \to \R$ and their partial derivatives on the domain $(0,1)^d$. We do not know how $\Phi_\varepsilon^f$ behaves outside of $[0,1]^d$. The following corollary shows that we can control the global Lipschitz constant by adjusting constant $c$ and the network $\Phi_\varepsilon^f$ from Proposition~\ref{prop:rates_for_NN}.

\begin{corollary}\label{cor:rates_for_NN_global_lip}
    Let $d \in \mathbb{N}$, $n \in \mathbb{N}_{\ge 2}$, $p = \infty$, $B>0$, and $s=1$. Then there exists a constant $\tilde c = \tilde c(d,n,p,B,s) > 0$ with the following properties. For any $\varepsilon \in (0,1/2)$, there exists a neural-network architecture
    $\mathcal{A}_{\varepsilon}=\mathcal{A}_{\varepsilon}(d, n, p, B, s, \varepsilon)$ with $d$-dimensional input and one-dimensional output such that, for every $f \in \widetilde{\mathcal{F}}_{n, d, p, B}$, there exists a ReLU neural network $\tilde{\Phi}_{\varepsilon}^f$ that realizes architecture $\mathcal{A}_{\varepsilon}$ and satisfies
    $$
        \left\|\tilde{\Phi}_{\varepsilon}^f-f\right\|_{W^{s, p}\left((0,1)^d\right)} \leq \varepsilon.
    $$
    Moreover:
    \begin{itemize}
        \item[(i)] $L\left(\mathcal{A}_{\varepsilon}\right) \leq \tilde c \cdot \log _2\left(\varepsilon^{-n /(n-s)}\right)$;
        \item[(ii)] $M\left(\mathcal{A}_{\varepsilon}\right) \leq \tilde c \cdot \varepsilon^{-d /(n-s)} \cdot \left(\log _2\left(\varepsilon^{-n /(n-s)}\right)\right)^2$;
        \item[(iii)] $\tilde{\Phi}_{\varepsilon}^f$ is globally Lipschitz continuous with respect to $\|\cdot\|_\infty$ with Lipschitz constant $d \cdot (B + \varepsilon)$.
    \end{itemize}       
\end{corollary}
\begin{proof}
    Let $c>0$ and ${\Phi}_{\varepsilon}^f$ be the constant and ReLU neural network from Proposition~\ref{prop:rates_for_NN}.

    Let $T:\R^d \to [0,1]^d$ be the two layer ReLU neural network $T(x_1,...,x_d) = (t(x_1),\ldots,t(x_d))$ with
    \begin{align*}
        t(x) = \operatorname{ReLU}(x) - \operatorname{ReLU}(x-1) = \begin{pmatrix}1\\-1\end{pmatrix} \operatorname{ReLU}\left[\begin{pmatrix}1\\1\end{pmatrix}x + \begin{pmatrix}0\\-1\end{pmatrix} \right] + \begin{pmatrix}0\\0\end{pmatrix}.
    \end{align*}
    It can be easily verified that 
    \begin{align*}
        t(x) = \begin{cases}
            0,\quad x<0,\\
            x,\quad0\leq x\leq1,\\
            1,\quad x>1,
        \end{cases}
    \end{align*}
    and that 
    $$
    |t(x)-t(y)| \leq |x-y|,
    $$
    for all $x,y \in \R$. Therefore, $T$ is $1$-Lipschitz continuous with respect to $\|\cdot\|_\infty$.
    
    Since $\operatorname{ReLU}(T(x)) = T(x)$, the composition $\tilde{\Phi}_\varepsilon^f = \Phi_\varepsilon^f \circc T$ is an $L(\mathcal{A}_\varepsilon)+2$ layer ReLU neural network with at most $M(\mathcal{A}_\varepsilon)+5d$ nonzero weights, for which 
    $$
        \tilde{\Phi}_\varepsilon^f(x) = \Phi_\varepsilon^f(x), \quad x \in [0,1]^d,
    $$
    and since $T(T(x))= T(x)$,
    $$
        \tilde{\Phi}_\varepsilon^f(T(x)) = \tilde{\Phi}_\varepsilon^f(x), \quad x \in \R^d.
    $$
    For $n\geq 2$, $s=1$, and $\varepsilon<1/2$, there exists a constant $\tilde c>0$ such that
    $$
        L(\mathcal{A}_\varepsilon)+2 \leq c \cdot \log_2\left(\varepsilon^{-n/(n-s)}\right) +2 \leq \tilde c \cdot \log_2\left(\varepsilon^{-n/(n-s)}\right),
    $$
    and
    $$
        M(\mathcal{A}_\varepsilon)+5d \leq c \cdot \varepsilon^{-d/(n-s)} \cdot \left(\log_2\left(\varepsilon^{-n/(n-s)}\right)\right)^2 +5d \leq \tilde c \cdot \varepsilon^{-d/(n-s)} \cdot \left(\log_2\left(\varepsilon^{-n/(n-s)}\right)\right)^2.
    $$
    Therefore, with this choice as constant we may choose $\tilde{\Phi}_\varepsilon^f$ in place of $\Phi_\varepsilon^f$, thereby controlling the Lipschitz constant globally, as we show next.
    
    Recall that $\tilde{\Phi}_\varepsilon^f$ and ${\Phi}_\varepsilon^f$ coincide on $[0,1]^d$ and approximate, with error $\varepsilon>0$, some function $f$ and its partial derivatives, which are bounded by $B>0$, on $(0,1)^d$. 
    Therefore, $\tilde{\Phi}_\varepsilon^f$ has partial derivatives bounded by $B+\varepsilon$ almost everywhere\footnote{Since a ReLU neural network is a piecewise affine-linear function, it may have differing left- and right-sided partial derivatives at its kinks.} on $(0,1)^d$.
    In general, a continuous function $h:\R^d\to\R^m$, $h=(h_1,\ldots,h_m)$, whose first-order partial derivatives are bounded (almost everywhere) in absolute value by some $B>0$, is Lipschitz continuous with respect to $\|\cdot\|_\infty$ with constant $d\cdot B$. By this argument $\tilde{\Phi}_\varepsilon^f$ is $d\cdot (B + \varepsilon)$-Lipschitz continuous on $(0,1)^d$. Since the function is continuous, it is also Lipschitz continuous on $[0,1]^d$. Finally, since for all $x,y \in \R^d$
    $$
    \|\tilde{\Phi}_\varepsilon^f(x) - \tilde{\Phi}_\varepsilon^f(y)\|_\infty = \|\tilde{\Phi}_\varepsilon^f(T(x)) - \tilde{\Phi}_\varepsilon^f(T(y))\|_\infty \leq d\cdot(B+\varepsilon) \cdot \|T(x)-T(y)\|_\infty \leq d\cdot(B+\varepsilon) \cdot \|x-y\|_\infty,
    $$
    $\tilde{\Phi}_\varepsilon^f$ is globally Lipschitz continuous.
\end{proof}

\section{Proof of Theorem~\ref{thm:rates_for_RGNN}}
To prove Theorem \ref{thm:rates_for_RGNN}, we proceed in several steps. We start with a result showing that permutation-equivariant functions can be represented by permutation-equivariant functions of a specific form with high probability.
\begin{proposition}\label{prop:repr_of_GPEqui_[0,1]}
    Let $N, d, d', l \in \N$, and let $f : \mathcal{D}_{N,d,d'} \to \R^{N \times l}$ be a measurable permutation-equivariant function. 
    Let $R \in L^0([0,1]^N)$ be an i.i.d.\ vector with uniformly distributed components $R_1,\ldots,R_N$. For every $\delta > 0$ and every $M \in \N$ with $M \ge N^2/\delta$, there exist input and output dimensions 
    $\bar d_\psi=\bar d_\phi = 2d+d'+2, d_\psi\leq M^2d', d_\phi =1, \bar d_\alpha = d+1+d_\psi, d_\alpha \leq M(d+2) + M^2d', \bar d_\rho = d+1+d_\phi+d_\alpha, d_\rho = l$,
    and measurable functions 
    $\rho:\R^{\bar d_\rho} \to \R^{d_\rho}$, $\alpha:\R^{\bar d_\alpha} \to \R^{d_\alpha}$, $\phi:\R^{\bar d_\phi} \to \R^{d_\phi}$, and $\psi:\R^{\bar d_\psi} \to \R^{d_\psi}$, 
    such that the function $g:\R^{N \times (d + 1)}\times \R^{N\times N \times d'}\to \R^{N \times l}$ defined by
    \begin{align}\label{eq:def_of_g}
    g(x,a)_k := \rho\left(x_k, \sum_{i \in \mathcal{N}(k)}\phi \left(x_k, a_{i,k}, x_i\right), \sum_{j \in [N]}\alpha\left(x_j,\sum_{i \in \mathcal{N}(j)} \psi \left(x_j,a_{i,j}, x_i\right)\right)\right),  \qquad k \in [N],
    \end{align}
    satisfies, for all $(x,a) \in \mathcal{D}_{N,d,d'}$,
    $$
    \mathbb{P}\left(f(x,a) = g(x\concat R,a ) \right) \geq 1-\delta.
    $$
    In particular, with $h:\R \to [M]$ being the function from Definition~\ref{def:rv_unique_finite_feat} \ref{def:rv_unique_finite_feat:finite_unique_whp}, guaranteed to exist by Lemma \ref{lemma:unif}, for the set
    \begin{align}\label{eq:mathfrak_A}
        \mathfrak{A} = \left\{(x,a) \in \mathcal{D}_{N,d+1,d'} \mid \forall i,j \in [N], i \neq j: h(x_{i,d+1}) \neq h(x_{j,d+1}) \right\} 
    \end{align}
    we have $\mathfrak{A} \subseteq \{f(x,a) = g(x\concat R,a )\}$ and $\mathbb{P}\left(\mathfrak{A}\right) \geq 1-\delta$.
\end{proposition}
    \begin{proof}
        Let $\delta > 0$ and $M \in \N$ with $M \ge N^2 / \delta$. Define the intervals $I_i = [\tfrac{i}{M}, \tfrac{i+1}{M})$ for $i = 0,\ldots,M-2$, $I_{M-1} = [\tfrac{M-1}{M},1]$, and let $h : \R \to [M]$ be as in Lemma~\ref{lemma:unif}. Then, by Lemma~\ref{lemma:unif}, for every $(x,a) \in \mathcal{D}_{N,d,d'}$,
        \begin{align*}
            \mathbb{P}((x\concat R,a) \in \mathfrak{A}) =  \mathbb{P}\left( \forall i,j \in [N], i \neq j: h(R_i) \neq h(R_j)\right) \geq 1 - \frac{N^2}{M} \geq 1-\delta.
        \end{align*}
        In other words, with probability at least $1 - \delta$, all components $R_i$ fall into different intervals $I_k$.
        
        We now define $\psi: \R^{(d+1) + d' + (d+1)} \to \R^{M \times M \times d'}$ by 
        \begin{align*}
            \psi = (\psi_{i,j,k})_{i,j \in [M], k \in [d']}
        \end{align*}
        where each component $\psi_{i,j,k}:\R^{(d+1) + d' + (d+1)} \to \R$, $i,j \in [M], k \in [d']$ is given by
        $$
        \psi_{i,j,k}((x_1, \ldots, x_d, r^x), (e_1,\ldots, e_{d'}), (y_1,\ldots,y_d,r^y)) := \mathds{1}_{\{r^y \in I_{i-1}\}}\mathds{1}_{\{r^x \in I_{j-1}\}}e_k. 
        $$
        We define $\alpha:\R^{d+1}\times \R^{M\times M \times d'} \to \R^{M \times (d+1+1)} \times \R^{M \times M \times d'}$ as 
        \begin{align*}
            \alpha = (\alpha^ \RN{1}, \alpha^ \RN{2})     
        \end{align*}
        where the components $\alpha^\RN{1}_{i,j}:\R^{d+1}\times \R^{M\times M \times d'} \to \R$ for $i \in [M], j \in [d+2]$ and $\alpha^ \RN{2}_{i,j,k}:\R^{d+1}\times \R^{M\times M \times d'} \to \R$ for $i,j \in [M], k \in [d']$ are defined by
        \begin{align*}
            \alpha^ \RN{1}_{i,j}((x_1,\ldots,x_d, r^x), w^\psi) &:= \mathds{1}_{\{r^x \in I_{i-1}\}}x_j \qquad &i \in [M], j \in [d],\\
            \alpha^ \RN{1}_{i,d+1}((x_1,\ldots,x_d, r^x),w^\psi) &:= \mathds{1}_{\{r^x \in I_{i-1}\}}r^x \qquad &i \in [M],\\
            \alpha^ \RN{1}_{i,d+2}((x_1,\ldots,x_d, r^x), w^\psi) &:= \mathds{1}_{\{r^x \in I_{i-1}\}} \qquad &i \in [M],\\
            \alpha^ \RN{2}_{i,j,k}((x_1,\ldots,x_d, r^x), w^\psi) &:= w^\psi_{i,j,k} \qquad &i,j \in [M], k \in [d'].
        \end{align*}
        We further set $\phi:\R^{d+1 + d' + d+1} \to \R$ to be the constant function $\phi \equiv 0$.
        Next, define $\rho:\R^{d+1}\times \R \times \R^{M \times (d+2)}\times \R^{M \times M \times d'} \to \R^{l}$ 
        as 
        \begin{align*}
            \rho = \rho^B \circ \rho^A    
        \end{align*} 
        where $\rho^A:\R^{d+1}\times \R^{d_\phi} \times \R^{M \times (d+2)}\times \R^{M \times M \times d'} \to \R^{d+1}\times \R^{N \times (d+1)}\times \R^{N \times N \times d'} \times \R^{M}$ 
        has components 
        $\rho^A = (\rho^{A,\RN{1}}, \rho^{A,\RN{2}}, \rho^{A,\RN{3}}, \rho^{A,\RN{4}})$ 
        defined by
        \begin{align*}
            \rho_i^{A,\RN{1}}(y,y^\phi,z,w) &:= y_{i}  \qquad & i \in [d+1],\\
            \rho_{i,j}^{A,\RN{2}}(y,y^\phi,z,w) &:= \sum_{k=1}^M z_{k,j} \mathds{1}_{\{\cdot = 1 \}}(z_{k,d+2})\mathds{1}_{\{\cdot = (i-1)\}}\left(\sum_{l=1}^{k-1}z_{l,d+2} \right) \qquad & i\in[N], j \in[d+1],\\
            \rho_{i,j,n}^{A,\RN{3}}(y,y^\phi,z,w) &:= \sum_{s,t=1}^M w_{s,t,n} \mathds{1}_{\{\cdot = 1 \}}(z_{s,d+2})\mathds{1}_{\{\cdot = (i-1)\}}\left(\sum_{m=1}^{s-1}z_{m,d+2} \right)\\
            &\phantom{:= \sum_{s,t=1}^M w_{s,t,n}}  \cdot\mathds{1}_{\{\cdot = 1 \}}(z_{t,d+2})\mathds{1}_{\{\cdot = (j-1)\}}\left(\sum_{m=1}^{t-1}z_{m,d+2} \right)
            \qquad & i,j \in [N], n \in [d'],\\
            \rho_i^{A,\RN{4}}(y,y^\phi,z,w) &:= \mathds{1}_{\{\cdot = 1 \}}(z_{i,d+2})\sum_{k=1}^i z_{k, d+2}   \qquad & i \in [M].
        \end{align*}
        Finally, define 
        $\rho^B:\R^{d+1} \times \R^{N \times (d+1)}\times \R^{N \times N \times d'} \times \R^{M} \to \R^{l}$ 
        by
        \begin{align*}
            \rho^B(y, z', w', b) := \sum_{i=1}^N \tilde f_i(z',w')\mathds{1}_{\{\cdot = i\}}\left(
                \sum_{j=1}^M b_j \mathds{1}_{\{\cdot = j\}}\left(
                    \sum_{k=1}^M k \mathds{1}_{\{y_{d+1} \in I_{k-1} \}}
                \right)
            \right),
        \end{align*}
        where $\tilde f$ is the natural extension of $f$, defined by $\tilde f(x \concat r, a) := f(x,a)$ for $(x,a)\in\mathcal{D}_{N,d,d'}$, $r\in\R^N$.
        
        These choices reproduce the behavior of the functions used in the proof of Proposition~\ref{prop:repr_of_GPEqui}.
        Since $\mathfrak{A}$ is a set with finite unique node features (Def.~\ref{def:set_unique_finite_feat}), Proposition~\ref{prop:repr_of_GPEqui} implies that for all $(x \concat r,a) \in \mathfrak{A}$ we have $f(x,a) = g(x \concat r,a)$.
        Hence, for any $(x,a) \in \mathcal{D}_{N,d,d'}$,
        \begin{align*}
            \mathbb{P}(f(x,a) = g(x \concat R,a)) \geq \mathbb{P}((x\concat R,a) \in \mathfrak{A}) \geq 1-\delta.
        \end{align*}
        This completes the proof.
    \end{proof}

The next lemma shows that certain indicator functions can be approximated while controlling the regularity of the approximation. This will be important later.
\begin{lemma}\label{lem:C^2_approx_of_indicator}
    Let $M \in \N$, $I_i = [\tfrac{i}{M}, \tfrac{i+1}{M})$ for $i = 0,\ldots,M-2$, $I_{M-1} = [\tfrac{M-1}{M},1]$ and $\delta' > 0$.
    We can approximate the indicator functions 
    $$
    \mathds{1}_{\{\cdot \in I_i\}}(x), \qquad i=0,\ldots,M-1
    $$
    and 
    $$
    \mathds{1}_{\{\cdot = j\}}(x), \qquad j \in \mathbb{Z}
    $$
    by smooth functions $\iota_{i, \delta'}:\R \to \R$ and $\nu_{j}:\R\to\R$, with $0\leq \iota_{i, \delta'},\nu_{j} \leq 1$, respectively, such that
    for all $i=0,\ldots,M-1$ it holds 
    $$
    \iota_{i,\delta'}(x) =  \mathds{1}_{\{\cdot \in I_i\}}(x) \qquad \text{for } x \notin (\frac{i}{M} - \delta', \frac{i}{M} + \delta') \cup  (\frac{i+1}{M} - \delta', \frac{i+1}{M} + \delta'),
    $$ 
    and for all $j \in \mathbb{Z}$,
    $$
    \nu_{j}(x) = \mathds{1}_{\{\cdot = j\}}(x) \qquad \text{for } x \in \mathbb{Z}.
    $$
    In particular, for any $k \in \N$, both functions are $k$ times continuously differentiable, bounded, and have bounded partial derivatives up to order $k$.
\end{lemma}
\begin{proof}
    Fix $M \in \N$, $i \in \{0,\ldots,M-1\}$, and $\delta' > 0$. It is well known that by using mollifiers one can obtain a smooth function $\iota_{i,\delta'}$ that equals $1$ on the compact set
    $K = [\frac{i}{M} + \delta',\frac{i+1}{M} - \delta']$ 
    and 0 outside the open set 
    $U = (\frac{i}{M} - \delta',\frac{i+1}{M} + \delta') \supset K$. 
    Since $\iota_{i,\delta'}$ is smooth and vanishes outside $U$, we obtain for any multi-index $\beta \in \mathbb{N}_0^1$ with $|\beta| \le k$,
    $$ \max_{\substack{x \in \R\\ |\beta| \leq k}}|D^\beta\iota_{i, \delta'}(x)| = \max_{\substack{x \in [\frac{i}{M} - \delta',\frac{i+1}{M} + \delta']\\ |\beta| \leq k}}|D^\beta\iota_{i, \delta'}(x)| 
    < \infty.
    $$
    Similarly, for any $j \in \mathbb{Z}$ we can construct a smooth function $\nu_j$ that equals $1$ on $K = \{j\}$ and $0$ outside $U = \left(j-1, j+1\right)$. 
    By the same reasoning, $\nu_j$ is bounded and has finite partial derivatives.
\end{proof}
The next proposition is similar to Proposition~\ref{prop:repr_of_GPEqui_[0,1]}. We replace the functions $\rho, \alpha, \phi, \psi$ defining $g$ in \eqref{eq:def_of_g} by continuously differentiable functions $\tilde\rho, \tilde\alpha, \tilde\phi, \tilde\psi$ coinciding with the original functions on a set $\mathfrak{B}$.
This introduces an additional source of approximation error outside the set $\mathfrak{B}$. Here, we show that permutation-equivariant functions can still be approximated with probability $1-\delta$ by choosing the set $\mathfrak{B}$ large enough. This will guarantee that the probability that the random features take values outside $\mathfrak{B}$ is small enough.

\begin{proposition}\label{prop:repr_of_GPEqui_[0,1]_approx}
    Let $N, d, d', l \in \N$ and $k \in \N_{\geq 2}$. Let $f : \mathcal{D}_{N,d,d'} \to \R^{N \times l}$ be a $k$-times continuously differentiable permutation-equivariant function, and let $R \in L^0([0,1]^N)$ be an i.i.d.\ vector with uniformly distributed components $R_1,\ldots,R_N$

    For every $\delta \in (0,1)$, every $M \in \N$ with $M \ge 2N^2/\delta$, and every $\delta' \le \frac{1 - \left(\frac{2-\delta}{2}\right)^{1/N}}{2M}$, there exist input and output dimensions
    $\bar d_\psi=\bar d_\phi = 2d+d'+2, d_\psi\leq M^2d', d_\phi =1, \bar d_\alpha = d+1+d_\psi, d_\alpha \leq M(d+2) + M^2d', \bar d_\rho = d+1+d_\phi+d_\alpha, d_\rho = l$,
    and corresponding $k$-times continuously differentiable functions 
    $\tilde\rho:\R^{\bar d_\rho} \to \R^{d_\rho}$, $\tilde\alpha:\R^{\bar d_\alpha} \to \R^{d_\alpha}$, $\tilde\phi:\R^{\bar d_\phi} \to \R^{d_\phi}$, and $\tilde\psi:\R^{\bar d_\psi} \to \R^{d_\psi}$,
    such that the function $\tilde g:\R^{N \times (d + 1)}\times \R^{N\times N \times d'}\to \R^{N \times l}$ defined by
    \begin{align}\label{eq:def_of_g_prop_diffbar}
    \tilde g(x,a)_k := 
    \tilde\rho\left(x_k, \sum_{i \in \mathcal{N}(k)}\tilde\phi \left(x_k, a_{i,k}, x_i\right), \sum_{j \in [N]}\tilde\alpha\left(x_j,\sum_{i \in \mathcal{N}(j)} \tilde\psi \left(x_j,a_{i,j}, x_i\right)\right)\right),  \qquad k \in [N],
    \end{align}
    satisfies, for all $(x,a) \in \mathcal{D}_{N,d,d'}$,
    $$
    \mathbb{P}\left(f(x,a) = \tilde g(x\concat R,a ) \right) \geq 1-\delta.
    $$
    In particular, for the set $\mathfrak{A} \subset \mathcal{D}_{N,d+1,d'}$ defined in \eqref{eq:mathfrak_A} and the set
    \begin{align}\label{eq:mathfrak_B}
        \mathfrak{B} &:=  \left\{(x,a) \in \mathcal{D}_{N,d+1,d'} \mid \forall i\in [N]: x_{i,d+1} \notin \bigcup_{j=0}^{M}\left[\frac{j}{M}-\delta', \frac{j}{M}+\delta'\right] \right\},
    \end{align}
    we have $\mathfrak{A} \cap \mathfrak{B} \subseteq \{f(x,a) = \tilde g(x\concat R,a )\}$ and $\mathbb{P}\left(\mathfrak{A}\cap \mathfrak{B}\right) \geq 1-\delta$.
\end{proposition}
\begin{proof}
    Let $\delta \in (0,1)$,  $\N \ni M \geq \frac{2N^2}{\delta}$, and $\delta' \leq \frac{1-(\frac{2-\delta}{2})^{\frac{1}{N}}}{2M}$. 
    Let  
    $\bar d_\psi, \bar d_\phi, \bar d_\alpha, \bar d_\rho,d_\psi, d_\phi, d_\alpha,  d_\rho$,
    and
    $\rho, \alpha, \phi, \psi$, be the dimensions and corresponding functions from Proposition~\ref{prop:repr_of_GPEqui_[0,1]}. 
    For $g$ as in~\eqref{eq:def_of_g} and all $(x,a)\in\mathcal{D}_{N,d,d'}$, we have $\mathbb{P}\left(f(x,a) = g(x\concat R,a ) \right) \geq 1-\frac{\delta}{2}$.
    
    Recall that
    \begin{align*}
        \mathfrak{A} = \left\{(x,a) \in \mathcal{D}_{N,d+1,d'} \mid \forall i,j \in [N], i \neq j: h(x_{i,d+1}) \neq h(x_{j,d+1}) \right\}.
    \end{align*}
    By Proposition~\ref{prop:repr_of_GPEqui_[0,1]}, for all $(x \concat r,a) \in \mathfrak{A}$ we have $f(x,a) = g(x \concat r,a)$, and $\mathbb{P}((x \concat R,a) \in \mathfrak{A}) \geq 1-\frac{\delta}{2}$.
    
    Let $\tilde \rho, \tilde \phi, \tilde \alpha,$ and $\tilde \psi$ be defined as $\rho, \alpha, \phi$, and $\psi$, respectively, 
    except that every indicator function is replaced by its approximation from Lemma~\ref{lem:C^2_approx_of_indicator} using the above choice of $\delta'$. That is, for $I_i = [\tfrac{i}{M}, \tfrac{i+1}{M}),$ $i = 0,\ldots,M-2$, $I_{M-1} = [\tfrac{M-1}{M},1]$, each occurrence of
    $$
    \mathds{1}_{\{\cdot \in I_i\}}(x), \qquad i=0,\ldots,M-1,
    $$
    is replaced by $\iota_{i, \delta'}(x)$, and each occurrence of
    $$
    \mathds{1}_{\{\cdot = i\}}(x), \qquad i \in \N,
    $$
    is replaced by $\nu_{i}(x)$.  
    We obtain
    \begin{align}\label{eq:def_of_tilde_psi}
            \tilde \psi = (\tilde \psi_{i,j,k})_{i,j \in [M], k \in [d']},
        \end{align}
        where
        $$
        \tilde \psi_{i,j,k}((x_1, \ldots, x_d, r^x), (e_1,\ldots, e_{d'}), (y_1,\ldots,y_d,r^y)) := \iota_{i-1, \delta'}(r^y)\iota_{j-1, \delta'}(r^x)e_k \qquad i,j \in [M], k \in [d'].
        $$
        Furthermore,
        \begin{align}\label{eq:def_of_tilde_alpha}
            \tilde\alpha = (\tilde\alpha^ \RN{1}, \tilde\alpha^ \RN{2})     
        \end{align}
        with 
        \begin{align*}
            \tilde\alpha^ \RN{1}_{i,j}((x_1,\ldots,x_d, r^x), w^\psi) &:= \iota_{i-1, \delta'}(r^x)x_j \qquad &i \in [M], j \in [d],\\
            \tilde\alpha^ \RN{1}_{i,d+1}((x_1,\ldots,x_d, r^x),w^\psi) &:= \iota_{i-1, \delta'}(r^x)r^x \qquad &i \in [M],\\
            \tilde\alpha^ \RN{1}_{i,d+2}((x_1,\ldots,x_d, r^x), w^\psi) &:= \iota_{i-1, \delta'}(r^x) \qquad &i \in [M],\\
            \tilde\alpha^ \RN{2}_{i,j,k}((x_1,\ldots,x_d, r^x), w^\psi) &:= w^\psi_{i,j,k} \qquad &i,j \in [M], k \in [d']. 
        \end{align*}
        Let $\tilde \phi:\R^{d+1 + d' + d+1} \to \R$ be the constant function $\tilde \phi \equiv 0$.
        Define
        \begin{align}\label{eq:def_of_tilde_rho^A}
            \tilde \rho = \tilde\rho^B \circ \tilde\rho^A 
        \end{align} 
        where
        \begin{align*}
            \tilde\rho_i^{A,\RN{1}}(y,y^\phi,z,w) &:= y_{i},  \qquad & i \in [d+1],\\
            \tilde\rho_{i,j}^{A,\RN{2}}(y,y^\phi,z,w) &:= \sum_{k=1}^M z_{k,j} \nu_{1}(z_{k,d+2})\nu_{i-1}\left(\sum_{l=1}^{k-1}z_{l,d+2}\right), \qquad & i\in[N], j \in[d+1],\\
            \tilde\rho_{i,j,n}^{A,\RN{3}}(y,y^\phi,z,w) &:= \sum_{s,t=1}^M w_{s,t,n} \nu_{1}(z_{s,d+2})\nu_{i-1}\left(\sum_{m=1}^{s-1}z_{m,d+2} \right)\\
            &\phantom{:= \sum_{s,t=1}^M w_{s,t,n}}  \cdot\nu_{1}(z_{t,d+2})\nu_{j-1}\left(\sum_{m=1}^{t-1}z_{m,d+2} \right),
            \qquad & i,j \in [N], n \in [d'],\\
            \tilde\rho_i^{A,\RN{4}}(y,y^\phi,z,w) &:= \nu_{1}(z_{i,d+2})\sum_{k=1}^i z_{k, d+2},   \qquad & i \in [M].
        \end{align*}
        Finally, $\tilde\rho^B = (\tilde\rho^B_1,\ldots,\tilde\rho^B_N)$ is given by
        \begin{align*}
            \tilde\rho^B(y, z', w', b) := \sum_{i=1}^N \tilde f_i(z',w')\nu_{i}\left(
                \sum_{j=1}^M b_j \nu_{j}\left(
                    \sum_{k=1}^M k \iota_{k-1, \delta'}(y_{d+1})
                \right)
            \right),
        \end{align*}
    where $\tilde f$ is the natural extension of $f$, defined by $\tilde f(x \concat r, a) := f(x,a)$ for $(x,a)\in\mathcal{D}_{N,d,d'}$, $r\in\R^N$.
    
    Consider the function $\tilde g$, defined for any $(x,a) \in \mathcal{D}_{N,d+1,d'}$ by
    \begin{align*}
    \tilde g(x,a)_k := \tilde\rho\left(x_k, \sum_{i \in \mathcal{N}(k)}\tilde\phi \left(x_k, a_{i,k}, x_i\right), \sum_{j \in [N]}\tilde\alpha\left(x_j,\sum_{i \in \mathcal{N}(j)} \tilde\psi \left(x_j,a_{i,j}, x_i\right)\right)\right),  \qquad k \in [N].
    \end{align*}
    We will now show that    
    $f(x,a) = \tilde g(x \concat r, a)$ holds for all $(x \concat r,a) \in \mathfrak{A} \cap \mathfrak{B} \subset \mathcal{D}_{N,d+1,d'}$.

    For $(x \concat r,a) \in \mathfrak{A}$ all node features $r_i$ lie in distinct intervals $I_k$.
    For $(x \concat r,a) \in \mathfrak{B}$, the node features $r_i$ avoid precisely those regions where the indicators do not coincide with their continuous approximations.
    In particular, Lemma~\ref{lem:C^2_approx_of_indicator} implies that for all $(x \concat r,a) \in \mathfrak{B}$ we have 
    $\iota_{j,\delta'}(r_i) =  \mathds{1}_{\{\cdot \in I_j\}}(r_i)$ for all $i\in [N], j=0,\ldots,M-1$. 
    Consequently, $\tilde \psi = \psi$ and $\tilde \alpha = \alpha$.
    Recall that the function $\alpha$ from Proposition~\ref{prop:repr_of_GPEqui_[0,1]} returns, as its first component, a matrix 
    $z \in \R^{M \times d+1+1}$ whose $d+2$-th column contains binary flags for each row $i \in [M]$.  
    Since $\tilde \alpha$ and $\tilde \psi$ coincide exactly with $\alpha, \psi$,     
    these entries are indeed equal to $0$ or $1$. 
    Thus, for 
    $$
    (z,w) = \sum_{j \in [N]}\tilde\alpha\left(x_j,\sum_{i \in \mathcal{N}(j)} \tilde\psi \left(x_j,a_{i,j}, x_i\right)\right)
    $$
    the components $z_{i,d+2},i \in [M]$ and their partial sums $\sum_{j=1}^{i-1}z_{j,d+2}$ necessarily take values in $\N_0$. 
    Moreover, $\phi = \tilde \phi \equiv 0$ by definition. 
    Hence, it remains to show that $\rho$ and $\tilde\rho$ also coincide on $\mathfrak{B}$ in order to conclude that $g = \tilde g$ holds on $\mathfrak{B}$.
    We already know that $\iota_{j,\delta'}(r_i) =  \mathds{1}_{\{\cdot \in I_j\}}(r_i)$ for all $i\in [N]$ and $j=0,\ldots,M-1$. 
    The only remaining approximated terms in $\tilde \rho$ are of the form $\mathds{1}_{\{\cdot = i\}}(x)$ for some $i \in \N_0$.
    Since $\tilde\alpha$ and $\tilde\psi$ coincide exactly with $\alpha$ and $\psi$, the corresponding inputs $x$ take values in $\N_0$.
    By Lemma~\ref{lem:C^2_approx_of_indicator}, the functions $\mathds{1}_{\{\cdot = i\}}(x)$ and $\nu_i(x)$ agree for all $x \in \mathbb{Z}$.
    Therefore, all approximations in $\tilde\rho$ are exact on $\mathfrak{B}$, and we conclude that $g = \tilde g$ holds on $\mathfrak{B}$.
    Since $f(x,a) = g(x \concat r,a)$ holds on $\mathfrak{A}$ and $g(x \concat r,a) = \tilde g(x \concat r,a)$ holds on $\mathfrak{B}$, we immediately obtain $f(x,a) = \tilde g(x \concat r, a)$ on $\mathfrak{A} \cap \mathfrak{B}$. 
    
    It remains to show that $\tilde g(x \concat R, a) = f(x,a)$ holds with the required probability.
    For any $(x,a) \in \mathcal{D}_{N,d,d'}$ we have
    \begin{align*}
    \mathbb{P}((x \concat R,a) \in \mathfrak{B}^c) &= 1- \mathbb{P}((x \concat R,a) \in \mathfrak{B})\\
    & = 1- \mathbb{P}\left( \forall i\in [N], j=0,\ldots,M-1: \iota_{j,\delta'}(R_i) =  \mathds{1}_{\{\cdot \in I_j\}}(R_i)\right)\\
     &=1- \mathbb{P}\left(\forall i\in [N]: R_i \notin \left\{ \bigcup_{j=0}^{M}\left[\frac{j}{M}-\delta', \frac{j}{M}+\delta'\right] \right\}\right)\\
     &=1-  \mathbb{P}\left(R_1 \notin  \bigcup_{j=0}^{M}\left[\frac{j}{M}-\delta', \frac{j}{M}+\delta'\right] \right)^N =1-  (1-M2\delta')^N.
    \end{align*}
    Since we chose 
    $\delta' \leq \frac{1-(\frac{2-\delta}{2})^{\frac{1}{N}}}{2M}$,
    we obtain
    \begin{align*}
        \mathbb{P}(f(x,a) &= \tilde g(x \concat R,a)) \geq  \mathbb{P}( (x \concat R,a) \in {\mathfrak{A}} \cap \mathfrak{B}) = 1 - \mathbb{P}((x \concat R,a) \in \mathfrak{A}^c \cup \mathfrak{B}^c)\\ 
        &\geq 1- \mathbb{P}((x \concat R,a) \in \mathfrak{A}^c) - \mathbb{P}((x \concat R,a) \in \mathfrak{B}^c) \geq 1 - \frac{\delta}{2} - (1-(1-M2\delta')^N)\\
        &\geq 1 - \frac{\delta}{2} - \frac{\delta}{2} = 1- \delta.  
    \end{align*}
    Finally, since $\tilde\rho$, $\tilde\phi$, $\tilde\alpha$, and $\tilde\psi$ are built from compositions, products, and sums of $k$‑times continuously differentiable (or even smooth) functions, they are themselves $k$‑times continuously differentiable. 
    This completes the proof.
\end{proof}
Recall that for $N,d,d',l \in \N$, $k \in \N_{\geq 2}$, and $\kappa > 0$, we define
\begin{align*}
\begin{split}
   \mathcal{F}_{N,d,d', l, k,\kappa} :=  &\Bigg\{ f:\mathcal{D}_{N,d,d'} \to \R^{N \times l} \text{ perm. equiv.} \,\Bigg\vert\, \forall i \in [N], j \in [l]:  f_{i,j}  \in \mathcal{C}^k(\mathcal{D}_{N,d,d'}) \text{ and }\\
  & \phantom{\{} \forall (x,a) \in \mathcal{D}_{N,d,d'}^{[0,1]}: \max_{\substack{i \in [N], j \in [l]\\ \beta \in \N^{Nd + N^2d'}_0, |\beta|\leq k}} |D^\beta f_{i,j}(x,a)| \leq \kappa  \Bigg\}.
\end{split}
\end{align*}
The following proposition shows that for any function $f \in \mathcal{F}_{N,d,d', l, k,\kappa}$ the partial derivatives of functions $\tilde\rho, \tilde\alpha, \tilde\psi$ defining $\tilde g$ in~\eqref{eq:def_of_g_prop_diffbar}, whose existence is guaranteed by Proposition~\ref{prop:repr_of_GPEqui_[0,1]_approx}, can be bounded independently of function $f$.
\begin{proposition}\label{prop:bounds_for_part_deriv_of_helpers}
    Let $N,d,d',l \in \N$, $k \in \N_{\geq 2}$, and $\kappa > 0$. 
    Let $\delta \in (0,1)$, choose $M \in \N$ with $M \ge \frac{2N^{2}}{\delta}$, and assume
    $\delta' \leq \frac{1-(\frac{2-\delta}{2})^{\frac{1}{N}}}{2M}$.
    Define
    \begin{align}\label{eq:func_dimensions}
    \begin{split}
        \bar d_\psi &:= 2d+d'+2\\
        d_\psi &:= M^2d'\\
        \bar d_\alpha &:= d+1+M^2d'\\
        d_\alpha &:= M(d+2) + M^2d'\\
        \bar d_\phi &:=d+1+d'+d+1\\
        d_\phi &:= 1\\
        \bar d_\rho &:= d+1+1+M(d+2)+M^2d'\\
        d_\rho &:= l.
    \end{split}
    \end{align}
    Then, there exist constants $\kappa_\psi(N,d,d',\delta,k)>0$, $\kappa_\alpha(N,d,d',\delta,k) >0$, and $\kappa_\rho(N,d,d',\delta, k,\kappa)>0$ such that for any $f \in \mathcal{F}_{N,d,d',l,k,\kappa}$ the corresponding functions $\tilde\rho:\R^{\bar d_\rho} \to \R^{d_\rho}$, $\tilde\alpha:\R^{\bar d_\alpha} \to \R^{d_\alpha}$, 
    $\tilde\phi:\R^{\bar d_\phi} \to \R^{d_\phi}$,
    and $\tilde\psi:\R^{\bar d_\psi} \to \R^{d_\psi}$ obtained from Proposition~\ref{prop:repr_of_GPEqui_[0,1]_approx} satisfy
    \begin{align*}
         \max_{\substack{x \in [0,1]^{\bar d_\psi}\\ |\beta|\leq k}} |D^\beta \tilde\psi_{i}(x)| &\leq \kappa_\psi  && \text{for } i \in [d_\psi],\\
         \max_{\substack{x \in [0,1]^{\bar d_\alpha}\\ |\beta|\leq k}} |D^\beta \tilde\alpha_{i}(x)| & \leq \kappa_\alpha && \text{for } i \in [d_\alpha],\\
         \max_{\substack{x \in [0,1]^{\bar d_\rho}\\ |\beta|\leq k}} |D^\beta \tilde\rho_{i}(x)| &\leq \kappa_\rho  && \text{for } i \in [l].
    \end{align*}
\end{proposition}
\begin{remark}
    As noted in the proof of Proposition~\ref{prop:repr_of_GPEqui}, the function $\phi$ (or $\tilde\phi$) is not needed to guarantee universality. This is the reason why we set the dimension of its codomain to $1$. In practice, however, including $\phi$ empirically improves learning performance. It is therefore reasonable to choose a more complex architecture for $\phi$, for example, similar to that of $\psi$, and consequently to select a codomain dimension comparable to $d_\psi$.
\end{remark}
\begin{proof} \textbf{of Proposition~\ref{prop:bounds_for_part_deriv_of_helpers}}
    Let $f$ be an arbitrary function in  $\mathcal{F}_{N,d,d',l,k,\kappa}$.
    From the definitions of $\tilde\psi$ and $\tilde\alpha$ in~\eqref{eq:def_of_tilde_psi} and~\eqref{eq:def_of_tilde_alpha}, we observe that both functions are identical for all choices of $f$. They depend only on the parameters $N,d,d',\delta,k$, on $M$ (which itself depends on $N,\delta$), and on $\delta'$ (which depends on $N,\delta,M$). Since their components are $k$-times continuously differentiable, we can bound all partial derivatives up to order $k$ on the compact sets $[0,1]^{\bar d_\psi}$ and $[0,1]^{\bar d_\alpha}$. 
    Thus, there exist constants 
    $\kappa_\psi = \kappa_\psi(N,d,d',\delta,k)>0$, $\kappa_\alpha = \kappa_\alpha(N,d,d',\delta,k) >0$
    such that
    \begin{align*}
        \max_{\substack{x \in [0,1]^{\bar d_\psi}\\ |\beta|\leq k}} |D^\beta \tilde \psi_{i}(x)| &\leq \kappa_\psi  && \text{for } i \in [d_\psi],\\
         \max_{\substack{x \in [0,1]^{\bar d_\alpha}\\ |\beta|\leq k}} |D^\beta \tilde \alpha_{i}(x)| & \leq \kappa_\alpha && \text{for } i \in [d_\alpha].
    \end{align*}
    The function $\tilde\rho$ does depend on $f$, as it is defined as a composition of $f$ with other $k$-times differentiable functions; see~\eqref{eq:def_of_tilde_rho^A}. However, since all partial derivatives of $f$ up to order $k$ are uniformly bounded by $\kappa$ for every $f \in \mathcal{F}_{N,d,d',l,k,\kappa}$, the chain rule yields a uniform bound for $\tilde\rho$. Hence, there exists a constant
    $\kappa_\rho(N,d,d',\delta,k,\kappa)$
    such that for all $f \in \mathcal{F}_{N,d,d',l,k,\kappa}$ and all $i \in [l]$,
    \begin{align*}
        \max_{\substack{x \in [0,1]^{\bar d_\rho}\\ |\beta|\leq k}} |D^\beta \tilde\rho_{i}(x)| &\leq \kappa_\rho.
    \end{align*}
\end{proof}
The next proposition builds on Proposition~\ref{prop:repr_of_GPEqui_[0,1]_approx}. Using Proposition~\ref{prop:bounds_for_part_deriv_of_helpers} and Corollary~\ref{cor:rates_for_NN_global_lip} we replace the continuously differentiable functions $\tilde \rho, \tilde \alpha, \tilde \phi, \tilde \psi$ by suitable ReLU neural networks while keeping the approximation error small. This is a major building block of the proof of Theorem~\ref{thm:rates_for_RGNN}.
\begin{proposition}\label{prop:component_error_to_total_error} 
    Let $N,d,d',l \in \N$, $k \in \N_{\ge 2}$, and $\kappa > 0$. 
    Let $\delta \in (0,1)$, choose $M \in \N$ with $M \ge \frac{2N^{2}}{\delta}$, and assume
    $\delta' \leq \frac{1-(\frac{2-\delta}{2})^{\frac{1}{N}}}{2M}$.
    Let $\bar d_\psi, d_\psi, \bar d_\alpha, d_\alpha, \bar d_\phi, d_\phi, \bar d_\rho, d_\rho$ be as in~\eqref{eq:func_dimensions}. 
    Then, there exist constants 
    $c_\rho = c_\rho(\bar d_\rho, N,\delta, \kappa, k)$, $c_\alpha = c_\alpha(\bar d_\alpha, N, \delta, k)$, $c_\psi = c_\psi(\bar d_\psi, N,\delta, k) >0$ with the following property.

    For every $\varepsilon \in (0,\frac{1}{2})$, there exist neural-network architectures $\mathcal{A}_\rho, \mathcal{A}_\alpha, \mathcal{A}_\phi$, and $\mathcal{A}_\psi$, such that for every $f \in \mathcal{F}_{N,d,d',l,k,\kappa}$ and the corresponding functions
    $\tilde\rho:\R^{\bar d_\rho} \to \R^{d_\rho}$, $\tilde\alpha:\R^{\bar d_\alpha} \to \R^{d_\alpha}$, 
    $\tilde\phi:\R^{\bar d_\phi} \to \R^{d_\phi}$,
    and $\tilde\psi:\R^{\bar d_\psi} \to \R^{d_\psi}$ from Proposition~\ref{prop:repr_of_GPEqui_[0,1]_approx}, there exist ReLU neural networks $\hat \rho:\R^{\bar d_\rho} \to \R^{d_\rho}$, $\hat \alpha:\R^{\bar d_\alpha} \to \R^{d_\alpha}$, 
    $\hat \phi:\R^{\bar d_\phi} \to \R^{d_\phi}$,
    and $\hat \psi:\R^{\bar d_\psi} \to \R^{d_\psi}$ with architecture $\mathcal{A}_\rho, \mathcal{A}_\alpha, \mathcal{A}_\phi $, and $\mathcal{A}_\psi$, respectively, for which the following holds.

    \begin{itemize}
        \item[(i)]
            Let $\tilde g:\R^{N \times (d +1) }\times \R^{N\times N \times d'}\to \R^{N \times l}$ be defined by
            \begin{align}\label{eq:def_g_total_error}
                \tilde g(x,a)_k := \tilde \rho \left(x_k, \sum_{i \in \mathcal{N}(k)}\tilde\phi \left(x_k, a_{i,k}, x_i\right), \sum_{j \in [N]}\tilde\alpha \left(x_j,\sum_{i \in \mathcal{N}(j)} \tilde\psi  \left(x_j,a_{i,j}, x_i\right)\right)\right), k \in [N].
            \end{align}
            Similarly, define $\hat g:\R^{N \times (d +1) }\times \R^{N\times N \times d'}\to \R^{N \times l}$ by
            \begin{align}\label{eq:def_tilde_g_total_error}
                \hat g(x,a)_k := \hat \rho \left(x_k, \sum_{i \in \mathcal{N}(k)}\hat \phi \left(x_k, a_{i,k}, x_i\right), \sum_{j \in [N]}\hat \alpha \left(x_j,\sum_{i \in \mathcal{N}(j)} \hat \psi  \left(x_j,a_{i,j}, x_i\right)\right)\right), k \in [N].
            \end{align}
            Let 
            \begin{align*}
                K := \mathfrak{A} \cap \mathfrak{B} \cap \mathcal{D}_{N,d+1,d'}^{[0,1]}
            \end{align*}
            where $\mathfrak{A}$ and $\mathfrak{B}$ are the sets defined in~\eqref{eq:mathfrak_A} and~\eqref{eq:mathfrak_B}, respectively.
            Then,
            \begin{align*}
            \max_{(x,a) \in K
            }\| \tilde g(x,a) - \hat g(x,a) \|_{\infty} \leq \varepsilon.
            \end{align*}
        \item[(ii)]
            The complexity of the neural-network architectures is bounded as follows.
            \begin{align*}
            &L(\mathcal{A}_\rho) \leq c_\rho \log_2\left(\left(\frac{\varepsilon}{3}\right)^{-k/(k-1)}\right),\\
            &M(\mathcal{A}_\rho) \leq c_\rho d_\rho \left(\frac{\varepsilon}{3}\right)^{-\bar d_\rho/(k-1)}\cdot \left(\log_2\left(\left(\frac{\varepsilon}{3}\right)^{-k/(k-1)}\right)\right)^2,\\
            &L(\mathcal{A}_\alpha) \leq c_\alpha \log_2\left(\left(\frac{\varepsilon}{3N\bar d_\rho(\kappa_\rho +1)}\right)^{-k/(k-1)}\right),\\
            &M(\mathcal{A}_\alpha) \leq c_\alpha d_\alpha \left(\frac{\varepsilon}{3N\bar d_\rho(\kappa_\rho +1)}\right)^{-\bar d_\alpha/(k-1)}\cdot \left(\log_2\left(\left(\frac{\varepsilon}{3N\bar d_\rho(\kappa_\rho +1)}\right)^{-k/(k-1)}\right)\right)^2,\\
            &L(\mathcal{A}_\psi) \leq c_\psi \log_2\left(\left(\frac{\varepsilon}{3N^2 \bar d_\rho(\kappa_\rho +1)\bar d_\alpha(\kappa_\alpha+1)}\right)^{-k/(k-1)}\right),\\
            &M(\mathcal{A}_\psi) \leq c_\psi d_\psi \left(\frac{\varepsilon}{3N^2 \bar d_\rho(\kappa_\rho +1)\bar d_\alpha(\kappa_\alpha+1)}\right)^{-\bar d_\psi/(k-1)}\\
            &\phantom{M(\mathcal{A}_\psi) \leq }\cdot \left(\log_2\left(\left(\frac{\varepsilon}{3N^2 \bar d_\rho(\kappa_\rho +1)\bar d_\alpha(\kappa_\alpha+1)}\right)^{-k/(k-1)}\right)\right)^2,\\
            &L(\mathcal{A}_\phi) = 1,\\
            &M(\mathcal{A}_\phi) = 0.
        \end{align*}
    \end{itemize}
\end{proposition}
\begin{proof}\,\\
    \textit{Part 1:}\\
    In the first part of the proof, we establish the existence of  
    constants $c_\rho$, $c_\alpha$, $c_\psi$, neural-network architectures $\mathcal{A}_\rho$, $\mathcal{A}_\alpha$, $\mathcal{A}_\phi$, $\mathcal{A}_\psi$, and ReLU networks $\hat\rho$, $\hat\alpha$, $\hat\phi$, $\hat\psi$.
    From Proposition~\ref{prop:bounds_for_part_deriv_of_helpers}, we obtain constants
    \begin{align*}
        &\kappa_\psi(N,d,d',\delta,k)>0,\\ 
        &\kappa_\alpha(N,d,d',\delta,k)>0,\\
        &\kappa_\rho(N,d,d',\delta,k,\kappa)>0,
    \end{align*}
    such that, for any $f \in \mathcal{F}_{N,d,d',l,k,\kappa}$ and the corresponding functions $\tilde\rho:\R^{\bar d_\rho}\to\R^{d_\rho}$, $\tilde\alpha:\R^{\bar d_\alpha}\to\R^{d_\alpha}$, and $\tilde\psi:\R^{\bar d_\psi}\to\R^{d_\psi}$ from Proposition~\ref{prop:repr_of_GPEqui_[0,1]_approx}, all partial derivatives up to order $k$ are uniformly bounded. In particular, their Sobolev norms satisfy
    \begin{align*}
        \| \tilde\psi_i \|_{W^{k,\infty}((0,1)^{\bar d_\psi})} &\leq \max_{\substack{x \in [0,1]^{\bar d_\psi}  \\ |\beta|\leq k}} |D^\beta \tilde \psi_{i}(x)| \leq \kappa_\psi  && \text{for } i \in [d_\psi],\\
        \| \tilde\alpha_i \|_{W^{k,\infty}((0,1)^{\bar d_\alpha})} &\leq \max_{\substack{x \in [0,1]^{\bar d_\alpha}\\ |\beta|\leq k}} |D^\beta \tilde \alpha_{i}(x)| \leq \kappa_\alpha && \text{for } i \in [d_\alpha],\\
        \| \tilde\rho_i \|_{W^{k,\infty}((0,1)^{\bar d_\rho})} &\leq \max_{\substack{x \in [0,1]^{\bar d_\rho}\\ |\beta|\leq k}} |D^\beta \tilde \rho_{i}(x)| \leq \kappa_\rho  && \text{for } i \in [l].
    \end{align*}
    Following Corollary~\ref{cor:rates_for_NN_global_lip}, there exist constants 
    \begin{align*}
        c_\rho &= c_\rho(\bar d_\rho, k, \infty,  \kappa_\rho(N,d,d',\delta, k, \kappa), 1)>0,\\ 
        c_\alpha &= c_\alpha(\bar d_\alpha, k, \infty,  \kappa_\alpha(N,d,d',\delta, k),1)>0,\\ 
        c_\psi &= c_\psi(\bar d_\psi,k,\infty, \kappa_\psi(N,d,d',\delta,k),1) >0, 
    \end{align*}
    with the following properties. 
    For $\varepsilon \in (0,\frac{1}{2})$, define
    \begin{align*}
        \varepsilon_\rho &:= \frac{\varepsilon}{3}, \\
        \varepsilon_\alpha &:= \frac{\varepsilon}{3N\bar d_\rho(\kappa_\rho +1)},\\
        \varepsilon_\psi &:= \frac{\varepsilon}{3N^2 \bar d_\rho(\kappa_\rho +1)\bar d_\alpha(\kappa_\alpha+1)}.
    \end{align*}
     Then there exist neural-network architectures 
     $\mathcal{A}_\rho = \mathcal{A}_\rho(\bar d_\rho,k, \infty, \kappa_\rho, 1, \varepsilon_\rho)$, $\mathcal{A}_\alpha = \mathcal{A}_\alpha(\bar d_\alpha,k,\infty, \kappa_\alpha,1, \varepsilon_\alpha)$, and $\mathcal{A}_\psi = \mathcal{A}_\psi(\bar d_\psi,k,\infty, \kappa_\psi,1, \varepsilon_\psi)$ 
     such that for any $f \in \mathcal{F}_{N,d,d',l,k,\kappa}$ and the associated functions 
     $\tilde\rho:\R^{\bar d_\rho} \to \R^{d_\rho}$, $\tilde\alpha:\R^{\bar d_\alpha} \to \R^{d_\alpha}$,
    and $\tilde\psi:\R^{\bar d_\psi} \to \R^{d_\psi}$ from Proposition~\ref{prop:repr_of_GPEqui_[0,1]_approx}, there exist ReLU neural networks 
    $\hat \rho_i, i \in [d_\rho]$, 
    $\hat \alpha_i, i \in [d_\alpha]$, 
    and $\hat \psi_i,i \in [d_\psi]$,
    of architectures $\mathcal{A}_\rho$, $\mathcal{A}_\alpha$, and $\mathcal{A}_\psi$, respectively, such that
    \begin{align}
   \|\tilde\psi_i - \hat \psi_i\|_{W^{1,\infty}((0,1)^{\bar d_\psi})} &\leq \varepsilon_\psi \qquad i \in [d_\psi],\label{eq:approx_of_psi}\\
    \|\tilde\alpha_i - \hat \alpha_i\|_{W^{1,\infty}((0,1)^{\bar d_\alpha})} &\leq \varepsilon_\alpha \qquad i \in [d_\alpha],\label{eq:approx_of_alpha}\\
        \|\tilde\rho_i  - \hat \rho_i\|_{W^{1,\infty}((0,1)^{\bar d_\rho})} &\leq \varepsilon_\rho \qquad i \in [l]\label{eq:approx_of_rho}.
    \end{align}
    Moreover, the architectures satisfy
    \begin{align*}
        &L(\mathcal{A}_\rho) \leq c_\rho \log_2\left(\left(\frac{\varepsilon}{3}\right)^{-k/(k-1)}\right),\\
        &M(\mathcal{A}_\rho) \leq c_\rho \left(\frac{\varepsilon}{3}\right)^{-\bar d_\rho/(k-1)}\cdot \log_2\left(\left(\frac{\varepsilon}{3}\right)^{-k/(k-1)}\right)^2,\\
        &L(\mathcal{A}_\alpha) \leq c_\alpha \log_2\left(\left(\frac{\varepsilon}{3N\bar d_\rho(\kappa_\rho +1)}\right)^{-k/(k-1)}\right),\\
        &M(\mathcal{A}_\alpha) \leq c_\alpha \left(\frac{\varepsilon}{3N\bar d_\rho(\kappa_\rho +1)}\right)^{-\bar d_\alpha/(k-1)}\cdot \log_2\left(\left(\frac{\varepsilon}{3N\bar d_\rho(\kappa_\rho +1)}\right)^{-k/(k-1)}\right)^2,\\
        &L(\mathcal{A}_\psi) \leq c_\psi \log_2\left(\left(\frac{\varepsilon}{3N^2 \bar d_\rho(\kappa_\rho +1)\bar d_\alpha(\kappa_\alpha+1)}\right)^{-k/(k-1)}\right),\\
        &M(\mathcal{A}_\psi) \leq c_\psi \left(\frac{\varepsilon}{3N^2 \bar d_\rho(\kappa_\rho +1)\bar d_\alpha(\kappa_\alpha+1)}\right)^{-\bar d_\psi/(k-1)}\\
        &\phantom{M(\mathcal{A}_\psi) \leq }\cdot \log_2\left(\left(\frac{\varepsilon}{3N^2 \bar d_\rho(\kappa_\rho +1)\bar d_\alpha(\kappa_\alpha+1)}\right)^{-k/(k-1)}\right)^2.
    \end{align*}
    Since the ReLU networks $(\hat\rho_n)_{n\in[d_\rho]}$ all have architecture $\mathcal{A}_\rho$, they share the same number of layers. By Proposition~\ref{prop:parallelize_relu_L}, Remark~\ref{rem:parallelize_multi_relu_L}, and Remark~\ref{rem:parallelization_of_nn_architectures}, we may parallelize both the networks and their architectures. This yields a single ReLU network
    $\hat \rho := \mathbf{P}(\hat \rho_1,\ldots, \hat \rho_{d_\rho}):\R^{\bar d_\rho} \to \R^{d_\rho}$ with architecture $\mathcal{A}_\rho^* = \mathbf{P}(\mathcal{A}_\rho,...,\mathcal{A}_\rho)$, satisfying 
    $L(\hat \rho) = L(\mathcal{A}^*_\rho) = L(\mathcal{A}_\rho)$ and $M(\hat \rho) \leq M(\mathcal{A}_\rho^*) = d_\rho M(\mathcal{A}_\rho)$.

    Analogously, we obtain ReLU networks $\hat\alpha$ and $\hat\psi$ with architectures $\mathcal{A}_\alpha^*$ and $\mathcal{A}_\psi^*$, respectively, such that
    $L(\hat \alpha)= L(\mathcal{A}^*_\alpha) = L(\mathcal{A}_\alpha)$, $M(\hat \alpha)\leq M(\mathcal{A}_\alpha^*) = d_\alpha M(\mathcal{A}_\alpha)$, and $L(\hat \psi) = L(\mathcal{A}_\psi^*) =  L(\mathcal{A}_\psi)$, $M(\hat \psi) = M(\mathcal{A}_\psi^*) = d_\psi M(\mathcal{A}_\psi)$.

    The network $\hat\phi$ is a special case. For any $f\in\mathcal{F}_{N,d,d',l,k,\kappa}$, the corresponding function $\tilde\phi$ from Proposition~\ref{prop:repr_of_GPEqui_[0,1]_approx} satisfies $\tilde\phi\equiv\phi\equiv 0$. Hence $\phi$ can be represented directly by the zero network $\hat\phi:\R^{\bar d_\phi}\to\R$ with trivial architecture $\mathcal{A}_\phi^*$ given by $\hat\phi(x)=\mathbf{0}_{\R\times\R^{\bar d_\phi}}x+0$ for which
    $L(\hat \phi) = L(\mathcal{A}_\phi^*) = 1$ and $M(\hat \phi) = M(\mathcal{A}_\phi^*) = 0$. 
    This completes the first part of the proof.\\\\
    \textit{Part 2:}\\
    In the second part of the proof, we show that the function $\hat g$, defined in~\eqref{eq:def_tilde_g_total_error} using the ReLU neural networks from above, approximates $\tilde g$ defined in~\eqref{eq:def_g_total_error} up to error $\varepsilon$.
    In general, for continuous functions $f,g:\R^N\to\R^M$ and any $\varepsilon>0$, continuity implies
    \begin{align}\label{eq:continuity_open_interval_bound}
     \sup_{x \in  (0,1)^N}\|f(x) -g(x)\|_\infty \leq \varepsilon \implies \max_{x \in  [0,1]^N}\|f(x) -g(x)\|_\infty \leq \varepsilon.
    \end{align}
    We first observe that for all $(x,a)\in K$
    \begin{align}
        &\sum_{i \in \mathcal{N}(j)}\tilde\psi(x_j, a_{i,j}, x_{i}) \in [0,1]^{d_\psi}, & j \in [N]. \label{eq:psi_in_01}
    \end{align}
    At first glance, this may seem unexpected. By definition (see~\eqref{eq:def_of_tilde_psi}),
    $$
        \tilde \psi_{i,j,k}((x_1, \ldots, x_d, r^x), (e_1,\ldots, e_{d'}), (y_1,\ldots,y_d,r^y)) = \underbrace{\iota_{i-1, \delta'}(r^y)}_{\in [0,1]} \underbrace{\iota_{j-1, \delta'}(r^x)}_{\in [0,1]} e_k,
    $$
    for $ i,j \in [M], k \in [d']$.
    So each component of $\tilde\psi$ lies in $[0,1]$ whenever the edge features satisfy $a\in[0,1]^{d'}$. One might therefore expect
    $$
    \sum_{i \in \mathcal{N}(j)}\tilde\psi(x_j, a_{i,j}, x_{i}) \in \big[0,|\mathcal{N}(j)|\big]^{d_\psi} \subseteq [0,N]^{d_\psi}.
    $$ 
    However, recall that $\tilde\psi$ is a $k$-times continuously differentiable approximation of the measurable function $\psi$ from Proposition~\ref{prop:repr_of_GPEqui}. By Proposition~\ref{prop:repr_of_GPEqui_[0,1]_approx}, the functions $\psi$ and $\tilde\psi$ coincide exactly on $\mathfrak{A}\cap\mathfrak{B}$, since the node features lie in the region where the approximation is exact and the last coordinates $x_{i,d+1}$, $i\in[N]$, belong to distinct intervals $I_{k-1}$, $k\in[M]$.
    Consequently, on $\mathfrak{A}\cap\mathfrak{B}$ the tensor $\tilde\psi(x_j,a_{i,j},x_i)\in\R^{M\times M\times d'}$ can be nonzero only at the index
    $(h(x_{i,d+1}), h(x_{j,d+1}),k), k \in [d']$ 
    where it takes the values $a_{i,j,k},k\in[d']$.
    In particular, for all $(i^*,j^*)\neq (i, j)$, and all $k \in [d']$, $\tilde\psi(x_{j^*},a_{i^*,j^*},x_{i^*})_{h(x_{i,d+1}), h(x_{j,d+1}),k} = 0$. 
    Thus, for each component $j \in [d_\psi]$, all but one summand in the neighborhood sum must vanish. 
    Therefore,
    \begin{align*}
        &\sum_{i \in \mathcal{N}(j)}\tilde\psi(x_j, a_{i,j}, x_{i}) \in [0,1]^{d_\psi}, & j \in [N].
    \end{align*}
    With analogous reasoning, for $(x,a) \in K$ we also obtain
    \begin{align}
        &\sum_{j \in [N]} \tilde\alpha \left(x_j,\sum_{i \in \mathcal{N}(j)} \tilde\psi \left(x_j,a_{i,j}, x_i\right)\right) \in [0,1]^{d_\alpha}.\label{eq:alpha_in_01}
    \end{align}
    Since $\tilde\phi \equiv 0$, it follows immediately that
    \begin{align*}
        &\sum_{i \in \mathcal{N}(k)} \tilde\phi \left(x_k, a_{i,k}, x_i\right) \in [0,1], & k \in [N].
    \end{align*}
    We now estimate
    \begin{align*}
        &\max_{(x,a) \in K }\| \tilde g(x,a) - \hat g(x,a) \|_{\infty} =\\
        &= \max_{(x,a) \in K } 
        \max_{k \in [N]}\left\|\tilde\rho \left(x_k, \sum_{i \in \mathcal{N}(k)} \tilde\phi \left(x_k, a_{i,k}, x_i\right), \sum_{j \in [N]} \tilde\alpha \left(x_j,\sum_{i \in \mathcal{N}(j)} \tilde\psi  \left(x_j,a_{i,j}, x_i\right)\right)\right) \right.\\ &\phantom{=+}\left. - \hat \rho \left(x_k, \sum_{i \in \mathcal{N}(k)}\hat \phi \left(x_k, a_{i,k}, x_i\right), \sum_{j \in [N]}\hat \alpha \left(x_j,\sum_{i \in \mathcal{N}(j)} \hat \psi  \left(x_j,a_{i,j}, x_i\right)\right)\right) \right\|_\infty\\
        &\leq \max_{(x,a) \in K } 
        \max_{k \in [N]}\left\|\tilde\rho \left(x_k, \sum_{i \in \mathcal{N}(k)} \tilde\phi \left(x_k, a_{i,k}, x_i\right), \sum_{j \in [N]} \tilde\alpha \left(x_j,\sum_{i \in \mathcal{N}(j)} \tilde\psi  \left(x_j,a_{i,j}, x_i\right)\right)\right) \right.\\ &\phantom{=+}\left. - \hat \rho \left(x_k, \sum_{i \in \mathcal{N}(k)} \tilde\phi \left(x_k, a_{i,k}, x_i\right), \sum_{j \in [N]} \tilde\alpha \left(x_j,\sum_{i \in \mathcal{N}(j)}  \tilde\psi  \left(x_j,a_{i,j}, x_i\right)\right)\right) \right\|_\infty\\
        &\phantom{\leq}+ \max_{(x,a) \in K } 
        \max_{k \in [N]}\left\|\hat \rho \left(x_k, \sum_{i \in \mathcal{N}(k)} \tilde\phi \left(x_k, a_{i,k}, x_i\right), \sum_{j \in [N]} \tilde\alpha \left(x_j,\sum_{i \in \mathcal{N}(j)} \tilde\psi  \left(x_j,a_{i,j}, x_i\right)\right)\right) \right.\\ &\phantom{=+}\left. - \hat \rho \left(x_k, \sum_{i \in \mathcal{N}(k)}\hat \phi \left(x_k, a_{i,k}, x_i\right), \sum_{j \in [N]} \tilde\alpha \left(x_j,\sum_{i \in \mathcal{N}(j)} \tilde\psi \left(x_j,a_{i,j}, x_i\right)\right)\right) \right\|_\infty\\
        &\phantom{\leq}+ \max_{(x,a) \in K } 
        \max_{k \in [N]}\left\|\hat \rho \left(x_k, \sum_{i \in \mathcal{N}(k)} \hat \phi \left(x_k, a_{i,k}, x_i\right), \sum_{j \in [N]} \tilde\alpha \left(x_j,\sum_{i \in \mathcal{N}(j)} \tilde\psi  \left(x_j,a_{i,j}, x_i\right)\right)\right) \right.\\ &\phantom{=+}\left. - \hat \rho \left(x_k, \sum_{i \in \mathcal{N}(k)}\hat \phi \left(x_k, a_{i,k}, x_i\right), \sum_{j \in [N]} \hat \alpha \left(x_j,\sum_{i \in \mathcal{N}(j)} \tilde\psi \left(x_j,a_{i,j}, x_i\right)\right)\right)\right\|_\infty\\
        &\phantom{\leq}+ \max_{(x,a) \in K } 
        \max_{k \in [N]}\left\|\hat \rho \left(x_k, \sum_{i \in \mathcal{N}(k)} \hat \phi \left(x_k, a_{i,k}, x_i\right), \sum_{j \in [N]} \hat \alpha \left(x_j,\sum_{i \in \mathcal{N}(j)} \tilde\psi  \left(x_j,a_{i,j}, x_i\right)\right)\right) \right.\\ &\phantom{=+}\left. - \hat \rho \left(x_k, \sum_{i \in \mathcal{N}(k)}\hat \phi \left(x_k, a_{i,k}, x_i\right), \sum_{j \in [N]} \hat \alpha \left(x_j,\sum_{i \in \mathcal{N}(j)} \hat \psi \left(x_j,a_{i,j}, x_i\right)\right)\right) \right\|_\infty,
    \end{align*}
    where we used the triangle inequality to decompose the difference into four terms.
    For the first term, we obtain
    \begin{align*}
        &\max_{(x,a) \in K } 
        \max_{k \in [N]}\left\|\tilde\rho \left(x_k, \sum_{i \in \mathcal{N}(k)} \tilde\phi \left(x_k, a_{i,k}, x_i\right), \sum_{j \in [N]} \tilde\alpha \left(x_j,\sum_{i \in \mathcal{N}(j)} \tilde\psi  \left(x_j,a_{i,j}, x_i\right)\right)\right) \right.\\ &\phantom{=+}\left. - \hat \rho \left(x_k, \sum_{i \in \mathcal{N}(k)} \tilde\phi \left(x_k, a_{i,k}, x_i\right), \sum_{j \in [N]} \tilde\alpha \left(x_j,\sum_{i \in \mathcal{N}(j)}  \tilde\psi  \left(x_j,a_{i,j}, x_i\right)\right)\right) \right\|_\infty\\
        &\leq \max_{z \in [0,1]^{\bar d_\rho}} \max_{k \in [N]} \left\|\tilde\rho(z) - \hat \rho(z) \right\|_\infty\\
        &= \max_{z \in (0,1)^{\bar d_\rho}} \left\|\tilde\rho(z) - \hat \rho(z) \right\|_\infty \leq \varepsilon_\rho = \frac{\varepsilon}{3},
    \end{align*}
    where we used \eqref{eq:psi_in_01}, \eqref{eq:continuity_open_interval_bound}, and finally \eqref{eq:approx_of_rho}.
    For the second summand, since $\tilde\phi \equiv \phi \equiv 0$ is represented exactly (not merely approximated) by the ReLU network $\hat\phi \equiv 0$, we obtain
    \begin{align*}
        &\max_{(x,a) \in K } 
        \max_{k \in [N]}\left\|\hat \rho \left(x_k, \sum_{i \in \mathcal{N}(k)} \tilde\phi \left(x_k, a_{i,k}, x_i\right), \sum_{j \in [N]} \tilde\alpha \left(x_j,\sum_{i \in \mathcal{N}(j)} \tilde\psi  \left(x_j,a_{i,j}, x_i\right)\right)\right) \right.\\ &\phantom{=}\left. - \hat \rho \left(x_k, \sum_{i \in \mathcal{N}(k)}\hat \phi \left(x_k, a_{i,k}, x_i\right), \sum_{j \in [N]} \tilde\alpha \left(x_j,\sum_{i \in \mathcal{N}(j)} \tilde\psi \left(x_j,a_{i,j}, x_i\right)\right)\right) \right\|_\infty = 0.
    \end{align*}
    For the third term, by Corollary~\ref{cor:rates_for_NN_global_lip} and  \eqref{eq:approx_of_rho} we know that the network $\hat\rho$ is $\bar d_\rho(\kappa_\rho+\varepsilon_\rho)$‑Lipschitz. We therefore obtain
    \begin{align*}
        &\max_{(x,a) \in K } 
        \max_{k \in [N]}\left\|\hat \rho \left(x_k, \sum_{i \in \mathcal{N}(k)} \hat \phi \left(x_k, a_{i,k}, x_i\right), \sum_{j \in [N]} \tilde\alpha \left(x_j,\sum_{i \in \mathcal{N}(j)} \tilde\psi  \left(x_j,a_{i,j}, x_i\right)\right)\right) \right.\\ &\phantom{=+}\left. - \hat \rho \left(x_k, \sum_{i \in \mathcal{N}(k)}\hat \phi \left(x_k, a_{i,k}, x_i\right), \sum_{j \in [N]} \hat \alpha \left(x_j,\sum_{i \in \mathcal{N}(j)} \tilde\psi \left(x_j,a_{i,j}, x_i\right)\right)\right)\right\|_\infty\\
        &\leq \max_{(x,a) \in K } 
        \max_{k \in [N]} \bar d_\rho (\kappa_\rho + \varepsilon_\rho)\left\|\left(x_k, \sum_{i \in \mathcal{N}(k)} \hat \phi \left(x_k, a_{i,k}, x_i\right), \sum_{j \in [N]} \tilde\alpha \left(x_j,\sum_{i \in \mathcal{N}(j)} \tilde\psi  \left(x_j,a_{i,j}, x_i\right)\right)\right) \right.\\ &\phantom{=+}\left. - \left(x_k, \sum_{i \in \mathcal{N}(k)}\hat \phi \left(x_k, a_{i,k}, x_i\right), \sum_{j \in [N]} \hat \alpha \left(x_j,\sum_{i \in \mathcal{N}(j)} \tilde\psi \left(x_j,a_{i,j}, x_i\right)\right)\right)\right\|_\infty\\
        &= \max_{(x,a) \in K } \bar d_\rho (\kappa_\rho + \varepsilon_\rho) 
        \left\|\left(\sum_{j \in [N]} \tilde\alpha \left(x_j,\sum_{i \in \mathcal{N}(j)} \tilde\psi  \left(x_j,a_{i,j}, x_i\right)\right)\right)
        \right.\\ &\phantom{=+}\left.- \left(\sum_{j \in [N]} \hat \alpha \left(x_j,\sum_{i \in \mathcal{N}(j)} \tilde\psi \left(x_j,a_{i,j}, x_i\right)\right)\right)\right\|_\infty\\
        &\leq \max_{(x,a) \in K } \bar d_\rho (\kappa_\rho + \varepsilon_\rho)
         N \max_{j \in [N]}\left\| \tilde\alpha \left(x_j,\sum_{i \in \mathcal{N}(j)} \tilde\psi  \left(x_j,a_{i,j}, x_i\right)\right) - \hat \alpha \left(x_j,\sum_{i \in \mathcal{N}(j)} \tilde\psi \left(x_j,a_{i,j}, x_i\right)\right) \right\|_\infty\\
         &\leq  \bar d_\rho (\kappa_\rho + \varepsilon_\rho)
         N \max_{z \in [0,1]^{\bar d_\alpha}} \left\| \tilde\alpha(z) - \hat \alpha(z)\right\|_\infty \\
         &\leq  \bar d_\rho (\kappa_\rho + \varepsilon_\rho)
         N \max_{z \in (0,1)^{\bar d_\alpha}} \left\| \tilde\alpha(z) - \hat \alpha(z)\right\|_\infty \leq \bar d_\rho (\kappa_\rho + \varepsilon_\rho)
         N \varepsilon_\alpha \leq \bar d_\rho (\kappa_\rho + \frac{\varepsilon}{3})N \frac{\varepsilon}{3N\bar d_\rho(\kappa_\rho +1)} \leq \frac{\varepsilon}{3},
    \end{align*}
    where we used \eqref{eq:alpha_in_01}, \eqref{eq:continuity_open_interval_bound}, \eqref{eq:approx_of_alpha}, and the definitions of $\varepsilon_\rho$ and $\varepsilon_\alpha$.
    Similarly, for the fourth term, Corollary~\ref{cor:rates_for_NN_global_lip} and \eqref{eq:approx_of_alpha} ensure that $\hat \alpha$ is $\bar d_\alpha(\kappa_\alpha + \varepsilon_\alpha)$ Lipschitz. 
    We obtain
    \begin{align*}
        &\max_{(x,a) \in K } 
        \max_{k \in [N]}\left\|\hat \rho \left(x_k, \sum_{i \in \mathcal{N}(k)} \hat \phi \left(x_k, a_{i,k}, x_i\right), \sum_{j \in [N]} \hat \alpha \left(x_j,\sum_{i \in \mathcal{N}(j)} \tilde\psi  \left(x_j,a_{i,j}, x_i\right)\right)\right) \right.\\ &\phantom{=}\left. - \hat \rho \left(x_k, \sum_{i \in \mathcal{N}(k)}\hat \phi \left(x_k, a_{i,k}, x_i\right), \sum_{j \in [N]} \hat \alpha \left(x_j,\sum_{i \in \mathcal{N}(j)} \hat \psi \left(x_j,a_{i,j}, x_i\right)\right)\right) \right\|_\infty\\
        &\leq \max_{(x,a) \in K } \bar d_\rho(\kappa_\rho + \varepsilon_\rho)N  
        \max_{j \in [N]}\left\|\hat \alpha \left(x_j,\sum_{i \in \mathcal{N}(j)} \tilde\psi  \left(x_j,a_{i,j}, x_i\right)\right) - \hat \alpha \left(x_j,\sum_{i \in \mathcal{N}(j)} \hat \psi \left(x_j,a_{i,j}, x_i\right)\right) \right\|_\infty\\
        &\leq \max_{(x,a) \in K } \bar d_\rho(\kappa_\rho + \varepsilon_\rho)N 
        \max_{j \in [N]}\bar d_\alpha(\kappa_\alpha + \varepsilon_\alpha)\left\|\left(\sum_{i \in \mathcal{N}(j)} \tilde\psi  \left(x_j,a_{i,j}, x_i\right)\right) - \left(\sum_{i \in \mathcal{N}(j)} \hat \psi \left(x_j,a_{i,j}, x_i\right)\right) \right\|_\infty\\
        &\leq \max_{(x,a) \in K } \bar d_\rho(\kappa_\rho + \varepsilon_\rho)N 
        \bar d_\alpha(\kappa_\alpha + \varepsilon_\alpha)N\max_{j,i \in [N]}\left\|\tilde\psi  \left(x_j,a_{i,j}, x_i\right) - \hat \psi \left(x_j,a_{i,j}, x_i\right) \right\|_\infty\\
        &\leq \bar d_\rho(\kappa_\rho + \varepsilon_\rho)N^2 
        \bar d_\alpha(\kappa_\alpha + \varepsilon_\alpha)\varepsilon_\psi \leq \frac{\varepsilon}{3}.
    \end{align*}
    Combining all four estimates, we conclude that
    \begin{align*}
        \max_{(x,a) \in K 
         }\| \tilde g(x,a) - \hat g(x,a) \|_{\infty} \leq \frac{\varepsilon}{3} + 0 + \frac{\varepsilon}{3} + \frac{\varepsilon}{3} = \varepsilon.
    \end{align*}
\end{proof}

\begin{remark}\label{rem:simplify_constants}
    We can simplify the presentation of Proposition~\ref{prop:component_error_to_total_error} by reorganizing the constants that appear.
    First, note that $M=M(N, \delta)$ and $\delta'=\delta'(N, \delta, M(N,\delta)) = \delta'(N, \delta)$ depend only on $N$ and $\delta$.
    Next, the quantities $\bar d_\psi$, $d_\psi$, $\bar d_\alpha$, $d_\alpha$, and $\bar d_\rho$ from \eqref{eq:func_dimensions} depend only on $(N, d,d',\delta)$, while $d_\rho = l$ depends solely on $l$.
    Define
    \begin{align}\label{eq:const_c_2}
    \begin{split}
        &c_2(N,d,d',\delta) := \max(\bar d_\rho, \bar d_\alpha, \bar d_\psi),\\
        &c_3(N,d,d',\delta, \kappa, k) := 3N^2\bar d_\rho(\kappa_\rho +1)\bar d_\alpha(\kappa_\alpha+1),\\
        &c_1(N,d,d',l,\delta, \kappa,k) := \max\left(c_\rho, c_\alpha, c_\psi \right)\cdot\max\left(d_\rho, d_\alpha, d_\psi \right)c_3^{c_2},\\
        &c_1^*(N,d,d',l,\delta, \kappa,k) := c_1 \left(\frac{k}{k-1}\right)^2.
        \end{split}
    \end{align}
    Then, for $\mathcal{A} \in \{\mathcal{A}_\rho, \mathcal{A}_\alpha, \mathcal{A}_\phi, \mathcal{A}_\psi\}$
    \begin{align*}
        L(\mathcal{A}) &\leq c_1 \log_2\left(\left(\frac{\varepsilon}{c_3}\right)^{-k/(k-1)} \right) \leq c_1^* \log_2\left(\left(\frac{\varepsilon}{c_3}\right)^{-1} \right) ;\\
        M(\mathcal{A}) &\leq c_1 \left(\frac{\varepsilon}{c_3}\right)^{-c_2/(k-1)} \left(\log_2\left(\left(\frac{\varepsilon}{c_3}\right)^{-k/(k-1)}\right)\right)^2 \leq c_1^* \left(\frac{\varepsilon}{c_3}\right)^{-c_2/(k-1)} \left(\log_2\left(\left(\frac{\varepsilon}{c_3}\right)^{-1}\right)\right)^2.
    \end{align*}
    Furthermore, since $\varepsilon \leq 1/2$, we have $\log_2(1/\varepsilon) \geq 1$. 
    Define
    \begin{align*}
        c_1^{**}(N,d,d',l,\delta,k,\kappa) := c_1^* c_3^{c_2/(k-1)}\left(1+ \log_2(c_3) \right)^2.
    \end{align*}
    Then
    \begin{align*}
        L(\mathcal{A}) &\leq c_1^* \log_2\left(\left(\frac{\varepsilon}{c_3}\right)^{-1} \right) = c_1^* \left(\log_2\left(\frac{1}{\varepsilon}\right)+ \log_2(c_3) \right) \leq c_1^* \left(\log_2\left(\frac{1}{\varepsilon}\right)+ \log_2\left(\frac{1}{\varepsilon}\right)\log_2(c_3) \right)\\ &= c_1^*(1+ \log_2(c_3)) \log_2\left(\frac{1}{\varepsilon}\right) \leq c_1^{**}\log_2\left(\frac{1}{\varepsilon} \right),
    \end{align*} 
    and similarly,
    \begin{align*}
        M(\mathcal{A}) &\leq c_1^* \left(\frac{\varepsilon}{c_3}\right)^{-c_2/(k-1)} \left(\log_2\left(\left(\frac{\varepsilon}{c_3}\right)^{-1}\right)\right)^2 = c_1^* c_3^{c_2/(k-1)}\left(\frac{1}{\varepsilon}\right)^{c_2/(k-1)} \left(\log_2\left(\frac{1}{\varepsilon}\right)+ \log_2(c_3)\right)^2\\ 
        &\leq c_1^* c_3^{c_2/(k-1)}\left(1+ \log_2(c_3) \right)^2\left(\frac{1}{\varepsilon}\right)^{c_2/(k-1)} \left(\log_2\left(\frac{1}{\varepsilon}\right)\right)^2 \leq c_1^{**} \left(\frac{1}{\varepsilon} \right)^{c_2/(k-1)}\left(\log_2\left(\frac{1}{\varepsilon}\right)\right)^2.
    \end{align*}   
\end{remark}

\section{Illustrations: Shortcomings of Simple Message Passing}\label{appendix:graph_examples}
This appendix collects concrete graph examples that illustrate two fundamental shortcomings of AC-GNNs: their inability to distinguish certain non-isomorphic graphs, and the resulting failure to compute certain graph-level or node-level quantities of interest. We further show, through the same examples, how these limitations can be overcome by augmenting node features with random features or by incorporating a global readout function.

Figure~\ref{fig:non_iso_graphs} shows two graphs with nodes of two types, $a$ and $b$. The graphs are not isomorphic, yet every node of type $a$ has an identical 1-hop neighborhood in both graphs — and the same holds for every node of type $b$ (Figure~\ref{fig:non_iso_graphs}b).
Because AC-GNNs update node representations by aggregating over local neighborhoods, nodes of the same type receive identical messages in every layer, and consequently all type-$a$ nodes (and all type-$b$ nodes) converge to the same representation regardless of the number of message-passing rounds. Starting from identical initial features $r_a, r_b \in \mathbb{R}^d$, no AC-GNN can therefore assign different labels to any two nodes of the same type, making the two graphs indistinguishable.

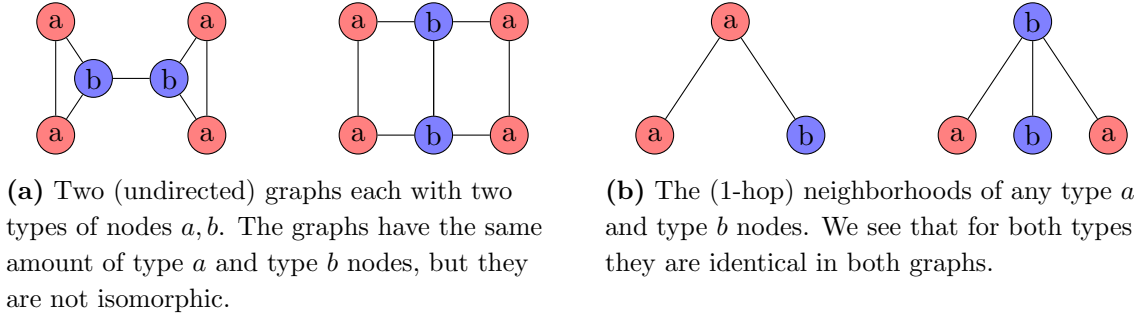
\begin{figure}[h]
  \centering
  \begin{minipage}[t]{0.48\textwidth}
    \centering
    \begin{tikzpicture}[scale=1.0,
      every node/.style={circle, draw, minimum size=5mm, inner sep=0pt, text centered}]
      \begin{scope}[xshift=-2cm]
        \node[fill=red!50]  (a1) at (0,1.5)   {a};
        \node[fill=red!50]  (a2) at (0,0)     {a};
        \node[fill=blue!50] (b1) at (0.5,0.75){b};
        \node[fill=red!50]  (a3) at (2,1.5)   {a};
        \node[fill=red!50]  (a4) at (2,0)     {a};
        \node[fill=blue!50] (b2) at (1.5,0.75){b};
        \draw (a1) -- (a2) -- (b1) -- (a1) -- cycle;
        \draw (a3) -- (a4) -- (b2) -- (a3) -- cycle;
        \draw (b1) -- (b2);
      \end{scope}
      \begin{scope}[xshift=2cm]
        \node[fill=red!50]  (a5) at (0,1.5){a};
        \node[fill=red!50]  (a6) at (0,0)  {a};
        \node[fill=blue!50] (b3) at (1,1.5){b};
        \node[fill=blue!50] (b4) at (1,0)  {b};
        \node[fill=red!50]  (a7) at (2,0)  {a};
        \node[fill=red!50]  (a8) at (2,1.5){a};
        \draw (a5) -- (a6) -- (b4) -- (b3) -- (a5) -- cycle;
        \draw (a8) -- (a7) -- (b4) -- (b3) -- (a8) -- cycle;
      \end{scope}
    \end{tikzpicture}
    \par\smallskip\raggedright
    {\small\textbf{(a)}~Two (undirected) graphs each with two types of nodes $a,b$.
    The graphs have the same amount of type $a$ and type $b$ nodes,
    but they are not isomorphic.}
  \end{minipage}
  \hfill
  \begin{minipage}[t]{0.48\textwidth}
    \centering
    \begin{tikzpicture}[scale=1.0,
      every node/.style={circle, draw, minimum size=5mm, inner sep=0pt, text centered}]
      \begin{scope}[xshift=-2cm]
        \node[fill=red!50]  (a9)  at (1,1.5){a};
        \node[fill=red!50]  (a10) at (0,0)  {a};
        \node[fill=blue!50] (b5)  at (2,0)  {b};
        \draw (a9)--(a10);
        \draw (a9)--(b5);
      \end{scope}
      \begin{scope}[xshift=2cm]
        \node[fill=blue!50] (b6)  at (1,1.5){b};
        \node[fill=red!50]  (a11) at (0,0)  {a};
        \node[fill=blue!50] (b7)  at (1,0)  {b};
        \node[fill=red!50]  (a12) at (2,0)  {a};
        \draw (b6)--(a11);
        \draw (b6)--(b7);
        \draw (b6)--(a12);
      \end{scope}
    \end{tikzpicture}
    \par\smallskip\raggedright
    {\small\textbf{(b)}~The (1-hop) neighborhoods of any type $a$ and type $b$ nodes.
    We see that for both types they are identical in both graphs.}
  \end{minipage}
  \caption{Two non-isomorphic graphs that cannot be distinguished by
    message-passing (AC) graph neural networks because all nodes of
    type $a$ and all nodes of type $b$ are ``tree-like equivalent''.
    Therefore, starting with identical node features $r_a, r_b \in \mathbb{R}^d$,
    all nodes of type $a$ (and all nodes of type $b$) can never end up with
    different node labels by means of traditional message-passing.
    }
  \label{fig:non_iso_graphs}
\end{figure}

The graphs in Figure~\ref{fig:readout_and_triangle} illustrate the usefulness of a global readout function and of augmented node features.

In Figure~\ref{fig:readout_and_triangle}a, the two graphs contain only type-$a$ nodes.
They are not isomorphic, since one has three nodes and the other four, yet every node has the same local neighborhood structure (each type-$a$ node is connected to two others of the same type).
Consequently, an AC-GNN produces identical node representations in both graphs.
If we allow a global readout that sums all node representations (denoted $r_a \in \mathbb{R}$), the left graph yields $3r_a$ and the right one $4r_a$, making them distinguishable.

In Figure~\ref{fig:readout_and_triangle}b, both graphs consist of six type-$a$ nodes.
They are clearly non-isomorphic, although in each graph every node has two type-$a$ neighbors.
Thus, an AC-GNN produces identical node representations and cannot approximate a function that detects triangles (it labels nodes with $1$ if they belong to a triangle and $0$ otherwise).
If we instead assign random features $R_i$ to nodes $i=1,\ldots,6$ (for example, i.i.d.\ uniform on $[0,1]$), then with probability one all node features are distinct; see Figure~\ref{fig:detect_triangle}a.
The resulting ``unrolling trees'' rooted at each node differ, and a message-passing GNN can, in principle, detect whether a feature $R_i$ reappears at specific depths in the tree of node $i$, which is necessary for identifying whether $i$ lies in a triangle; see Figure~\ref{fig:detect_triangle}b and \ref{fig:detect_triangle}c.

\begin{figure}[h]
  \centering
  \begin{minipage}[t]{0.48\textwidth}
    \centering
    \begin{tikzpicture}[scale=1.0,
      every node/.style={circle, draw, minimum size=5mm, inner sep=0pt, text centered}]
      \begin{scope}[xshift=-2cm]
        \node[fill=red!50] (a1) at (0,0)   {a};
        \node[fill=red!50] (a2) at (2,0)   {a};
        \node[fill=red!50] (a3) at (1,1.5) {a};
        \draw (a1) -- (a2) -- (a3) -- (a1);
      \end{scope}
      \begin{scope}[xshift=2cm]
        \node[fill=red!50] (a1) at (0,0)   {a};
        \node[fill=red!50] (a2) at (0,1.5) {a};
        \node[fill=red!50] (a3) at (2,1.5) {a};
        \node[fill=red!50] (a4) at (2,0)   {a};
        \draw (a1) -- (a2) -- (a3) -- (a4) -- (a1);
      \end{scope}
    \end{tikzpicture}
    \par\smallskip
    \raggedright
    {\small\textbf{(a)}~Two graphs with only one node type. A message-passing GNN
    cannot assign different node labels, because the neighborhoods of all nodes are
    ``tree-like equivalent''. A global readout function with access to all nodes in
    the graph can distinguish nodes from both graphs, for example by summing all
    node features.}
  \end{minipage}
  \hfill
  \begin{minipage}[t]{0.48\textwidth}
    \centering
    \begin{tikzpicture}[scale=1.0,
      every node/.style={circle, draw, minimum size=5mm, inner sep=0pt, text centered}]
      \begin{scope}[xshift=-2cm]
        \node[fill=red!50] (a1) at (0,0.75)  {a};
        \node[fill=red!50] (a2) at (0.5,1.5) {a};
        \node[fill=red!50] (a3) at (1.5,1.5) {a};
        \node[fill=red!50] (a4) at (2,0.75)  {a};
        \node[fill=red!50] (a5) at (1.5,0)   {a};
        \node[fill=red!50] (a6) at (0.5,0)   {a};
        \draw (a1)--(a2)--(a6)--(a1);
        \draw (a3)--(a4)--(a5)--(a3);
      \end{scope}
      \begin{scope}[xshift=2cm]
        \node[fill=red!50] (a1) at (0,0.75)  {a};
        \node[fill=red!50] (a2) at (0.5,1.5) {a};
        \node[fill=red!50] (a3) at (1.5,1.5) {a};
        \node[fill=red!50] (a4) at (2,0.75)  {a};
        \node[fill=red!50] (a5) at (1.5,0)   {a};
        \node[fill=red!50] (a6) at (0.5,0)   {a};
        \draw (a1)--(a2)--(a3)--(a4)--(a5)--(a6)--(a1);
      \end{scope}
    \end{tikzpicture}
    \par\smallskip
    \raggedright
    {\small\textbf{(b)}~Two undirected graphs with six nodes of type $a$. One graph
    consists of two disconnected triangles, the other forms a hexagon. Since the
    neighborhoods of all nodes are ``tree-like equivalent'' a message-passing GNN
    cannot distinguish nodes from the two graphs, which would, for example, be
    necessary in order to detect triangles.}
  \end{minipage}
  \caption{Two pairs of graphs that highlight the necessity of additional expressive
    power via global readout functions or random node features.
    }
  \label{fig:readout_and_triangle}
\end{figure}
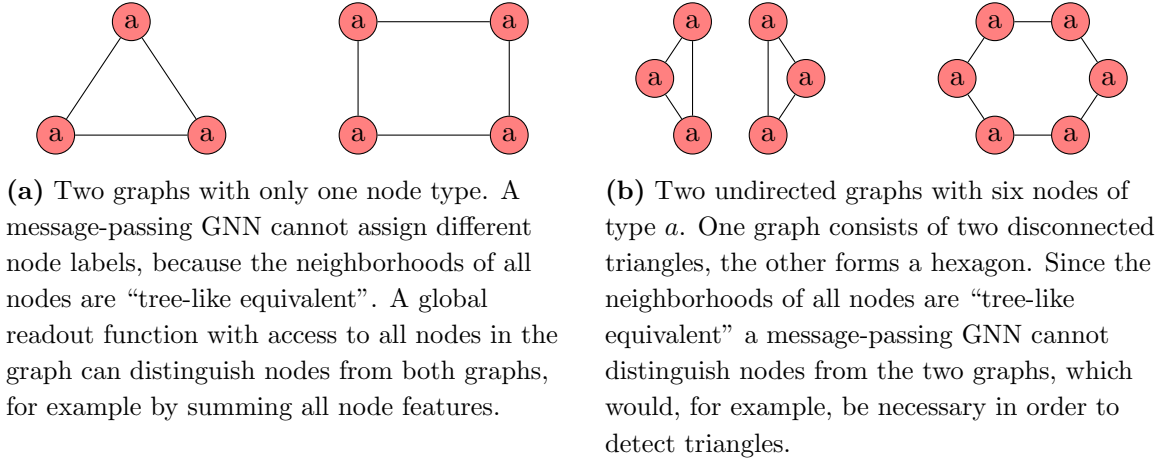
\begin{figure}[h]
  \centering
  \begin{minipage}[t]{0.29\textwidth}
    \centering
    \begin{tikzpicture}[scale=0.95,
      every node/.style={circle, draw, minimum size=5mm, inner sep=0pt, text centered}]
      \begin{scope}[xshift=-1.2cm]
        \node[fill=red!50]    (a1) at (0,0.75)  {$r_2$};
        \node[fill=blue!50]   (a2) at (0.4,1.5) {$r_3$};
        \node[fill=purple!50] (a3) at (1.1,1.5) {$r_4$};
        \node[fill=orange!50] (a4) at (1.5,0.75){$r_5$};
        \node[fill=brown!50]  (a5) at (1.1,0)   {$r_6$};
        \node[fill=green!50]  (a6) at (0.4,0)   {$r_1$};
        \draw (a1)--(a2)--(a6)--(a1);
        \draw (a3)--(a4)--(a5)--(a3);
      \end{scope}
      \begin{scope}[xshift=1.2cm]
        \node[fill=red!50]    (a1) at (0,0.75)  {$r_2$};
        \node[fill=blue!50]   (a2) at (0.4,1.5) {$r_3$};
        \node[fill=purple!50] (a3) at (1.1,1.5) {$r_4$};
        \node[fill=orange!50] (a4) at (1.5,0.75){$r_5$};
        \node[fill=brown!50]  (a5) at (1.1,0)   {$r_6$};
        \node[fill=green!50]  (a6) at (0.4,0)   {$r_1$};
        \draw (a1)--(a2)--(a3)--(a4)--(a5)--(a6)--(a1);
      \end{scope}
    \end{tikzpicture}
    \par\smallskip
    \raggedright
    {\small\textbf{(a)}~The graphs from Figure~\ref{fig:readout_and_triangle} (b)
    after replacing the (uninformative) node type by random features $r_1,\ldots,r_6$.}
  \end{minipage}
  \hfill
  \begin{minipage}[t]{0.34\textwidth}
    \centering
    \begin{tikzpicture}[scale=0.95,
      every node/.style={circle, draw, minimum size=5mm, inner sep=0pt, text centered}]
      \node[fill=green!50]  (a1)  at (2.45,2)   {$r_1$};
      \node[fill=red!50]    (a21) at (1.05,1.33) {$r_2$};
      \node[fill=blue!50]   (a22) at (3.85,1.33) {$r_3$};
      \node[fill=green!50]  (a31) at (0.35,0.66) {$r_1$};
      \node[fill=blue!50]   (a32) at (1.75,0.66) {$r_3$};
      \node[fill=green!50]  (a33) at (3.15,0.66) {$r_1$};
      \node[fill=red!50]    (a34) at (4.55,0.66) {$r_2$};
      \node[fill=red!50]    (a41) at (0,0)        {$r_2$};
      \node[fill=blue!50]   (a42) at (0.7,0)      {$r_3$};
      \node[fill=green!50]  (a43) at (1.4,0)      {$r_1$};
      \node[fill=red!50]    (a44) at (2.1,0)      {$r_2$};
      \node[fill=red!50]    (a45) at (2.8,0)      {$r_2$};
      \node[fill=blue!50]   (a46) at (3.5,0)      {$r_3$};
      \node[fill=green!50]  (a47) at (4.2,0)      {$r_1$};
      \node[fill=blue!50]   (a48) at (4.9,0)      {$r_3$};
      \draw (a21)--(a1)--(a22);
      \draw (a31)--(a21)--(a32);
      \draw (a33)--(a22)--(a34);
      \draw (a41)--(a31)--(a42);
      \draw (a43)--(a32)--(a44);
      \draw (a45)--(a33)--(a46);
      \draw (a47)--(a34)--(a48);
    \end{tikzpicture}
    \par\smallskip
    \raggedright
    {\small\textbf{(b)}~The ``unrolling tree'' up to level four of root node $r_1$
    in the left graph. Since $r_1$ appears on the fourth level, it is part of a
    triangle. Thanks to the random features, a message-passing GNN is now able
    to detect this.}
  \end{minipage}
  \hfill
  \begin{minipage}[t]{0.34\textwidth}
    \centering
    \begin{tikzpicture}[scale=0.95,
      every node/.style={circle, draw, minimum size=5mm, inner sep=0pt, text centered}]
      \node[fill=green!50]  (a1)  at (2.45,2)   {$r_1$};
      \node[fill=red!50]    (a21) at (1.05,1.33) {$r_2$};
      \node[fill=brown!50]  (a22) at (3.85,1.33) {$r_6$};
      \node[fill=green!50]  (a31) at (0.35,0.66) {$r_1$};
      \node[fill=blue!50]   (a32) at (1.75,0.66) {$r_3$};
      \node[fill=green!50]  (a33) at (3.15,0.66) {$r_1$};
      \node[fill=orange!50] (a34) at (4.55,0.66) {$r_5$};
      \node[fill=red!50]    (a41) at (0,0)        {$r_2$};
      \node[fill=brown!50]  (a42) at (0.7,0)      {$r_6$};
      \node[fill=red!50]    (a43) at (1.4,0)      {$r_2$};
      \node[fill=purple!50] (a44) at (2.1,0)      {$r_4$};
      \node[fill=red!50]    (a45) at (2.8,0)      {$r_2$};
      \node[fill=brown!50]  (a46) at (3.5,0)      {$r_6$};
      \node[fill=purple!50] (a47) at (4.2,0)      {$r_4$};
      \node[fill=brown!50]  (a48) at (4.9,0)      {$r_6$};
      \draw (a21)--(a1)--(a22);
      \draw (a31)--(a21)--(a32);
      \draw (a33)--(a22)--(a34);
      \draw (a41)--(a31)--(a42);
      \draw (a43)--(a32)--(a44);
      \draw (a45)--(a33)--(a46);
      \draw (a47)--(a34)--(a48);
    \end{tikzpicture}
    \par\smallskip
    \raggedright
    {\small\textbf{(c)}~The ``unrolling tree'' up to level four of root node $r_1$
    in the right graph. Since $r_1$ does not appear on the fourth level, it is
    not possible that $r_1$ is part of a triangle. Thanks to the random features,
    a message-passing GNN is now able to detect this.}
  \end{minipage}
  \caption{Assigning random features to the nodes ensures that the neighborhoods
    for the nodes in the graphs from
    Figure~\ref{fig:readout_and_triangle} (b) are not
    ``tree-like equivalent'' anymore. Therefore, message-passing GNNs now have a
    chance to detect whether a node is in a triangle or not.
    }
  \label{fig:detect_triangle}
\end{figure}
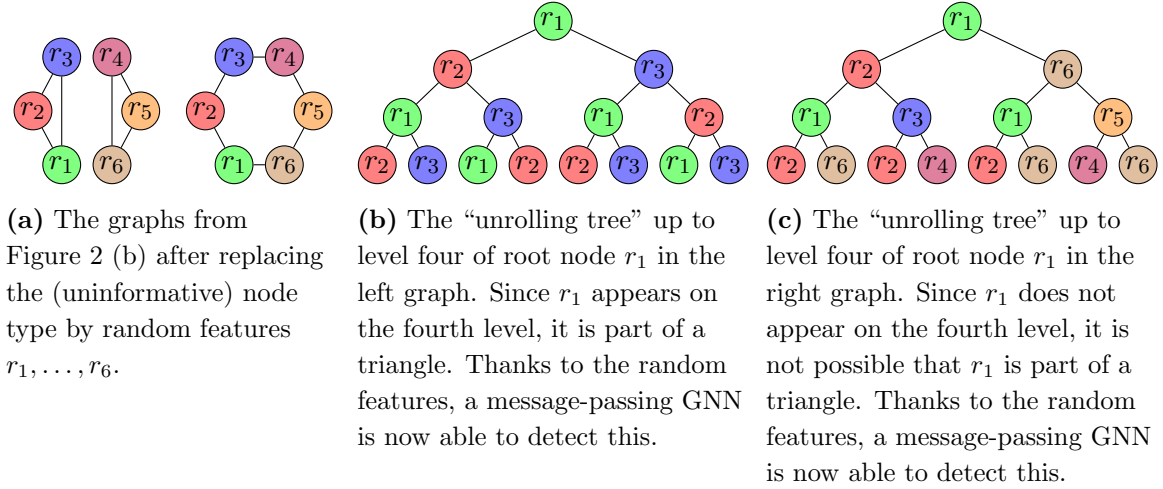

\newpage
\bibliography{references}

\end{document}